\newif\ifSUPP
\SUPPfalse
\SUPPtrue 

\documentclass[11pt]{article}

\usepackage{algorithm}

\usepackage{hyperref}
\usepackage{algpseudocode}
\usepackage{amsmath}
\usepackage{amssymb}
\usepackage{mathtools}
\usepackage{amsthm}
\usepackage{xspace}
\usepackage{microtype}
\usepackage{graphicx}
\usepackage{subfigure}
\usepackage{booktabs} 
\usepackage[capitalize,noabbrev]{cleveref}
\usepackage{xcolor}
\usepackage{natbib}
\usepackage{bm}
\usepackage{grffile}

\usepackage{authblk}
\usepackage{blindtext}

\theoremstyle{thmstyleone}%
\theoremstyle{thmstyletwo}%
\newtheorem{example}{Example}%

\theoremstyle{thmstylethree}%

\title{Time-Constrained Learning}

\author[1]{Sergio Filho  }
\author[1]{Eduardo Laber} 
\author[1]{Pedro Lazera} 
\author[1]{Marco Molinaro} 

\affil[1]{Computer Science Department, PUC-RIO}
\affil[ ] {\{sfilhofreitas,eduardo.laber1,molinaro.marco,pedrolazera\}@gmail.com}

\begin{document}




\newtheorem{thm}{Theorem}
\newtheorem{cor}{Corollary}
\newtheorem{prop}{Proposition}
\newtheorem{fact}{Fact}
\newtheorem{claim}{Claim}
\newtheorem{obs}{Observation}
\newtheorem{question}{Question}
\newcommand{\remove}[1]{}
\newcommand{\ex}[1]{\bigskip \noindent {\bfseries #1.}}
\newcommand{\exxB}[2]{\bigskip \noindent {\bfseries #1.} {(\bfseries #2)}}
\newcommand{\exx}[1]{\bigskip \noindent {(\bfseries #1)}}
\newcommand{\E}[0]{\ensuremath \mathbb{E}}
\newcommand{\R}[0]{\ensuremath \mathbb{R}}
\newcommand{\e}{\varepsilon}
\newcommand{\ip}[2]{\langle #1, #2 \rangle}
\newcommand{\cx}{\mathrm{cx}}
\newcommand{\F}{\mathcal{F}}
\newcommand{\G}{\mathcal{G}}
\newcommand{\Var}{\mathrm{Var}}
\newcommand{\ones}{\bm{1}}
\newcommand{\blue}[1]{\textcolor{blue}{#1}}
\newcommand{\red}[1]{{\color{red} #1}}
\newcommand{\magenta}[1]{{\color{magenta} #1}}
\newcommand{\cL}{\mathcal{L}}
\newcommand{\cX}{\mathcal{X}}
\newcommand{\cY}{\mathcal{Y}}
\newcommand{\OPT}{\textrm{OPT}\xspace}
\newcommand{\mom}[1]{{\left\vert\kern-0.25ex\left\vert\kern-0.25ex\left\vert #1 \right\vert\kern-0.25ex\right\vert\kern-0.25ex\right\vert}}
\newcommand{\overbar}[1]{\mkern 1.5mu\overline{\mkern-1.5mu#1\mkern-1.5mu}\mkern 1.5mu}
\newcommand{\overtilde}[1]{\mkern 1.5mu\widetilde{\mkern-1.5mu#1\mkern-1.5mu}\mkern 1.5mu}

\newcommand{\err}{\textup{err}}
\newcommand{\Lw}{\textup{Lw}}
\newcommand{\Rw}{\textup{Rw}}
\newcommand{\ERM}{\textup{PAC}}
\newcommand{\cD}{\mathcal{D}}

\newcommand{\OSCT}{{\tt OSCT}\xspace}

\newcommand{\TCT}{{\tt TCT}\xspace}
\newcommand{\TCTbase}{{\tt TCTbase}\xspace}
\newcommand{\Train}{\textrm{Train}}
\newcommand{\PACt}{{\tt TBatch}\xspace}

\renewcommand{\S}{\mathcal{S}}
\newcommand{\cE}{\mathcal{E}}
\newcommand{\bad}{\textrm{Bad}}
\newcommand{\cA}{\mathcal{A}}
\newcommand{\cH}{\mathcal{H}}
\newcommand{\wrong}{\textrm{wrong}}
\newcommand{\D}{\mathcal{D}}

\newcommand{\TS}{\mathcal{TS}}
\newcommand{\iweight}{\omega}
\newcommand{\basic}{\mathcal{A}_{\textrm{base}}}
\newcommand{\agno}{\mathcal{A}_{\textrm{agno}}}
\newcommand{\rand}{\mathcal{A}_{\textrm{rand}}}

\newcommand{\improved}{\mathcal{A}_{\textrm{err}}}
\newcommand{\closer}{\mathcal{A}_{\textrm{close}}^{\alpha}}
\newcommand{\agnocloser}{\mathcal{A}_{\textrm{ag-close}}^{\alpha}}

\newcommand{\learner}{Learner\xspace}
\newcommand{\teacher}{Teacher\xspace}
\newcommand{\full}{{\tt Full}}

\newcommand{\mnote}[1]{{\color{magenta} MARCO: #1}}
\newcommand{\replace}[2]{{\color{magenta} #1} {\color{gray} \sout{#2}}}

\maketitle

\begin{abstract}
Consider a scenario in which we have a huge
labeled dataset ${\cal D}$ and  a limited time to train some given learner   using ${\cal D}$.
 Since we may not be able to use the whole dataset, how should we proceed?
Questions of this nature motivate the definition of the  Time Constrained Learning Task (TCL): 
 Given a dataset ${\cal D}$ sampled from an unknown distribution $\mu$, a learner ${\cal L}$  and a time limit $T$,
 the goal is to obtain in at most $T$ units of time the classification model with highest
 possible accuracy w.r.t. to $\mu$, among those that can be built by ${\cal L}$  using the dataset ${\cal D}$.

We propose \TCT, an algorithm for the TCL task designed based that on principles from Machine Teaching. 
We present  an experimental study involving 5 different Learners and 20 datasets where we show that 
\TCT consistently outperforms two other algorithms: the first is a
Teacher for black-box learners proposed in [Dasgupta et al., ICML 19]
and the second is a natural adaptation of random sampling for the TCL setting. 
We also compare \TCT with 
{\tt Stochastic Gradient Descent} training -- our method is again consistently better.

While our work is primarily practical, we also show that a stripped-down version of \TCT has  provable guarantees.
Under reasonable assumptions, the time our algorithm takes to achieve a certain accuracy is never much bigger than the time it takes the batch teacher (which sends a single batch of examples) to achieve similar accuracy, and in some case it is almost exponentially better.
\end{abstract}

\section{Introduction} 
A common problem that arises in many supervised machine learning applications is 
the difficulty of acquiring labeled data. To overcome this problem techniques as active learning and 
semi-supervised learning have been successfully employed. 

However, in other situations one faces the complementary scenario where  a large number of 
labeled examples are available but the computational resources to train a model over them are limited.
That may arise when one has a limited financial budget to train a model on a  cloud; in this case the limited budget naturally
translates into a  time limit.
In another setting, very current, one may set a limit on the training time to reduce the  environmental impact. There are also applications such as ad advertisement and learning from search logs where the number of labeled examples is massive and repeated training is necessary to track new user behaviour distributions, so that  training
with all labeled examples is not feasible.

For a more concrete situation, assume that  our financial budget only allows  4 hours of cloud usage but we want to train a random forest on
a huge set (e.g. billions) of examples.  Given that we may not be able to use the whole dataset, how should we proceed?


Questions of this nature are faced by machine learning practitioners.
This is the  motivation behind the Time Constrained Learning task (TCL for short), which is the focus of our work.
We consider the  following formulation for TCL:

\medskip

{\em Input.} A dataset ${\cal D}$ of labeled examples sampled from an unknown probability distribution $\mu$, a learner ${\cal L}$  and a time limit $T$

\medskip

{\em Output.} The most accurate classification model w.r.t. $\mu$
that can be built by Learner ${\cal L}$, spending at most $T$ time units.

\medskip

TCL admits two natural variations: one in which we have a considerable amount of information about the Learner (e.g. its hypothesis class and training/classification time complexity) and the other in which the information is limited (black-box Learners). We are interested in the latter due to its wider applicability. Given the lack of information about the Learner, the following question arises:
\begin{center}
{\em Is it possible to outperform random sampling (or some natural variation) for TCL?} 
\end{center}

To give a positive answer to this question,  we approach the TCL task 
 via 
a Machine Teaching framework \cite{journals/ngc/Shinohara91,DBLP:journals/corr/abs-1801-05927}
where 
a \teacher and a \learner interact over multiple  rounds and in each of them the former sends selected examples to the latter, which returns a trained model. 
Our goal is then to design a teacher that guides the Learner towards a model with high accuracy w.r.t. $\mu$ in a way that is as time efficient as possible.



 
  While most of the initial  works on Machine Teaching  \citep{journals/ngc/Shinohara91,journals/jcss/GoldmanK95} assume that the \teacher has significant knowledge about the \learner, there have been several recent advances on the case of interest where there is limited knowledge about the \learner \citep{DBLP:conf/ijcai/MeloGL18,conf/icml/LiuDLLRS18,conf/icml/Dasgupta0PZ19,DBLP:conf/icml/CicaleseFLM20,DBLP:conf/ijcai/DevidzeMH0S20}.


However, these methods do not exactly address the TCL task and instead focus on minimizing the number of samples 
sent to the Learner, that is, in the {\em teaching set.} 
 While the size of the teaching set and time complexity are related, there are factors, 
such as training time and the amount of interaction
with the Learner,
that should be considered when the latter is taken into account. This aspect is illustrated in the beginning of Section \ref{sec:EfficientTraining}.
Indeed, methods that disregard these factors are not suitable for TCL, 
as indicated by our experiments.
Thus, new methods shall be developed to properly handle time constrained learning.

\subsection{Our contributions.}
Our main contribution is the algorithm Time Constrained Teacher (\TCT) for the Time Constrained Learning task that is designed based on well-established  principles/ideas from Computational Learning Theory
and Machine Teaching. 
We compare {\tt TCT} with two other Teachers,  using 5 different Learners over 20 datasets. 
The first Teacher  can be viewed as an adaptation of random sampling for
the TCL task and, thus, is a natural baseline.
The second, denoted here by \OSCT, is based on  an algorithm for  the Online Set Covering problem \citep{journals/siamcomp/AlonAABN09}. Its use for teaching black-box learners (aiming at minimizing the size of the teaching set) was proposed in  \citep{conf/icml/Dasgupta0PZ19}
and refined/extended in \citep{DBLP:conf/icml/CicaleseFLM20}. 

For two of the Learners,  namely SVM and Logistic Regression, we also compare the error achieved by {\tt TCT}
with that achieved by training them  via Stochastic Gradient Descent (SGD).

Our algorithm \TCT consistently outperforms its {competitors and, for some Learners, perhaps surprisingly, \TCT is competitive even against the ``oracle'' teacher that  sends (in one single batch) $m^*$ random samples to the Learner, where $m^*$ (which is guessed
by the oracle)   is the largest number of examples that the Learner can handle within the time limit.
A nice feature of our algorithm is its simplicity:  it just employs one parameter that is easy to set and it does not require any information about the learner. Hence, we believe that it
could be easily implemented as a wrapper in machine learning libraries to address
Time Constrained Learning -- the user provides the time limit and the
Learner that she is comfortable with, and then \TCT completes the job.

While our work is primarily practical, we also show 
that a stripped-down version of \TCT has  provable guarantees.  Under reasonable assumptions, the time our algorithm takes to achieve a certain accuracy is never much bigger than the time it takes the batch teacher (which sends a single batch of examples) to achieve similar accuracy.  Moreover, for learning a threshold function, a canonical problem in Active Learning and Machine Teaching, our algorithm is almost exponentially faster, despite not being tailored to this hypothesis class.



\subsection{Related work.}

Although some Learners admit online versions that can be employed to   TCL task,
we are not aware of works that directly address this task for black-box Learners.
In \citep{conf/icml/Dasgupta0PZ19}, this task is mentioned as a potential application 
for  their proposed Teacher.
 \citep{DBLP:conf/sdm/DuL11}    mention  the possibility of minimizing the training time
rather than the size of the teaching set, although it handles the latter.


On the other hand, there are quite a few papers aiming at minimizing the size of the teaching set. 
Among these works, we can find some that  consider the batch setting \cite{conf/icml/SinglaBBKK14,DBLP:conf/aistats/MaNRZ018} and others that consider the sequential one \citep{DBLP:conf/sdm/DuL11,conf/icml/LiuDLLRS18,NIPS2018_7421}. 
We can also find Teachers that assume a considerable amount of information about the Learner \citep{conf/icml/SinglaBBKK14} as well as some that require very limited information \citep{DBLP:conf/sdm/DuL11,conf/icml/Dasgupta0PZ19,DBLP:conf/icml/CicaleseFLM20}

\remove{These works can be split according to whether they consider the batch setting \citep{conf/icml/SinglaBBKK14,DBLP:conf/aistats/MaNRZ018} or the iterative one \citep{DBLP:conf/sdm/DuL11,conf/icml/LiuDLLRS18,NIPS2018_7421}. They can also be split by taking into account  
whether they assume a considerable amount of information about the Learner \citep{conf/icml/SinglaBBKK14} 
 or very limited information \citep{DBLP:conf/sdm/DuL11,conf/icml/Dasgupta0PZ19,DBLP:conf/icml/CicaleseFLM20}.\mnote{Alguns desses artigos aparecem em uma forma de fazer o split mas nao na outra. Fica esquisito?}
}
  
Most of these Teachers do not admit a simple adaptation for the TCL task 
because they require a lot of information about the Learner's hypothesis class or training algorithm (e.g., \citep{conf/icml/LiuDLLRS18} considers learners that use SGD for training).
 One exception is the method   
\OSCT, from \cite{conf/icml/Dasgupta0PZ19} and \cite{DBLP:conf/icml/CicaleseFLM20}, that assumes very limited information about the Learner. Experiments from the latter,   using 
{\tt LGBM} and {\tt Random Forest} as Learners,  
show that \OSCT requires significantly fewer examples to reach a given 
accuracy (on the training set) than a Teacher that sends random examples.
Despite these gains in \emph{number of examples}, on \emph{time constrained} learning \OSCT performs much worse than our algorithm \TCT, and is even worse than random sampling,  as we show in our experiments.

 Our Time Constrained Learning scenario can be related
to the traditional Active Learning scenario \cite{settles2009active}.
On one hand, the scenarios are quite different: in Active Learning, labeled examples are the constrained resource, while in our setting  they are abundant. 
On the other hand, we cannot use
all the available labeled examples and, thus,
we need to choose (carefully) the ones to be used for training; this is related to the
key aspect of  Active Learning, that is,
the selection of informative examples.
Thus,  Active Learning strategies could be used 
to select examples for Time-Constrained Learning. One issue here is that classical strategies, 
e.g. uncertainty sampling,
may require assumptions about the Learner (e.g. 
the ability to produce class probabilities), which is not aligned with our goal of training {\bf black-box} learners. One strategy that does not need these assumptions 
is Query by Committee, but it is too time-consuming for our setting. In Section \ref{sec:experiments}, we present  experiments
where we evaluate a variation of 
 \TCT that uses ideas from
 a classical 
active learning strategy to select examples.

\remove{In any event,
mixing ideas of active learning with those
proposed here seems to be a promising direction 
as indicated by the preliminary experiments that can be found in Section 
\ref{sec:active-learning}.
}
\remove{
In Section \ref{sec:experiments}, we present  experiments
where the an active learning strategy is combined with the \TCT framework. For some Learners we observed some gains while for other some loss, so that we do not have conclusive results.} 

\remove{

This extra information allows strategies
that are not \blue{pertinent} in the active learning setting. Moreover, 
some
active learning strategies, as uncertainty sampling,
may require assumptions about the Learner (e.g. 
producing class probabilities)
that are not required by TCT -- this requirement is not aligned with our goal of training {\bf black-box} learners.
Query by Committee, on the other hand,
does not need these assumptions but it seems to be too time-consuming for our setting.

Another issue is that coming up with an active learning baseline for the time-constraint setting does involve several design choices, as frequency of retraining, strategy for preventing selection bias, number of samples
from which we select the most informative ones. The point here is that we might end up with a variation of TCT in which examples are selected based on some active learning strategy. During the rebuttal we run preliminary experiments on our datasets with a variation of TCT in which we replace
 ``wrong examples'' \blue{(lines 15-16)} by examples selected via  
uncertainty sampling (smallest margin)  --
the use of wrong examples yielded  much better results
for LGBM and Random Forest and worse results for
Logist. Regression. For decision trees
\blue{they were comparable}
while for SVM class. probabilities are not directly available.

That said, we believe that combining ideas from active learning as uncertainty sampling with those  employed in the paper (doubling, model selection via statistical tests and mixing wrong and random examples) is an interesting and promising direction for future research.
This will be discussed in a Conclusions section that will be added to the paper.

}

 \remove{
Some of these works report experiments where   the proposed  Teachers require  fewer examples than  random sampling  
 \citep{DBLP:conf/sdm/DuL11,conf/icml/SinglaBBKK14,DBLP:conf/icml/CicaleseFLM20}. 
In particular, the
experiments from \citep{DBLP:conf/icml/CicaleseFLM20},   using 
{\tt LGBM} and {\tt Random Forest} as Learners,  
show that \OSCT requires significantly fewer examples to reach a given 
accuracy (on the training set) than a Teacher that sends random examples.
Despite these gains, \OSCT spends a significant amount of time training and classifying and, hence, it performs
much worse than random sampling  for Time Constrained Learning,  as we show in our experiments.   
}



\section{A Teacher for the Time Constrained Learning task}
\label{sec:EfficientTraining}

In this section we describe our algorithm \TCT. Its 
 design takes into account the following  principles:

\begin{enumerate}
\item[P1] The larger the number of examples, the better the learning;
\item[P2] Examples where a Learner fails are more helpful to improve its accuracy   than those in which
it succeeds;
\item[P3] Examples selected to train a Learner should follow approximately the probability distribution $\mu$,
employed in TCL's definition.	
\end{enumerate}

Principle 1 is very natural and justified by both empirical studies and standard  statistical guarantees.
However, to obtain a model trained on a large number of examples by a time limit $T$, it is important not to send too few examples in each round, as  illustrated below.

\begin{example} Consider a learner with quadratic running time, that is, $\Theta(m^2)$ time units are required to train a model when $m$ examples are available.
If the teacher always sends one single example per round, then the largest model will have
$O(\sqrt[3]{T})$ examples, where $T$ is the time limit. On the other hand, if the teacher always doubles the size of the  set
of examples sent to the learner, then the largest model will have $\Omega(\sqrt{T})$ examples.
\end{example}

We note that in contrast to this idea, previous works on Machine Teaching that focus on minimizing the size of the teaching set
suggest the addition of few examples per round \citep{DBLP:conf/sdm/DuL11,conf/icml/Dasgupta0PZ19,DBLP:conf/icml/CicaleseFLM20}.

Principle 2 is motivated
by human learning and also by works on Machine Teaching  \citep{conf/icml/Dasgupta0PZ19,DBLP:conf/icml/CicaleseFLM20} and  boosting \citep{DBLP:conf/ijcai/Schapire99}, where  wrong examples get larger priority than those in which the Learner succeeds.
While wrong examples are helpful, their acquisition may be expensive in terms of computational time  since finding them may require
the classification of  a large number of examples. Thus, we need to balance between the usefulness of having wrong examples and the computational
cost to get them. Moreover, we have to add wrong examples with parsimony, otherwise we may build models using a set of examples that is not representative of the real population and hence with potential poor performance on unseen data. This is captured in Principle 3. Indeed, issues with using biased samples for learning is well-documented in active learning~\citep{dasguptaTwoFaces}. 

%

Based on these observations, we designed \TCT, presented in 
Algorithm \ref{alg:Teacher}. It receives an integer $m_0$ (number of examples for the first model), a learner $L$,
a parameter $\alpha \in [0,1]$ that defines the ratio between wrong and random examples
provided to the learner at each round; a time limit $T$ and a pool $P$
of labeled examples. We note that $\alpha$ allows a trade-off between following Principles 2 and 3.
For the ease of presentation we assume that the pool $P$ is large enough
so that it is always possible to obtain unselected examples from it.

At each round, {\tt TCT} receives from the learner  a model $M$,  trained on a set
$S$, and then builds and provides to the learner a new set of examples  that contains (approximately) $\alpha|S|$ examples in which $M$ fails and also $(1-\alpha)|S|$ random unseen examples from $P$.
 More precisely, {\tt TCT}
 first employs $M$ 
to classify  a random set $A_1$ containing $|S|$ examples (Lines \ref{Class11} -\ref{Class12}). As a result, it also obtains an unbiased estimation $acc_1$ for the accuracy of $M$.
Next, based on $acc_1$ {\tt TCT}  builds a second set of samples $A_2$, which is done to guarantee that in expectation we have at least $\alpha |S|$ examples in $A_1 \cup A_2$ where $M$ fails. It then classifies this set $A_2$ using $M$ (Lines \ref{Class14}-\ref{Class15}), obtaining a second unbiased estimation for the accuracy of $M$,
which is combined with $acc_1$ to evaluate whether the current model is better than the best so far
(Lines \ref{BestModelStart}-\ref{BestModelEnd}). We note that the {\tt CurrentEstimator} at Line \ref{confid95}
corresponds to the lower limit of the $95\%$ confidence interval of $acc$.  
To update $S$, {\tt TCT} first adds a set $U$ containing   $(1-\alpha)|S|$ random examples from $A_1$.
Then, it adds   $\alpha |S|$ examples from $A_2 \cup (A_1 \setminus U)$, 
prioritizing the wrong ones.

\remove{
At the beginning of each round, Teacher calls  the Learner to train a new model $M$
using a set of examples $S$ (line \ref{Train}). 
Then,  Teacher uses $M$ to classify a batch of random examples and,
next, selects some of them to build the new set of examples that will
be employed by Learner in the next round (lines \ref{Class11} - \ref{NewSet}). More precisely,
 Teacher first employs $M$ 
to classify  a random set of $|S|$ examples (lines \ref{Class11}-\ref{Class12}). As a result, it obtains
an unbiased estimation $acc_1$ for the accuracy of $M$ as well as the set of examples for which $M$ fails.
By using $acc_1$, it defines  the number of extra examples that will be classified
in order to obtain the desired balance (according to $\alpha$) between wrong and random examples (lines \ref{Class21}-\ref{Class22}).
This additional classification yields a second unbiased estimation for the accuracy of $M$,
which is combined with $acc_1$ to evaluate whether the current model is better than the best so far
(lines \ref{BestModelStart}-\ref{BestModelEnd}). 

}


\begin{algorithm}
  \caption{{\tt TCT} ($m_0$: integer; $L$: Learner; $\alpha$: real parameter; $P$: pool of examples; $T$: time limit )}
  \begin{algorithmic}[1]
  	\small

    \State $S \leftarrow$ set of $m_0$ random examples from $P$ \label{line1}
  
     \Repeat
        \State $M \leftarrow $ model trained by learner $L$  on set $S$ \label{Train}
        
        \State $A_1 \leftarrow  $ set of $|S|$ random examples from $P$ that have not been selected so far  \label{Class11} 

       \State $acc_1 \leftarrow $ Classify$(M,A_1)$ \label{Class12}

   
       \State $A_2 \leftarrow  $ set of  $ \alpha |S|acc_1 /(1-acc_1)$ random examples  from $P$ that have not been selected so far  \label{Class14} 

        \State $acc_2 \leftarrow $ Classify$(M,A_2)$  \label{Class15}

           \State $acc \leftarrow (acc_1 |A_1| + acc_2|A_2|) / ( |A_1|+|A_2|)$
                      
        \State  {\tt CurrentEstimator}  $\leftarrow acc - 1.96 \sqrt{\frac{acc(1-acc)}{|A_1|+|A_2|}} $ \label{confid95} \label{conf-bound}
        

         \If{{\tt ElapsedTime} $\le  T$ AND  (first round OR  {\tt CurrentEstimator} $>$ {\tt BestEstimator})} \label{BestModelStart}
              \State {\tt BestModel} $\leftarrow M$
              \State {\tt BestEstimator} $\leftarrow$ {\tt CurrentEstimator} 
         \EndIf  \label{BestModelEnd}

       \State $U \leftarrow $ set of $(1-\alpha)|S|$ random  examples from $A_1$  \label{Class13}

        \State $V \leftarrow $ list of examples in $A_2 \cup (A_1 \setminus U)$ with the 
        wrong ones (w.r.t. $M$) appearing before the correct ones. 
        
        \State $W \leftarrow \alpha|S|$ first examples from $V$
      \label{Class21} 
         
        \State $ S \leftarrow S  \cup  U \cup W$ \label{NewSet}
                       
     \Until{{\tt ElapsedTime} $\ge  T$}

     \State Return {\tt BestModel}. 

  \end{algorithmic}
  \label{alg:Teacher}
\end{algorithm}

Some observations are in order:
\begin{itemize}

\item If the parameter $\alpha$ is very small, then {\tt TCT} 
becomes similar to pure random sampling. In contrast, if  $\alpha$ is large,
the Learner is guided to learn a model that relies on a distribution that may be significantly different  from the real 
one, which is not in line with Principle 3. Indeed, we present an example in Appendix 
\ifSUPP \ref{app:badExample},
\else
B.1
\fi
where using a large $\alpha$ is problematic.
  In addition, a large value of $\alpha$ may have a
negative impact on the running time due to the Lines~\ref{Class14} and \ref{Class15}.
 The results from our experiments suggest that $\alpha=0.2$  works well in practice  and also that the method
is robust to moderate variations of this parameter.

\item \TCT is more suitable for the typical scenario where the  classification time per example is (much) smaller than the training 
time per example. For scenarios where this property does not hold, \TCT may spend a large amount of time classifying examples and, hence, end up with a model trained on a relatively small set.      


\item 
Since  {\tt TCT}  classifies 
several new examples in each round,   it obtains, at no
additional cost, an unbiased estimation of the real accuracy of $M$. 
Thus, it is reasonable to
use this estimation for selecting the model  (Lines \ref{BestModelStart}-\ref{BestModelEnd}). 
The reason why  {\tt TCT} picks the lower bound
of the 95$\%$ confidence interval (line \ref{conf-bound})  rather than the estimated accuracy is because in
the first rounds it uses few examples and, as a consequence, the variance
of the estimation is high.

\item We have assumed that the pool $P$ is large enough so that we can always sample examples that have not been considered so far (Lines \ref{Class11} and \ref{Class14}). In practice, this does not necessarily occur. In that case, when {\tt TCT}
reaches the point in which all the examples have already been classified, it  starts
to use examples that have not been added to set $S$ yet.


\end{itemize}



\section{Experimental Study}
\label{sec:experiments}

In out first set of experiments, we compare the algorithm \TCT with two others teachers, namely  {\tt Double} and \OSCT
\citep{conf/icml/Dasgupta0PZ19,DBLP:conf/icml/CicaleseFLM20}
that are briefly described below. 

\begin{itemize}

\item  {\tt Double} is a  Teacher that at each round $i$
sends  new $m_0 2^i$ randomly  selected examples to the Learner,
where $m_0$ is  the number of examples used to train the first model.
Next, the Learner returns  a model trained on all  examples received so far
and, then,  round $i+1$ starts. The model returned by {\tt Double} is the last one built within the given time limit;

\item  \OSCT keeps a weight for each  example in the training set.
At each round, it receives a new model (hypothesis) from the Learner and uses it to classify all the examples
in the training set. Next, it repeatedly doubles the weights of the wrong examples until their sum exceeds 1. At this point, it samples $O(\log n)$ examples  following a distribution induced by these weights, where $n$
is an estimate of  the number of effective hypotheses in the Learner's class.
The  Learner receives these sampled examples and use them to update its current model.



\end{itemize}

 {\tt Double} can be viewed as an adaptation of random sampling for the TCL task and, thus, we
  understand that it is a very natural baseline.
  \OSCT is a recent  method for teaching black-box learners that focuses
  on minimizing the size of the teaching set. In our experiments we employed the implementation discussed in
  \citep{DBLP:conf/icml/CicaleseFLM20}.
A pseudo-code of the implementation can be found
on Appendix \ref{sec:TCT-OSCT-Additional}

 We also compare $\TCT$ with  
 Stochastic Gradient Descent ({\tt SGD}) based training,  
 a  widely used  strategy for online training. 
For this purpose, we use the class {\tt SGDClassifier} from {\tt Scikit-Learn},
with its default parameters.

To compare \TCT with  {\tt Double} and  \OSCT, we considered 5 Learners:   {\tt LGBM}, {\tt Random Forest}, {\tt Decision Tree}, {\tt SVM} and {\tt Logistic Regression}.  In our comparison with 
{\tt SGD}, we only considered the last two Learners since it is  not clear how to 
(directly) train the others  via  {\tt SGD}.

{\tt LGBM} belongs to the family of Gradient Boosting methods and it is very popular on contests like Kaggle.
{\tt Random Forest} is also widely used. 
We selected {\tt Decision Trees} with small depth as a representative of interpretable methods, while
{\tt SVM} and {\tt Logistic Regression} were selected as representatives of linear classifiers.

 Our learners were implemented using  Python, version 3.8.5, with the libraries numpy (1.20.1), pandas (1.2.2), lightgbm (3.1.1), scikit-learn (0.24.1), scipy (1.6.1).
We use the default parameters for all learners but
for {\tt Decision Trees}, where we
set $min\_samples\_split=30$ and $max\_depth=5$ to build interpretable trees 
and for {\tt Random Forest} where we set
 $min\_samples\_split=30$ to
prevent  very long running time and, hence, limiting our experimental study.
 To obtain a {\tt SVM} and a {\tt Logistic Regression}
classifiers via {\tt SGD} we set the parameter
{\tt loss} from class {\tt SGDClassifier} to {\tt hinge} and {\tt log},
respectively.

We considered  20 datasets in our experiments, whose main features
are shown in Table \ref{tab:Datasets}.
The  experiments from Section \ref{sec:comparison-osct-double}-\ref{sec:sensibility} 
were executed using 
processor Core i9-7900X 3.3GHz, with 128GB RAM DDR4 and  Windows 10, while those from Section \ref{sec:appendix-active-learning}
were executed with the following settings:
Core i7-4790 3.60GHz, 32GB RAM DDR3, Ubuntu 20.04.3 LTS}
For timing we used the method {\tt timeit.default\_timet} from {\tt timeit} library.
Our code can be found at \url{https://github.com/sfilhofreitas/TimeConstrainedLearning}.
More details about the datasets and learners
can be found in Appendix \ref{app:additional}.



\remove{
	
\subsection{Learners}
 \begin{itemize}
	\item SVMLinearLearner: modelo LinearSVC da biblioteca sklearn.svm (dual=False).
	\item RandomForestLearner: modelo RandomForestClassifier da biblioteca sklearn.ensemble (n\_jobs=1, n\_estimators=100, min\_samples\_split=30).
	\item LogisticRegressionLearner: modelo LogisticRegression da biblioteca sklearn.linear\_model (n\_jobs=1, solver=saga).
	\item LGBMLearner: modelo LGBM da biblioteca lightgbm (n\_jobs=1).
	\item DecisionTreeLearner: modelo DecisionTreeClassifier da biblioteca sklearn.tree (min\_samples\_split=30, max\_depth=5).
 \end{itemize}

}

\begin{table}[]
\begin{center}

\caption{Datasets. $m$: size of training set, $d$: number of attributes;  $k$: number of classes}
\label{tab:Datasets}

\begin{tabular}{c|ccc|}
{\bf Dataset}                                 & $\mathbf{m}$     & $\mathbf{d}$    & $\mathbf{k}$    \\ \hline
vehicle\_sensIT               & 68969  & 100  & 2     \\
MiniBooNE                     & 91044  & 50   & 2    \\
SantanderCustomer & 140000 & 200  & 2 \\
BNG\_spambase                 & 699993 & 171  & 2 \\
BNG\_spectf\_test             & 700000 & 44   & 2 \\
Diabetes130US                 & 71236  & 2518 & 3 \\
BNG\_wine                     & 700000 & 13   & 3 \\
jannis                        & 58613  & 54   & 4 \\
BNG\_eucalyptus               & 700000 & 95   & 5 \\
BNG\_satimage                 & 700000 & 36   & 6 \\
 covtype                      & 406708  & 54   & 7 \\
 volkert                      & 40817   & 180  & 10  \\
  cifar\_10                    & 42000   & 3072 & 10   \\
   mnist                        & 60000   & 784  & 10   \\
 BNG\_mfeat\_fourier          & 700000  & 76   & 10   \\   
  poker\_hand                  & 1000000 & 85   & 10   \\  
 Sensorless\_drive & 40956   & 48   & 11   \\
  BNG\_letter\_5000\_1         & 700000  & 16   & 26   \\
 GTSRB-HueHist                & 36287   & 256  & 43   \\  
 aloi                         & 75600   & 128  & 1000  
\end{tabular}
\end{center}
\end{table}

\remove{
\begin{table}[]
\begin{center}

\caption{Datasets. $m$: size of training set, $d$: number of attributes;  $k$: number of classes}
\label{tab:Datasets}

\begin{tabular}{c|ccc||c|ccc}
Dataset                                 & m     & d    & k    & Dataset                                       & m       & d   & k  \\ \hline
GTSRB-HueHist\                & 36287 & 256  & 43   & SantanderCustomerSatisfaction\      & 140000  & 200 & 2  \\
volkert\                      & 40817 & 180  & 10   & covtype\                            & 406708  & 54  & 7  \\
Sensorless\_drive & 40956 & 48   & 11   & BNG\_spambase\ & 699993  & 171 & 2  \\
cifar\_10\                    & 42000 & 3072 & 10   & BNG\_letter\_5000\_1\               & 700000  & 16  & 26 \\
jannis\                       & 58613 & 54   & 4    & BNG\_spectf\_test\                  & 700000  & 44  & 2  \\
mnist\                        & 60000 & 784  & 10   & BNG\_wine\                          & 700000  & 13  & 3  \\
vehicle\_sensIT\              & 68969 & 100  & 2    & BNG\_eucalyptus\                    & 700000  & 95  & 5  \\
Diabetes130US\                & 71236 & 2518 & 3    & BNG\_satimage\                      & 700000  & 36  & 6  \\
aloi\                         & 75600 & 128  & 1000 & BNG\_mfeat\_fourier\                & 700000  & 76  & 10 \\
MiniBooNE\                    & 91044 & 50   & 2    & poker\_hand\                        & 1000000 & 85  & 10
\end{tabular}
\end{center}
\end{table}
}

\remove{
\begin{table}[]
\caption{Datasets and its details}
\label{tab:Datasets}
\begin{center}
\small

\begin{tabular}{c|c|c|c|c}
{\bf Dataset}                      & {\bf \#Training Examples} &{\bf \#Testing Examples}  &{\bf Attributes} & {\bf Classes} \\ \hline \hline
Diabetes130US                      & 71,236     &   30,530           & 2,518                     & 3                     \\
BNG\_letter\_5000\_1               & 700,000    &   300,000          & 16                       & 26                    \\
poker\_hand                        & 1,000,000   &   25,010          & 85                       & 10                    \\
SantanderCustomerSatisfaction      & 140,000    &  60,000           & 200                      & 2                     \\
BNG\_spectf\_test                  & 700,000    &  300,000     & 44                       & 2                     \\
BNG\_wine                          & 700,000    &  300,000    & 13                       & 3                     \\
vehicle\_sensIT                    & 68,969     &  29,559      & 100                      & 2                     \\
MiniBooNE                          & 91,044     &  39,020           & 50                       & 2                     \\
BNG\_eucalyptus                    & 700,000    &  300,000    & 95                       & 5                     \\
mnist                              & 60,000     &  10,000     & 784                      & 10                    \\
BNG\_spambase & 699,993    &  299,997           & 171                      & 2                     \\
covtype                            & 406,708    &  174,304     & 54                       & 7                     \\
cifar\_10                          & 42,000     &  18,000           & 3,072                     & 10                    \\
jannis                             & 58,613     &  25,120           & 54                       & 4                     \\
volkert                            & 40,817     &  17,493           & 180                      & 10                    \\
BNG\_satimage                      & 700,000    &   300,000    & 36                       & 6                     \\
Sensorless\_drive\_diagnosis       & 40,956      &   17,553         & 48                       & 11                    \\
GTSRB-HueHist                      & 36,287      &   15,552         & 256                      & 43                    \\
BNG\_mfeat\_fourier                & 700,000     & 300,000           & 76                       & 10                    \\
aloi                               & 75,600      &  32,400          & 128                      & 1,000                 
\end{tabular}
\end{center}
\end{table}
\normalsize
}

For each combination of dataset $\mathcal{D}$ and Learner $\mathcal{L}$,
we calculated the time $t_{\mathcal{D},\mathcal{L}}$ required by ${\cal L}$ to build a classification model for the whole dataset ${\cal D}$ (averaged over 4 runs) 
and kept the pairs $(\cal D, \cal L)$ for which $t_{\mathcal{D},\mathcal{L}}$ is at least 10 seconds.
We define $Valid({\cal L}):=\{ {\cal D} \mid t_{\mathcal{D},\mathcal{L}} >10\}$.

\remove{{The pairs for which the corresponding time $t_{\mathcal{D},\mathcal{L}}$
 is smaller than 10 seconds were discarded, since
small times are greatly affected by factors that
are not related with the learning process, as operating system  tasks.}\mnote{Removi qualquer explicacao sobre porque remover os pares menores que 10 seg, pra passar batido e tentar nao levantar red flag. Se achar estrategia perigosa, revertemos a mudanca}
We ended up with 68 pairs.}

For each valid combination $({\cal D} , {\cal L})$,
we run each of the teachers to build a classification model within  time limit $t_{\mathcal{D},\mathcal{L}}$.
During these trainings, whenever a new model was obtained we evaluated it on the testing set (pausing the timing), which allowed us  to plot the evolution curves that are presented next.
For all  experiments we set $m_0$, the size of the first  set of examples sent to the learner,
as $0.5\%$ of the size of the full dataset.
In practice, $m_0$ should be set 
as the maximum value for which we are confident
that a training set with $m_0$ labeled examples can be trained
within the time limit.

\subsection{Comparisons with \OSCT and {\tt Double}}
\label{sec:comparison-osct-double}

\newcommand{\addDecTreeTCTvsOSCT}{\includegraphics[width=19em]{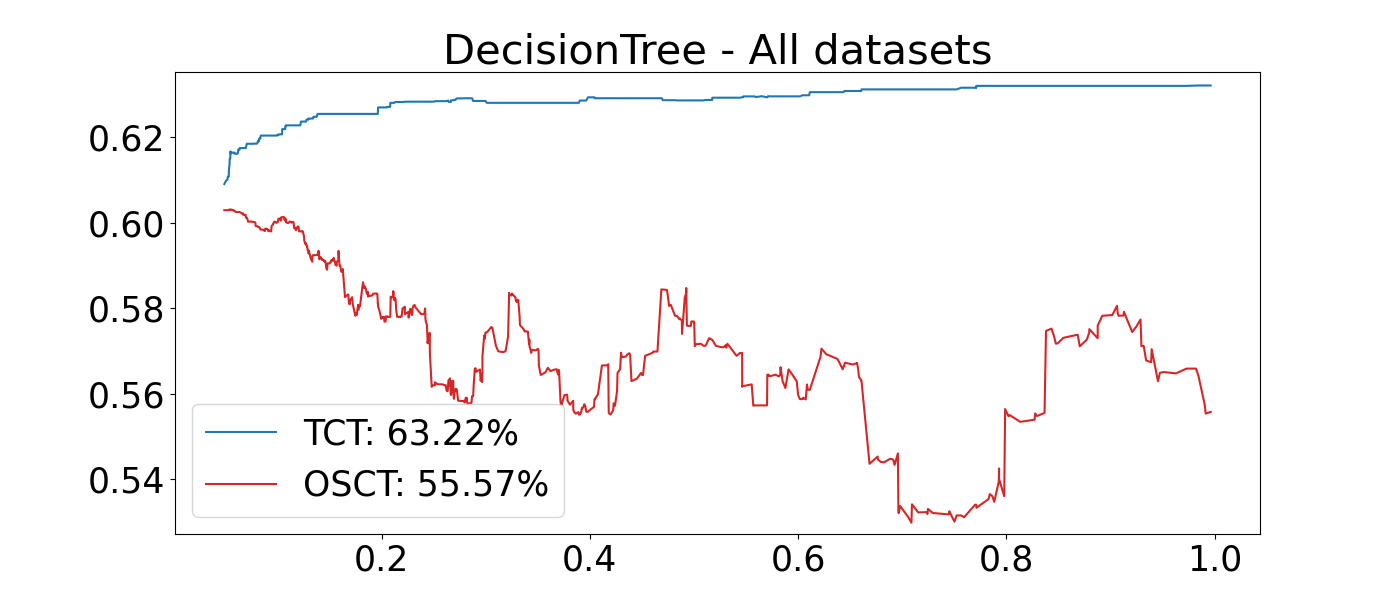}}

\newcommand{\addLGBMTCTvsOSCT}{\includegraphics[width=19em]{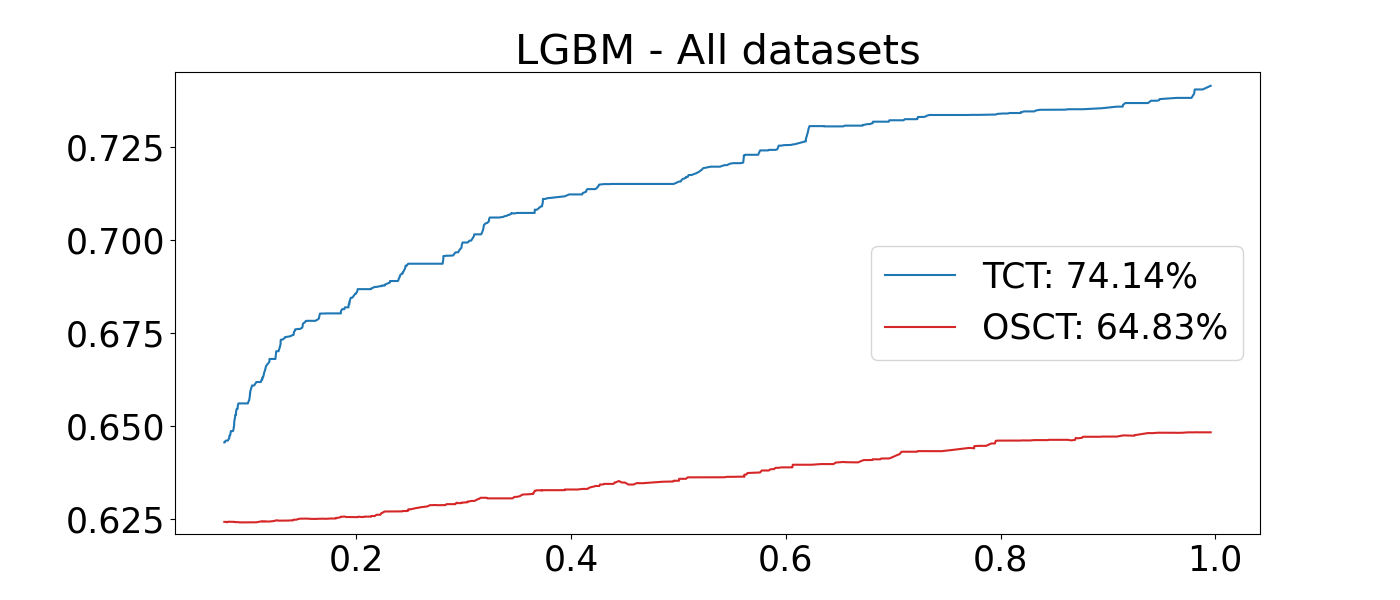}}

\newcommand{\addLogRegTCTvsOSCT}{\includegraphics[width=19em]{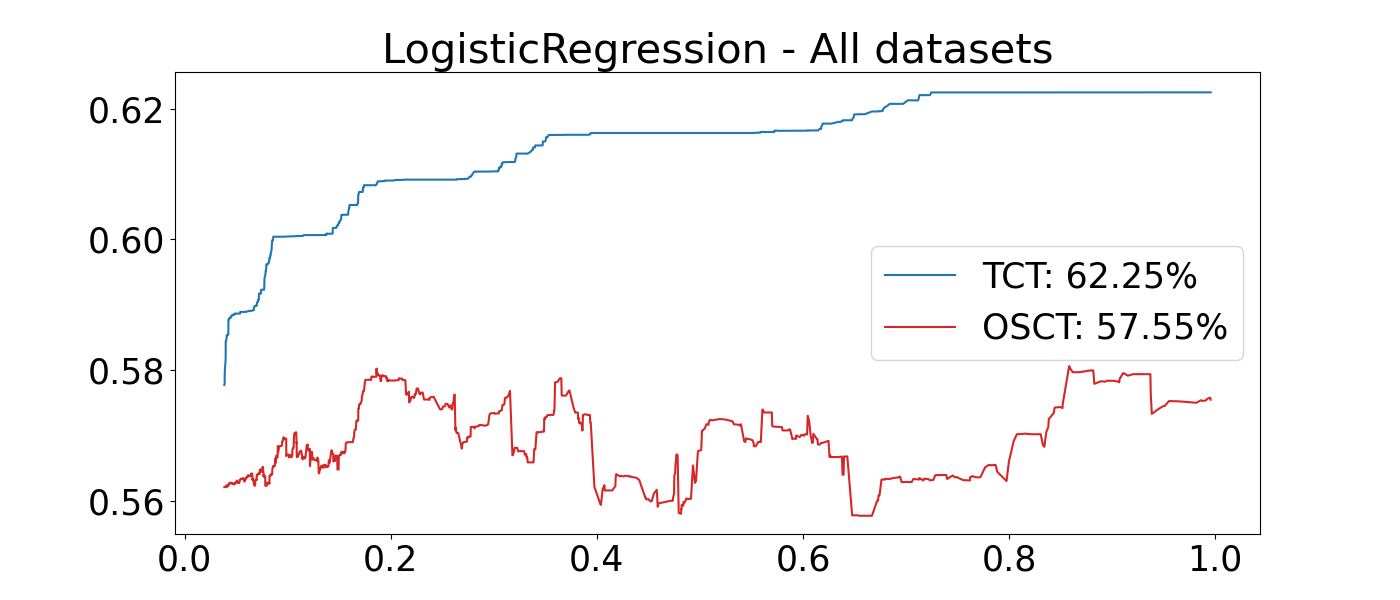}}

\newcommand{\addRandForestTCTvsOSCT}{\includegraphics[width=19em]{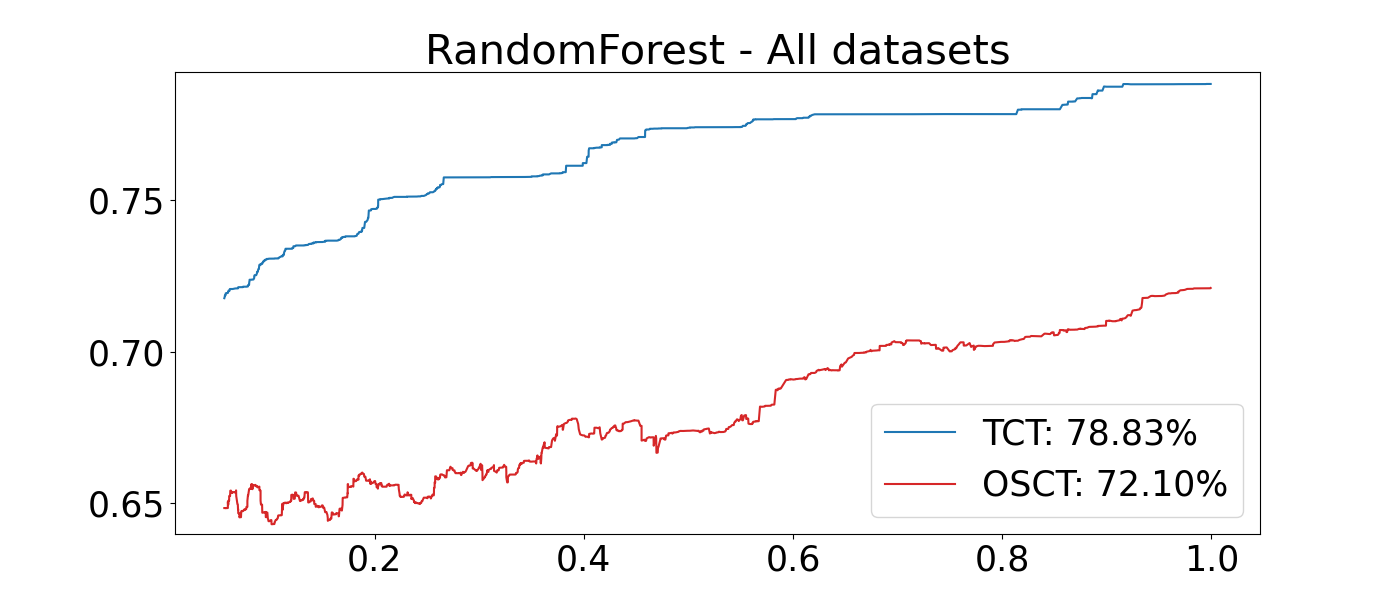}}

\newcommand{\addSVMLinTCTvsOSCT}{\includegraphics[width=19em]{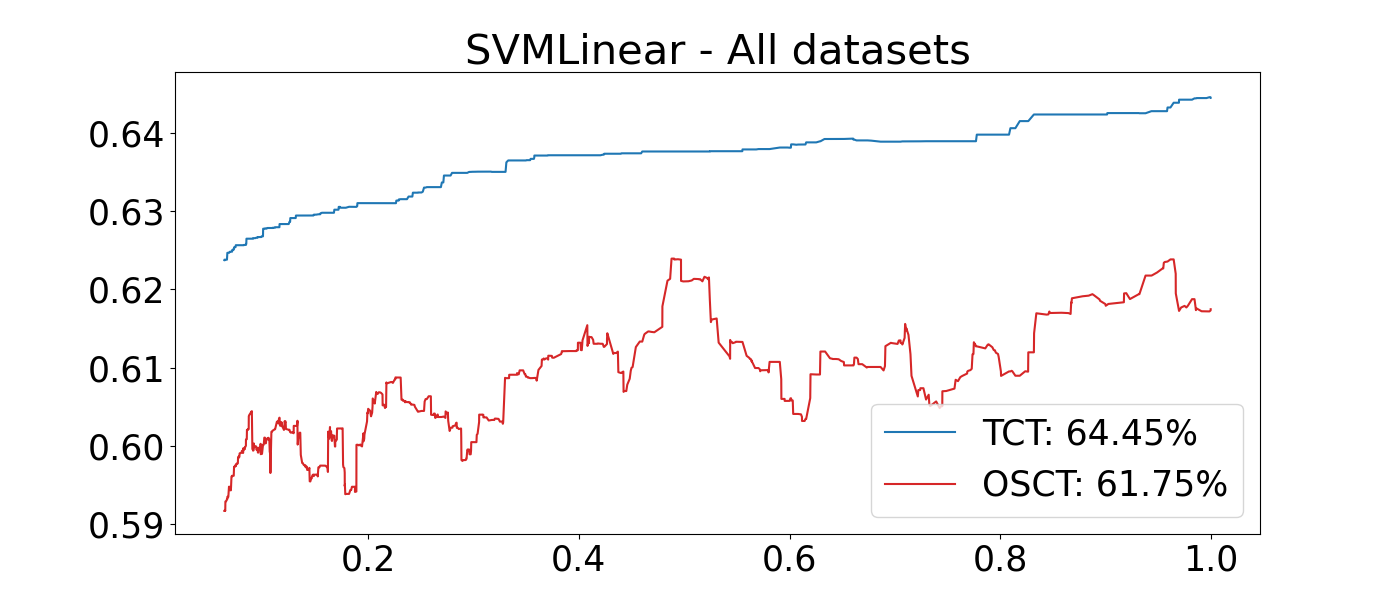}}

\begin{figure}
\begin{center}
\begin{tabular}{lcc}
& \addDecTreeTCTvsOSCT  & \addLogRegTCTvsOSCT    \\ 
&  \addSVMLinTCTvsOSCT & \addLGBMTCTvsOSCT  \\
&  \addRandForestTCTvsOSCT &  \\ 
\end{tabular}
\end{center}
\caption{Average accuracies on testing set along normalized time for \TCT and \OSCT. The numbers next to the labels are their average accuracies at the last normalized time limit $t=1$. The initial guess for $n$ is 2.}
\label{fig:tct-vs-osct}
\end{figure}


\newcommand{\addallSCT}{\includegraphics[width=19em]{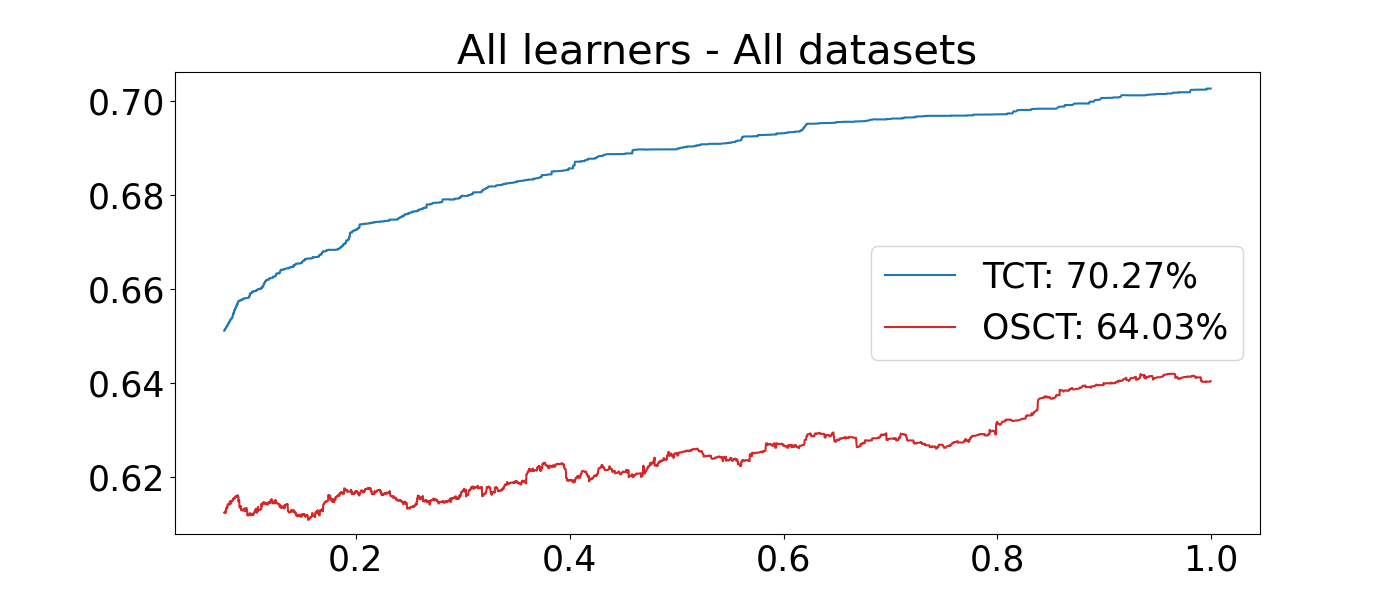}}


\newcommand{\addpic}{\includegraphics[width=21em]{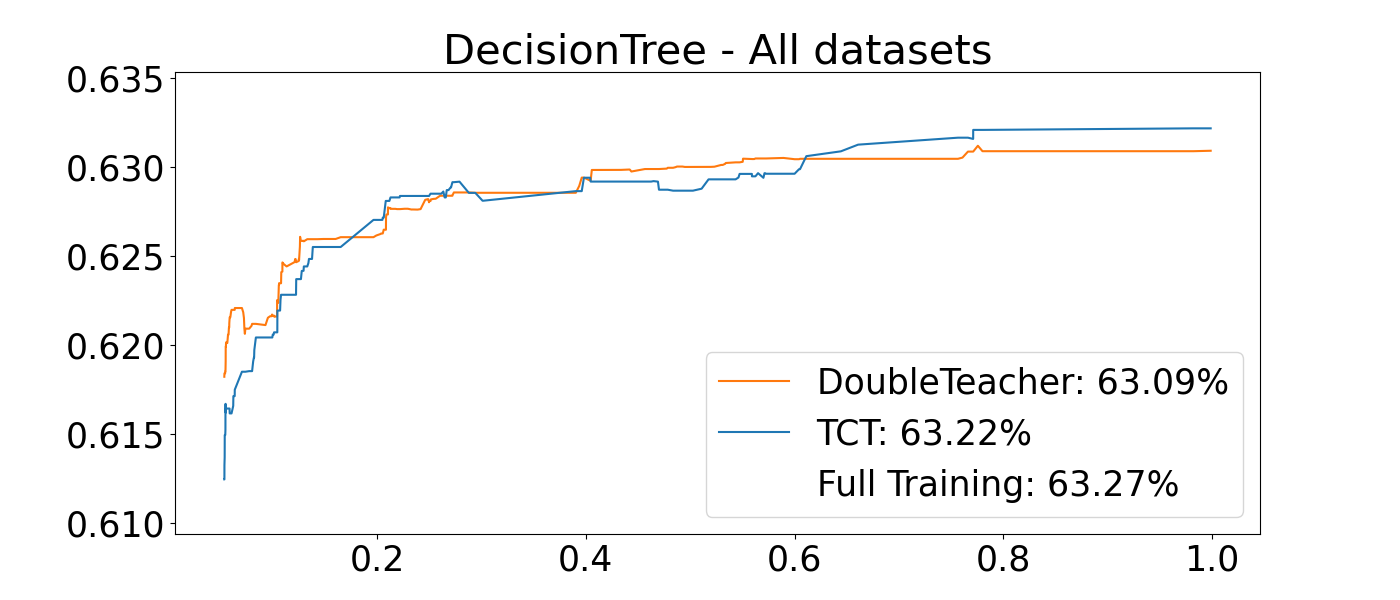}}

\newcommand{\addp}{\includegraphics[width=21em]{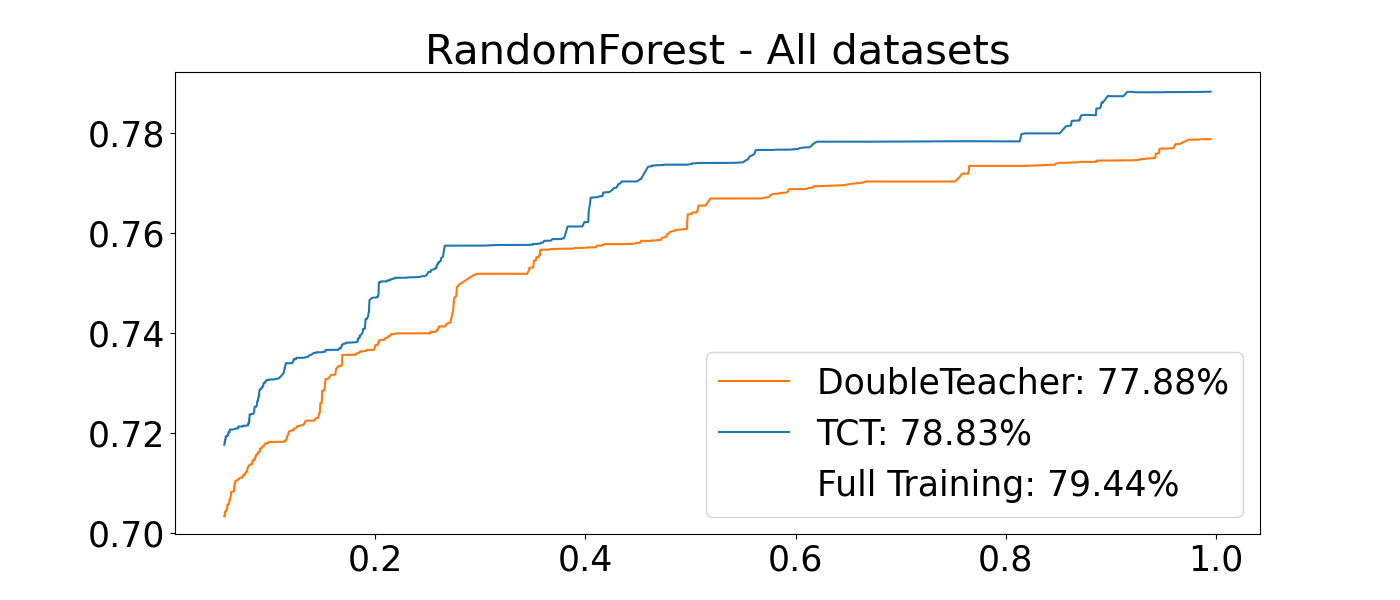}}

\newcommand{\addtwenty}{\includegraphics[width=21em]{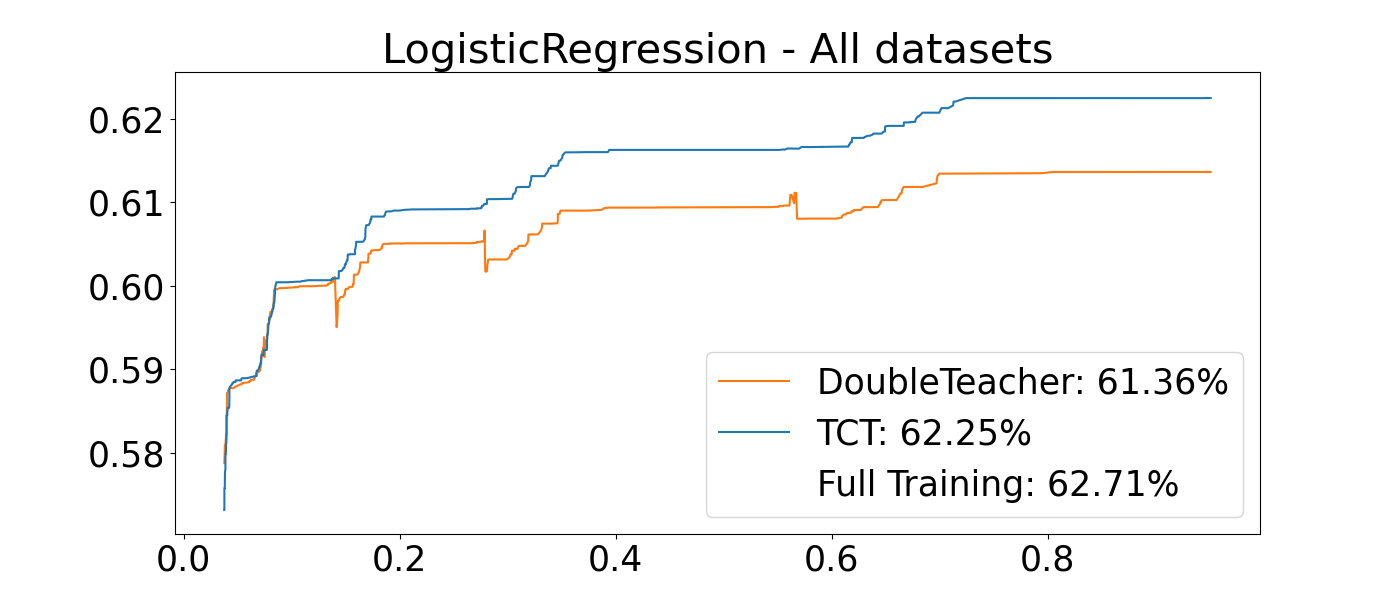}}

\newcommand{ \addthirty}{\includegraphics[width=21em]{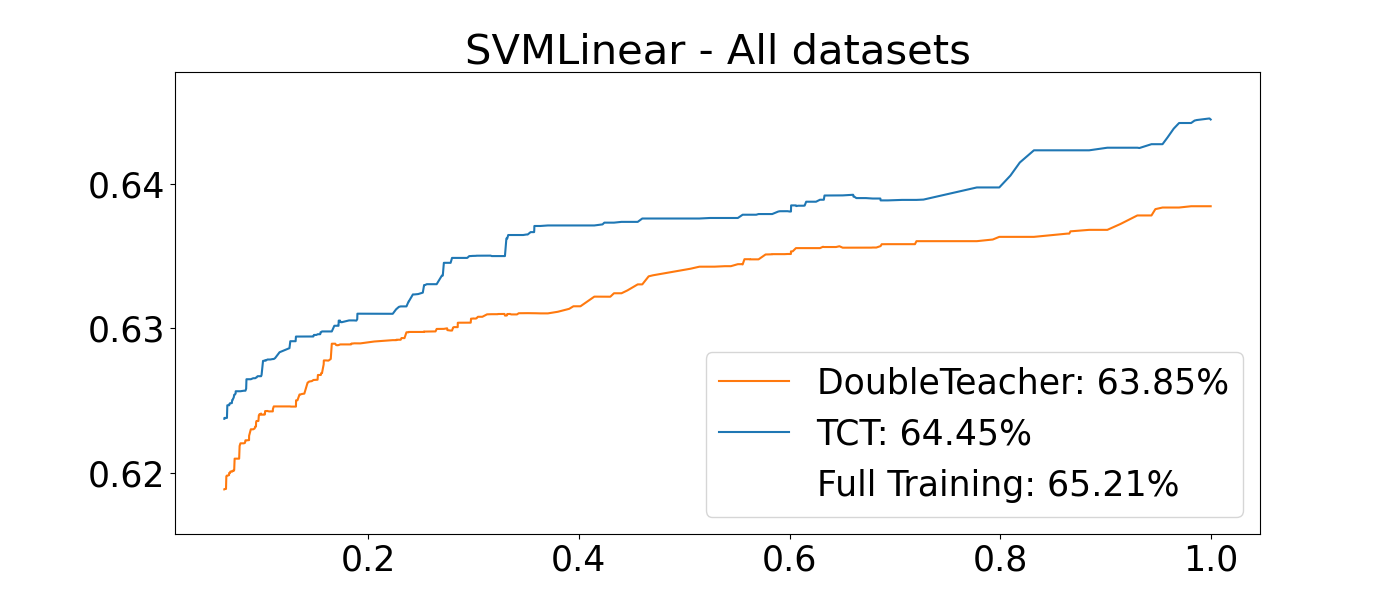}}

\newcommand{ \addLGBM}{\includegraphics[width=21em]{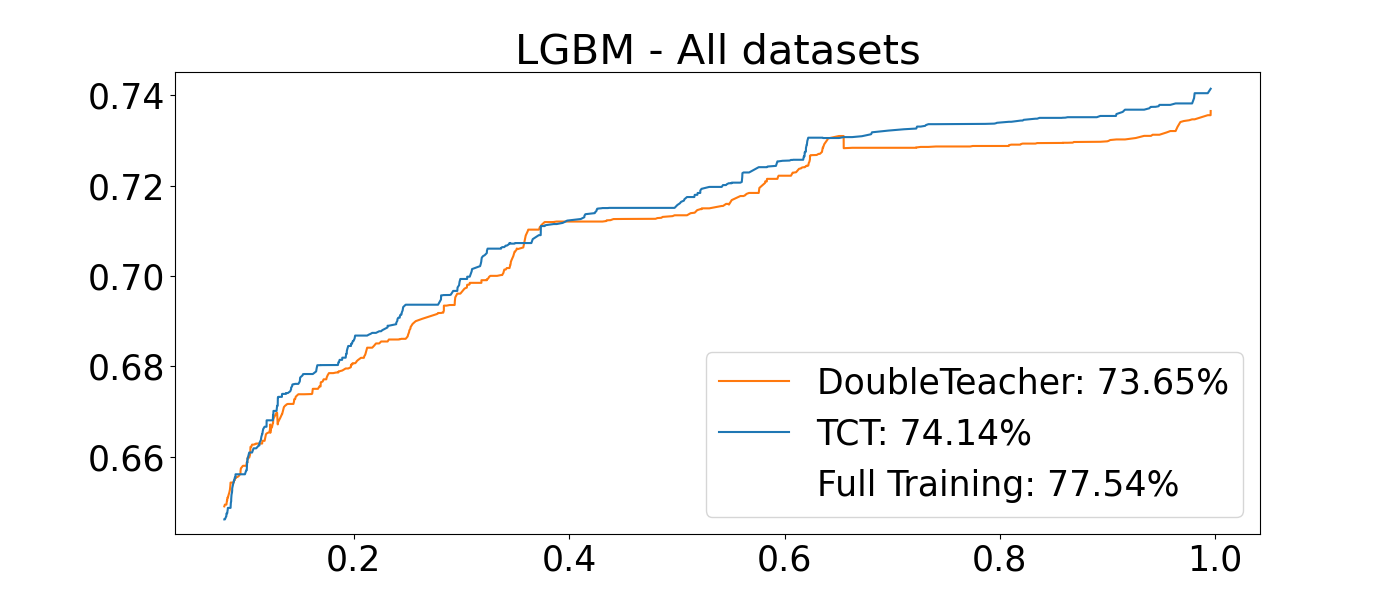}}


\begin{figure*}

\begin{center}
\begin{tabular}{lcc}

&  \addpic    &  \addtwenty  \\
& \addthirty  &  \addLGBM  \\
&  \addp  &   \\ 
\end{tabular}
\end{center}
\caption{Average accuracies on testing set along normalized time for \TCT, {\tt Double} and \OSCT. The numbers next to the labels are their average accuracies at the last normalized time limit $t=1$. }
\label{fig:histograms}
\end{figure*}



Figure \ref{fig:tct-vs-osct} shows 5 images,
 each of them 
corresponding to a comparison between \TCT ($\alpha=0.2$) and {\tt OSCT} for each of the Learners.
The horizontal axis corresponds to  the normalized time limit $t \in [0,1]$ (as a fraction of $t_{\mathcal{D},\mathcal{L}}$), and the vertical axis corresponds to the (average) accuracy on testing sets.
More precisely, for a  Teacher  ${\cal T}$ and a learner
${\cal L}$,  the accuracy  associated with the normalized time  $ t \in [0,1]$ is given by
\begin{equation}
\label{eq:xxx}
 \sum_{{\cal D} \in Valid({\cal L})  } \frac{ acc({\cal T},{\cal L},{\cal D}, t \cdot t_{{\cal D},{\cal L}} )}{|Valid({\cal L})| }, 
 \end{equation}
where $acc({\cal T},{\cal L},{\cal D}, t \cdot t_{{\cal D},{\cal L}} )$ is
the accuracy of the model obtained by Teacher ${\cal T}$, with learner ${\cal L}$, over dataset
${\cal D}$ when the time limit for training is  $ t \cdot t_{{\cal D},{\cal L}}$. 

As an example, consider
a dataset where training an SVM with all datapoints took 1000 secs. Then, Figure 1 "SVMLinear" at $x$-axis equal to 0.6, for example, shows the accuracy of the algorithms \TCT and {\tt OSCT} when given 0.6 · 1000 = 600secs of
 execution time (actually this figure shows the average of such accuracy over all datasets).

We see that \TCT presents a huge advantage compared to \OSCT.
One of the reasons is that  \OSCT does not manage to build models on large training set since
it adds a few examples per round and spends a non-negligible time classifying examples.
We tried variations of \OSCT to improve its accuracy
(see Appendix  \ref{sec:TCT-OSCT-Additional}), but the results did not change significantly.


\remove{
The other  five images present a comparison between \TCT ($\alpha=0.2$) and {\tt Double} for each of the Learners.
The points of the curves are built as in (\ref{eq:xxx}), with the exception that 
$Valid$ is replaced with $Valid({\cal L})$.
}

Figure \ref{fig:histograms} shows a similar comparison
where  {\tt Double}, rather than {\tt OSCT}, is used. 
We observe that {\tt TCT} outperforms {\tt Double} 
for all Learners but {\tt Decision Trees}, where their performances are very similar.
We also note  that at the final time limit $t_{{\cal D}, {\cal L}}$, for most of the Learners, both {\tt TCT} and {\tt Double} are able to reach accuracy comparable to that of training using the whole dataset (label {\tt Full Training}).
The accuracy associated with training on the whole dataset can be thought as what can be achieved by  an ``oracle'' Teacher for {\tt TCL} that knows
beforehand the size of the largest batch of  (random) examples that can be trained within a given time limit 
($t_{{\cal D}, {\cal L}}$ in this case) and sends such batch  to the Learner.   


\remove{We note that this accuracy associated with training on the whole dataset is also the one achieved by a non-implementable Teacher, denoted  by {\tt Oracle},
that obtains from  an oracle the size $m^*$ of the largest dataset that can be trained within
the given time limit and, then, send to the Learner $m^*$ random examples (selected without replacement). 
In our experiments, by design,
$m^*$ is the size of the full training set.
}

Table \ref{tab:Win-Losses0} shows other relevant statistics at the final time limit $t_{{\cal D}, {\cal L}}$. The multi-column
\TCT vs {\tt Double} (resp.  \TCT  vs \full)  
gives, for each Learner, the number of datasets where the classifier built by  {\tt TCT} outperforms  that of {\tt Double} (resp. \full)  with $95\%$ confidence (Section 5.5 of \citep{mitchell}).
As an example, for {\tt LGBM}, {\tt TCT} outperformed ${\tt Double}$ 9 times
and it was outperformed just once.
We observe a clear advantage of \TCT over {\tt Double} for {\tt LGBM}, {\tt Random Forest} and {\tt SVM}.
For {\tt Logistic Regression} and {\tt Decision Trees} these algorithms have similar performance. Perhaps surprisingly, \TCT is even competitive against the training with the full dataset for 
{\tt Random Forest} and {\tt SVM} and, thus, also competitive against the aforementioned ``oracle'' Teacher.  The multi-column ``\% of the Training Set'' gives the average of
the   number of examples, relative to the size of the full training set,
employed by the models built by each pair (Teacher, Learner). We observe that \TCT builds
more accurate models than {\tt Double}, despite using $15 \%$ fewer examples (simple average over the different learners).



\begin{table*}[t]
\caption{Additional Relevant Statistics}
\label{tab:Win-Losses0}

\begin{center}
\begin{tabular}{lc||cc||cc||cc}
  \multicolumn{2}{c||}{}      & \multicolumn{2}{c||}{\TCT vs {\tt Double}} &  \multicolumn{2}{c||}{\TCT vs \full} &  \multicolumn{2}{c}{ \% of Training Set}    \\ \hline \hline
      Learner       & $\#$ Datasets       &  Win & Loss  & Win & Loss & {\tt Double} & \TCT   \\  \hline
{\tt LGBM}   & 15             & 9 & 1  &  6 & 7 &  24.1 & 17.1 \\  \hline 
{\tt Random Forest} & 20      & 13 & 1  & 4 & 4 &     45.0 & 45.5  \\  \hline 
{\tt SVM}     & 13  & 7 & 2  &  5 & 5 &    39.0 & 34.6   \\  \hline 
{\tt Log. Regression} & 13  & 2 & 2  & 1 & 8  & 29.2 & 23.3  \\  \hline 
{\tt Decision Tree}  & 7 & 2 & 3  & 2 &  4 &  33.4 & 25.0\\  \hline \hline 
{\tt Overall} & 68  & 33 & 9  & 18 & 28  &  &    
     
\end{tabular}
\end{center}

\end{table*}

\subsection{Comparisons with SGD}
We compare
 \TCT and {\tt SGD} for training
{\tt SVM} and {\tt Logistic Regression}.
For that, we use the {\tt SGDClassifier} module from the {\tt sklearn} library with its default settings. 
In each iteration the Learner receives a set (mini-batch) of random
labeled examples and {\tt SGD}, via  {\tt partial\_fit} method, is used to update the classification model. This is repeated as long as we do not reach the time limit (possibly with several passes over the
training set). 

To choose the size of the mini batch we considered  all the possibilities 
in the set $ \{64,128,256,512\}$. The results  reported here consider those  that achieved the best results, namely 256 for hinge loss ({\tt SVM}) and 512 for log loss ({\tt Logistic Regression}).

\remove{

For training the Learners via  Stochastic Gradient Descent ({\tt SGD}) we use the SGDClassifier module from the {\tt sklearn} library with its default settings. At each iteration a mini batch of random examples is used to estimate the loss function gradient and the model is updated by the {\tt partial\_fit} method.
}



\remove{
Table~\ref{tab:svm-sgd-accuracies} (resp. Table~\ref{tab:logistic-regression-sgd-accuracies}) shows the average accuracy obtained by {\tt  SVM} (resp. {\tt  Logistic Regression}) on each dataset for \TCT ($\alpha=0.2$) and {\tt SGD}.
}

$\TCT$ had a much better performance:
for {\tt Logistic Regression}, with $95\%$ statistical confidence,
 it outperformed ${\tt SGD}$ in $11$ out of $13$ datasets and was worse  in only 1; for ${\tt SVM}$, also out of 13 datasets, 
 it was   better in 10 and
	 worse in 2.   Figure \ref{fig:SGD-TCT-Figures} shows
the average accuracy over time on the testing set.  
Additional tables regarding this experiment are given in the appendix A.5.

\remove{We compare
\TCT and {\tt SGD} for training
{\tt SVM} and {\tt Logistic Regression}.
In each iteration the Learner receives a set (mini-batch) of random
labeled examples and {\tt SGD} is used to update the classification model. This is repeated
as long as we do not reach the time limit (possibly with several passes over the
training set). 
$\TCT$ had a much better performance
as shown in Table \ref{tab:SGD} and Figure \ref{fig:SGD-TCT-Figures}. 
Additional details regarding this experiment are given in the appendix A.5.
}
 
We note that the advantage of \TCT\xspace may have to do with the fact that
the {\tt Scikit-Learn} classes that it uses to train {\tt Logistic Regression} and {\tt SVM}
employ optimization techniques and have hyper-parameters that are different from those of {\tt SGD}.  
That said, the key information revealed by our experiments is that $\TCT$ obtains much better results 
than using a very  natural alternative for a practitioner, that is, training via the {\tt SkLearn}'s implementation of  {\tt SGD}  with
its default parameters.  This confirms the practical appeal of a wrapper for Time Constrained Learning  that relies on our proposed method.


\remove{

\begin{table}[t]
\caption{Comparison between \TCT  and SGD}
\label{tab:SGD}
\vspace{-0.1cm}
\begin{center}
\begin{tabular}{lc||cc}
      Learner       & $\#$ Datasets       &  Win & Loss \\  \hline
{\tt SVM}     & 13 & 10 &  2  \\  \hline 
{\tt Log. Regression} & 13 & 11 & 1   \\  \hline 
     
\end{tabular}
\end{center}

\end{table}

}


\newcommand{\addLogRegTCTvsSGD}{\includegraphics[width=21em]{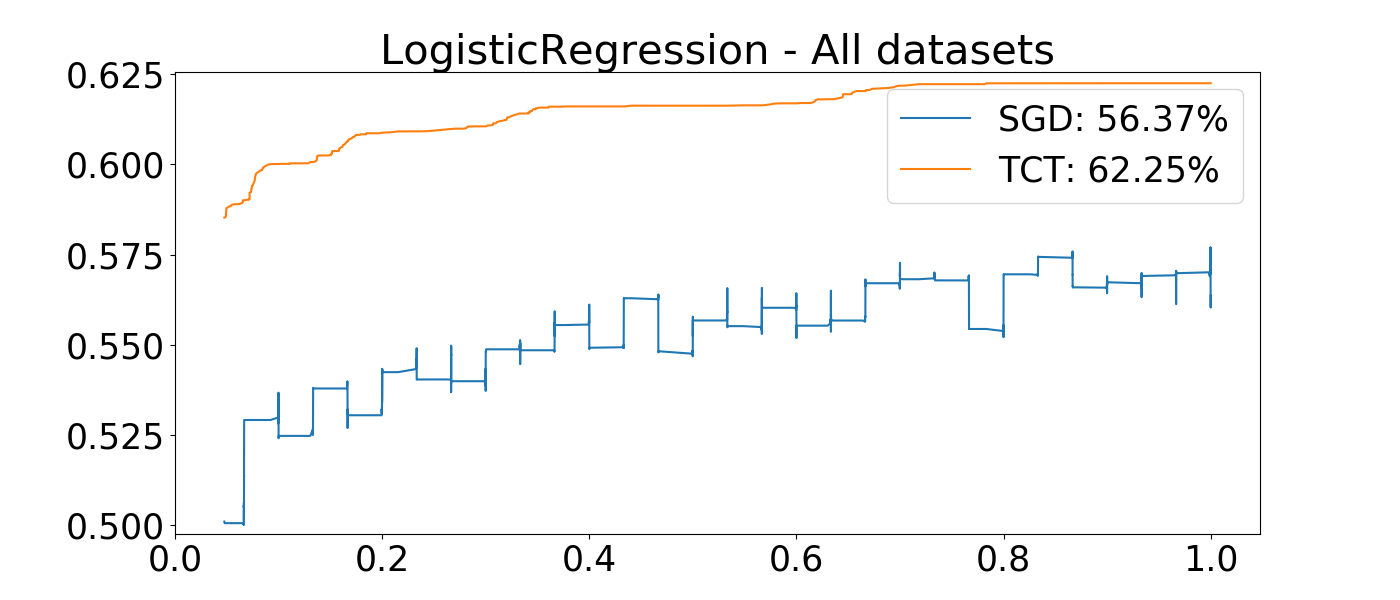}}
\newcommand{\addSvmTCTvsSGD}{\includegraphics[width=21	em]{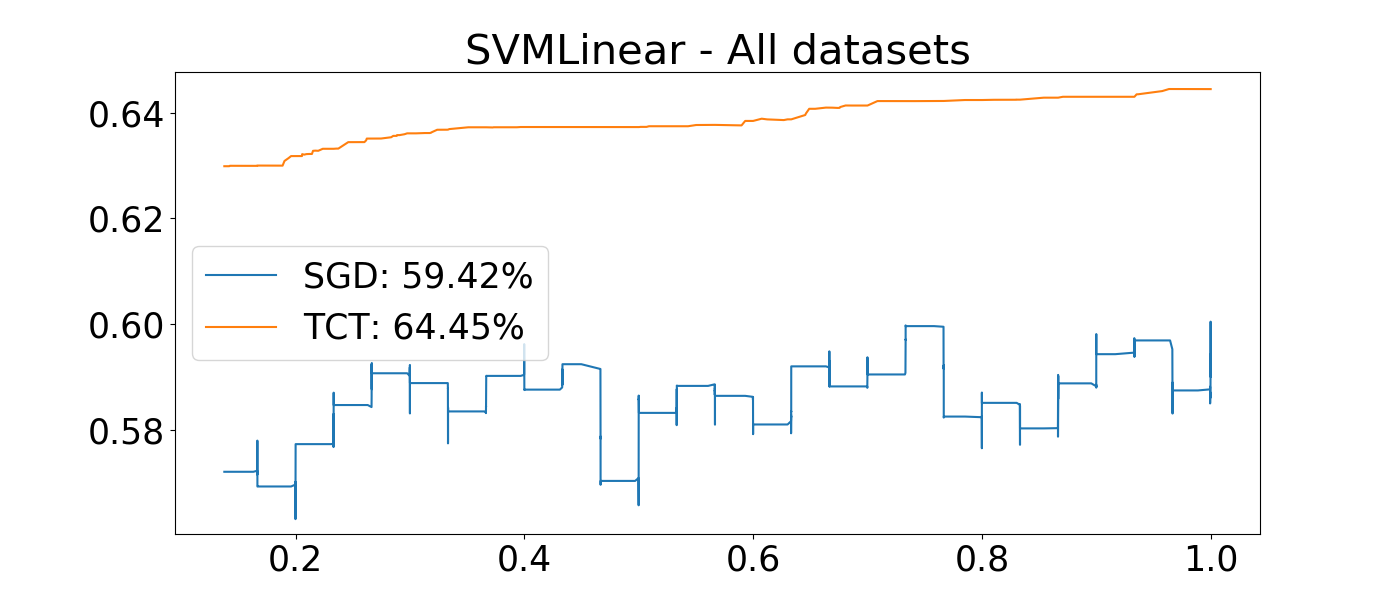}}

\begin{figure*}

\begin{center}
\begin{tabular}{lcc}
& \addLogRegTCTvsSGD & \addSvmTCTvsSGD    \\ 
\end{tabular}
\end{center}
\caption{Average accuracies on testing set along normalized time for SGD and TCT. The numbers next to the labels are their average accuracies at the last normalized time limit $t=1$. }
\label{fig:SGD-TCT-Figures}
\end{figure*}

\remove{

\begin{figure}
	\begin{center}
		\begin{tabular}{c}
			 \addSvmTCTvsSGD\\ 
		\end{tabular}
	\end{center}
	\caption{Average accuracies on testing set along normalized time for SGD and TCT. The numbers next to the labels are their average accuracies at the last normalized time limit $t=1$.}
	\label{fig:tct-vs-sgd_hinge-appendix}
\end{figure}

\begin{figure}
	\begin{center}
		\begin{tabular}{c}
			 \addLogRegTCTvsSGD\\ 
		\end{tabular}
	\end{center}
	\caption{Average accuracies on testing set along normalized time for SGD and TCT. The numbers next to the labels are their average accuracies at the last normalized time limit $t=1$. }
	\label{fig:tct-vs-sgd_log-appendix}
\end{figure}

}

\subsection{Sensibility to hyper-parameter $\alpha$}
\label{sec:sensibility}
Finally, we  performed some experiments to understand the impact of parameter $\alpha$ that controls
the number of wrong examples allocated at  each round.
Table \ref{tab:parameter-alpha} shows the number of 
wins and losses of \TCT over {\tt Double} for different values of $\alpha$.
In general, the larger the value
of $\alpha$ the larger the number of pairs (${\cal L}$, ${\cal D}$), given by column {\tt Total}, for which there is a difference with 95\% confidence
between the accuracy of {\tt TCT} and that of ${\tt Double}$.
This is  not surprising since the  smaller the $\alpha$ the larger the intersection between
the training sets employed by {\tt TCT} and {\tt Double}.
The most interesting result  is the deterioration of the average accuracy with the growth of 
$\alpha$, which is in line with Principle 3.
These experiments suggest that \TCT
 is robust with respect to  the choice of $\alpha$: one can set it safely on
the interval $[0.05,0.3]$  and expect consistent  gains.

\begin{table}[]
\caption{Sensibilty to $\alpha$: Win and Losses of \TCT over {\tt Double} and  average
accuracy of \TCT}
\label{tab:parameter-alpha}

\begin{center}
\begin{tabular}{l|ccc|c}
$\alpha$                    &  Win & Loss  & Total & Avg. Accuracy \\ \hline  \hline
{\tt $0.05$}       & 24 & 3  & 27 &  69.9 \\ \hline 
{\tt $0.1$}       & 31 & 4  & 35 & 70.2 \\ \hline 
{\tt $0.2$}      & 33 & 9  & 42   & 70.3 \\ \hline 
{\tt $0.3$}      & 37 & 12 & 49 & 70.4\\ \hline 
{\tt $0.45$}      & 32 & 21 & 53 & 70.0\\ \hline 
{\tt$0.6$}   & 25 & 35 & 60   & 69.5  \\ \hline
    {\tt$0.9$}   & 19 & 38 & 57  & 68.6  
     
\end{tabular}
\end{center}

\end{table}

\subsection{Selecting Examples via Active Learning}
\label{sec:active-learning}
We also evaluate the possibility of replacing the wrong examples in our strategy with examples selected via uncertainty sampling,  a
classic active learning strategy.

More precisely, we consider the following variation of $\TCT$, dubbed
$\TCT_{AL}$.
At the beginning of each round, $\TCT_{AL}$ uses the current
training set to build a new model. Next, this model is employed to classify the examples
of a set $S$, containing $2i$ random examples, where $i$ is the number of examples employed by the last trained model. Then, $\TCT_{AL}$  selects a set $S' \subset S$ containing $i$ examples and adds it to the current training set, which triggers the beginning of a new round.
The set $S'$ is built by picking the  $\alpha \cdot i$ most uncertain examples from $S$ and  $(1-\alpha)\cdot  i$ additional random examples from $S$, where the uncertainty 
of an example is given by the difference between
the probabilities (assigned by the Learner) to the two most probable classes. In our experiments we
used $\alpha=0.2$ for both $\TCT$ and $\TCT_{AL}$.

\remove{
We also test a variation of our approach, namely
$\TCT_{AL}$,  a hybrid between  $\TCT$ and uncertainty sampling, a
classic active learning strategy. 
More specifically, we follow the workflow of
$\TCT$ but at each iteration 
instead of selecting a mix of wrong and random examples, we classify $2x$ new examples, where $i$ is the number
of examples employed by the last trained model,
and then we sort them
by increasing order of the difference between
the probabilities of the two most probable classes.
Then, we pick the first $x$ one and add them to the training set.
}

The results are presented in Table \ref{tab:ActiveLearning}.
The ``standard" $\TCT$  is clearly better  for
{\tt Random Forest}, it has some advantage for
{\tt LGBM}
and it is worse  for both {\tt Log. Regression}
and {\tt Decision Tree.}  We do not have results
for {\tt SVM} because the probabilities
required for the strategy are
not directly available. 

Figure \ref{fig:tct-vs-alsavebest-appendix}
shows the normalized  accuracy of \TCT and $\TCT_{AL}$
over time.
Additional tables  can be found in Appendix 
\ref{sec:appendix-active-learning}.

\begin{table}
	\caption{Comparison Between \TCT and variation
	of \TCT that selects examples via active learning} 
\begin{center}
\begin{tabular}{lc||cc}
      & $\#$ Datasets       &  Win & Loss    \\  \hline
{\tt LGBM}   & 15             & 7 & 4   \\  \hline 
{\tt Random Forest} & 20      & 13 & 1  \\  \hline 
{\tt Log. Regression} & 15  & 0 & 4  \\  \hline 
{\tt Decision Tree}  & 8 & 1 & 5 
\end{tabular}
\end{center}
\label{tab:ActiveLearning}
\end{table}

\newcommand{\addDecTreeTCTvsALsavebest}{\includegraphics[width=19em]{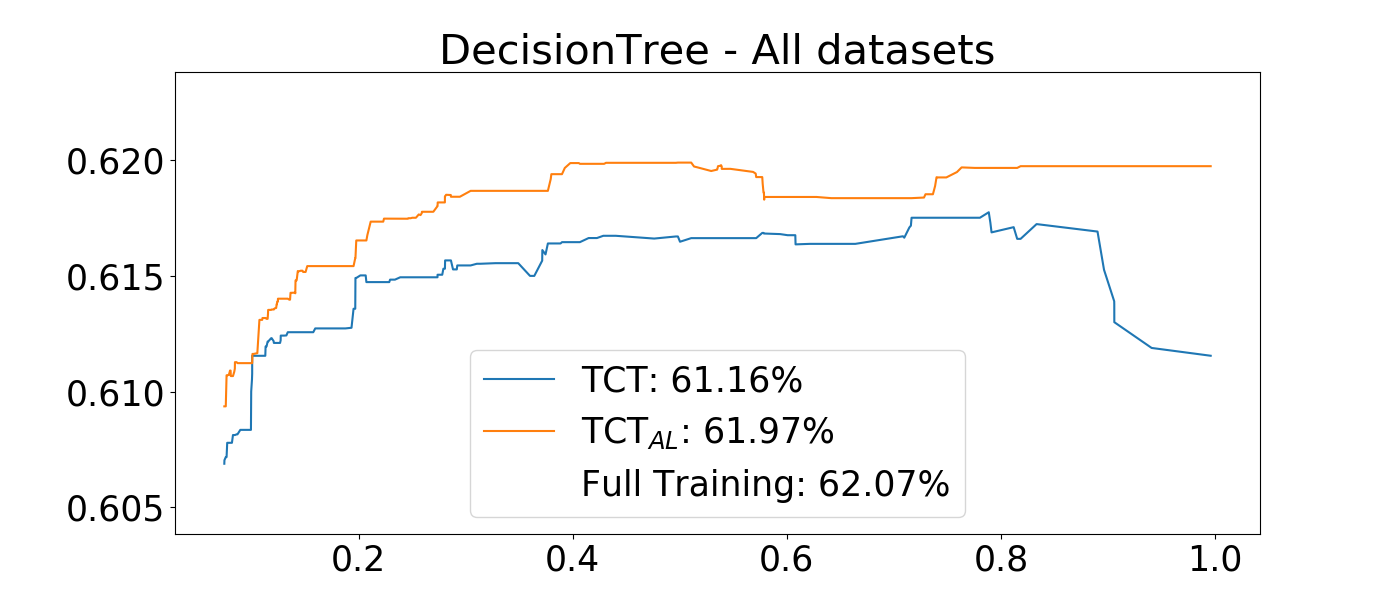}}

\newcommand{\addLGBMTCTvsALsavebest}{\includegraphics[width=19em]{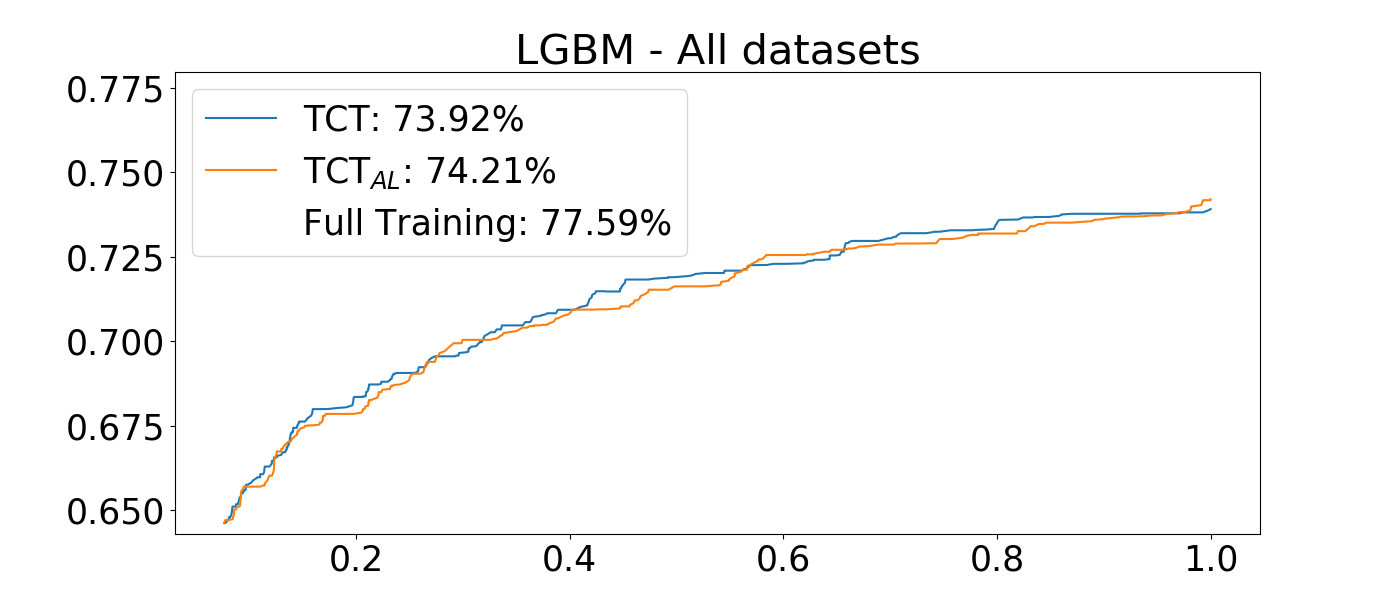}}

\newcommand{\addLogRegTCTvsALsavebest}{\includegraphics[width=19em]{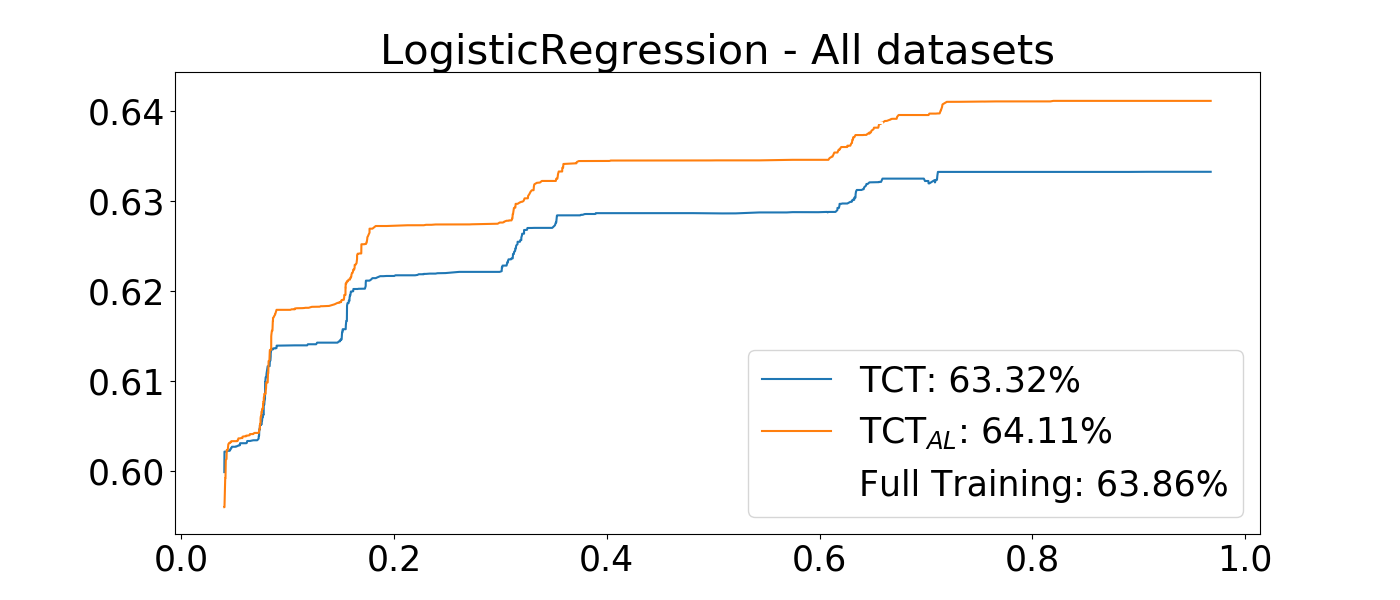}}

\newcommand{\addRandForestTCTvsALsavebest}{\includegraphics[width=19em]{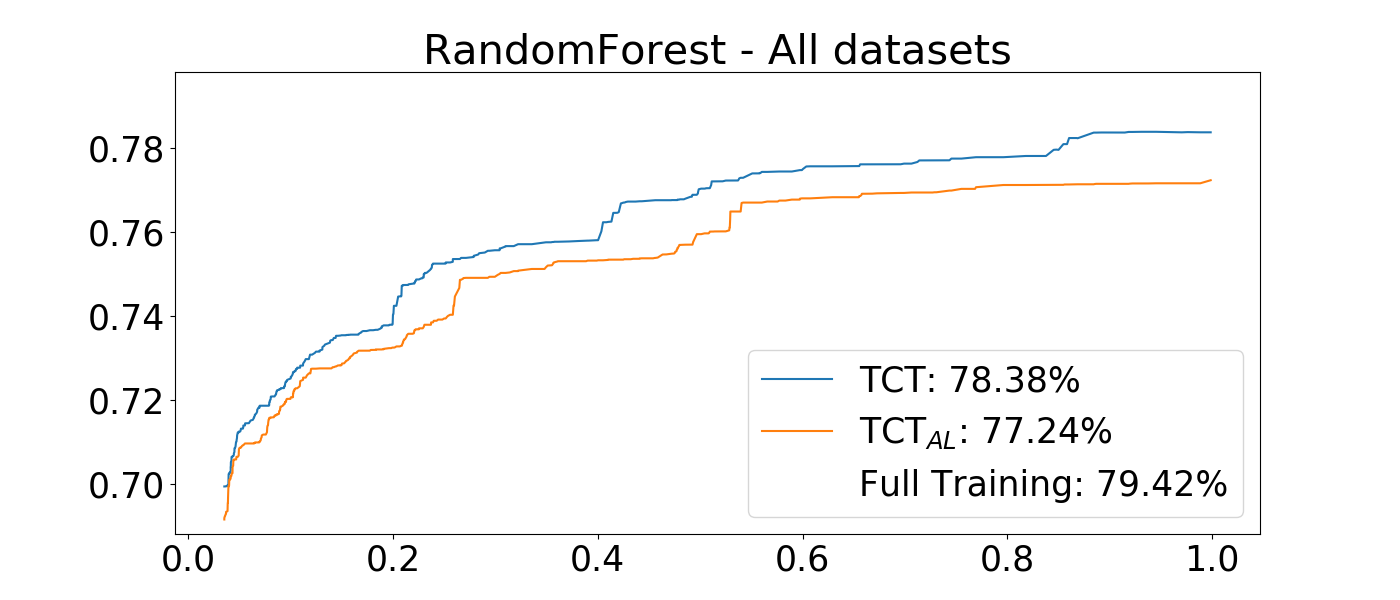}}

\begin{figure}
\begin{center}
\begin{tabular}{lcc}
& \addDecTreeTCTvsALsavebest  & \addLogRegTCTvsALsavebest    \\ 
&  \addRandForestTCTvsALsavebest & \addLGBMTCTvsALsavebest  \\

\end{tabular}
\end{center}
\caption{Average accuracies on testing set along normalized time for TCT and TCT$_{AL}$. The numbers next to the labels are their average accuracies at the last normalized time limit $t=1$.}
\label{fig:tct-vs-alsavebest-appendix}
\end{figure}



\remove{

Each histogram in Table \ref{tab:histograms} corresponds to a different value of $\alpha$ in the
set $\{5,10,20,30\}$. Each bar is associated with an interval  of length 0.5 and its
height gives the number of pairs (Learner,Dataset) for which the difference of accuracy between  
${\tt TCT}$ and ${\tt Double}$ lies in the interval. The central interval is $[-0.25,0.25]$.
As expected,  the concentration around 0 gets smaller as $\alpha$ increases.
In addition, the number of pairs for which {\tt TCT} has a significant advantage over  {\tt Double} 
also increases. 

}





\remove{

For some specific Learners, a natural way to approach Time Constrained Learning is via 
Stochastic Gradient Descent (SGG).
We make some experiments to verify how \TCT compares with SGD.
Among our Learners,  {\tt SVM} and {\tt Logistic Regression}
admit a direct online training via SGD.
We trained them using the library XXX from {\tt scikit-Learn},
using partial fit and default parameters. The size of the mini-batch was set to 256.
\red{CONFERIR DADOS DA TABELA}
Table \ref{tab:SGD} show our results.
}

\section{Theoretical Analysis} 
\label{sec:ProvableGuarantees}
To complement our work,  we present  theoretical results that provide
some insight on how the parameter $\alpha$
affects  the performance of $\TCT$.
The main conclusions of this analysis are:
 a large $\alpha$ can be harmful;
 with a small $\alpha$ $\TCT$ is never much worse
than random sampling  and, in some situations,
it is significantly better. 
To carry out the analysis we consider the following setting:

\begin{itemize}
\item[(i)] The training algorithm has access to as many labeled examples from the unknown distribution $\mu$ as it wants;

\item[(ii)] The  training time of the learner ${\cal L}$ 
can be approximated by a non-sublinear function that
does not grow very fast, that is,  training ${\cal L}$ with $m$ examples takes time $m^k f(m)$, where $k$ is a small positive integer and $f$ is a sublinear non-decreasing function;

\item[(iii)] The learner is an empirical risk minimizer (ERM), that is, it returns a hypothesis $h$
in its hypothesis class $\cH$ that makes the smallest number of mistakes in the set of examples $S$ it receives; 

\item[(iv)] It takes no time to pick an example and have it classified by the learner ${\cal L}$ using its
current classification model
\end{itemize}

Assumption (i) can be approximated by having a huge dataset $\D$ sampled from $\mu$ and is exactly the one that motivates our research,
since for smaller datasets Time-Constrained Learning is not particularly relevant.
The second assumption is also reasonable in the sense that most of the known learning methods
neither  take sublinear time nor
have a high time complexity.  
Assumption (iii) is a standard assumption employed to perform theoretical analyses.

  With regards to the last item, it is motivated by the  quite common  situation in which the classification time is very small compared to the training time (e.g. decision trees and SVM's).
In fact, we could have replaced  assumption (iv) by a weaker 
one but that  would compromise the clarity of the presentation without
changing our main conclusions regarding the impact of the parameter $\alpha$.



	To understand the accuracy of the models obtained by \TCT, we analyze a stripped-down version of the algorithm denoted by	\TCTbase; we compare it against the batch teacher \PACt that receives a time limit $T$ and simply sends to the learner in one round the largest number of random examples from $\mu$ that the learner can be trained over within time $T$. As \TCT, in each round \TCTbase sends to the learner a  $(1-\alpha)$ fraction of examples from the original distribution and an $\alpha$ fraction of examples where the learner is currently wrong. The pseudo-code of \TCTbase is presented below.

\begin{algorithm}[H]

Start with a random set $S_0$ with a single labeled example from $\mu$.

For each round $i$ (starting with $i=1$):

\begin{enumerate}
	\item Run the \learner on examples $S_i$.
	Get back hypothesis $h_i$
	
	\item Get $(1-\alpha)\, 2^i$ unbiased labeled examples from $\mu$.
	
	Then repeatedly sample from $\mu$ until getting $\alpha\, 2^i$ labeled examples $(x,y)$ where $h_i$ is wrong, namely $h_i(x) \neq y$.
	
	\item Set $S_{i+1}$ as $S_i$ plus these new $2^i$ samples
\end{enumerate}

Return the last hypothesis $h_i$ found within the time limit $T$
	
  \caption{\TCTbase ($\alpha $: real parameter; $T$: time limit) }
  \label{alg:simpler}
\end{algorithm}

We note that due to assumption (iv),
the second step of the algorithm incurs negligible running time.	 
If $h_i$ has a small error, however, this  assumption  becomes
 unrealistic since we would need to sample a huge number of examples  
from $\mu$ to obtain $\alpha 2^i$ wrong ones. 
This issue can be fixed by stopping
the algorithm as soon as it obtains a hypothesis with error at most some $\varepsilon$ (e.g. $< 1\%$). This only incurs an additional $+\varepsilon$ in the bounds of Theorems \ref{thm:fallbackR} and \ref{thm:expImprov} presented further in this section.


	Before presenting our results,	we briefly recall some definitions  from statistical learning.
	For an unknown distribution $\mu$ over labeled examples $\cX \times \cY$, the true error of a classifier $h$ is $$\err(h) := \Pr_{(X,Y) \sim \mu}(h(X) \neq Y).$$ Given a set of samples $S = ((X_1,Y_1),\ldots,(X_m,Y_m))$ from $\mu$, the sample error of $h$ is $$\err_S(h):= \frac{1}{m} \sum_{i = 1}^m \ones(h(X_i) \neq Y_i).$$ 
Let $\cH$ be the Learner's set of hypotheses. In the \emph{realizable} setting all the labels of the samples are given by a $h^* \in \cH$, namely $y = h^*(x)$ for every $(x,y)$ in the support of $\mu$. 
If the setting is not realizable then it is called \emph{agnostic}.



\medskip
\noindent{\bf Fallback analysis.}
As discussed in Principle 3 above, there is a concern that by including ``wrong examples'' we bias the distribution of the examples sent to the learner and compromise the real accuracy of the hypothesis learned. Indeed, we construct a small instance where \TCTbase set with a large value of $\alpha$ (thus, sending a large fraction of ``wrong examples'') is significantly worse than \PACt; see Appendix 
\ifSUPP \ref{app:badExample}  \else B.1 \fi for details.


\remove{	Despite this difficulty we prove that as long as the percentage $\alpha$ of wrong examples is not so big \TCTbase never requires many more samples than \PACt to obtain comparable accuracy. To make this precise, first consider the realizable case and let $\ERM^R_{\cH}(\e, \delta)$ (``realizable ERM bound'') be the smallest number of samples $m$ so that for every distribution $\mu$ we have $$\Pr_{S \sim \mu^m}\big(\textrm{There is $h \in \cH$ such that $\err(h) > \e$ and $\err_S(h) = 0$}\big) \,\le\, \delta.$$ Recall the standard bound $\ERM^R_{\cH}(\e, \delta) = O(\frac{1}{\e} (d \log \frac{1}{\e} + \log \frac{1}{\delta}))$,
	where $d$ is the VC-dimension of $\cH$  (see for example Chapter 4.5 of \citep{anthonyBartlett}). For realizable instances, $\ERM^R_{\cH}(\e,\delta)$ essentially characterizes the sample complexity of ERM learning: $\ERM^R_{\cH}(\e,\delta)$ samples are enough for ERM to learn a classifier with error at most $\e$ with probability $\ge 1- \delta$, and $\Omega(\ERM^R_{\cH}(\e,\delta))$ are necessary in some cases~\citep{LBrealizableERM}. Thus,  within constants, $\ERM^R_{\cH}(\e, \delta)$ is the best possible bound on the the number of samples that $\PACt$ needs to send to the learner to obtain error $\e$ with probability at least $1-\delta$. We show that \TCTbase never needs many more examples than this (proof in Appendix \ifSUPP \ref{app:fallbackR}). \else .1). \fi  
}

\remove{
Despite this difficulty, we prove that as long as the percentage $\alpha$ of wrong examples is not so big, even in the worst case \TCTbase returns a hypothesis with the same accuracy as \PACt as long as it is given a slightly bigger time limit (again, crucially this includes the total time consumed by the \learner during the executions of these Teachers). 
}

Despite this difficulty, we prove that  even in the worst case \TCTbase returns a hypothesis with the same accuracy as \PACt as long as:
the percentage $\alpha$ of wrong examples is not so big and  it is given a slightly bigger time limit (again, crucially this includes the total time consumed by the \learner during the executions of these algorithms).

To make this concrete, we discuss here the realizable case. 
Notice that in this case any ERM learner trained by \PACt  returns some hypothesis $h$ with zero sample error.  Let
   $m_T$ be the number of examples that \PACt sends to the learner under time limit $T$ and let
 $\e_T = \e_T(\cH,\mu,\delta)$ be the smallest value such that

\remove{
To make this concrete, we first consider the realizable case.  Notice that in this case any ERM learner trained by \PACt using the set of examples $S$ returns some hypothesis $h$ with zero sample error $\err_S(h)$. Let $\e_T = \e_T(\cH,\mu,\delta)$ be such that \PACt with time limit $T$ returns a hypothesis with true error at most $\e_T$ with probability at least $1- \delta$ regardless of the ERM learner; more precisely letting $m_T$ be the number of examples that \PACt sends to the learner under time limit $T$ and letting $S$ be $m_T$ random samples from $\mu$, $\e_T$ is the smallest value such that 
}

\scalebox{0.9}{\parbox{\linewidth}{%
	\begin{align*}
		\Pr\bigg( \textrm{ $ \exists h \in \cH$ such that $\err_S(h) = 0$ but $\err(h) > \e_T$}\bigg) < \delta,   
	\end{align*}
}}

where the probability is taken over  sets $S$ of $m_T$ examples
sampled according to $\mu$.
In words, $\e_T$ is  the best provable guarantee in terms of error for \PACt,
with $1- \delta$ probability, when the time limit is $T$.

\begin{thm} \label{thm:fallbackR}
	Given $\delta \in (0,1)$ and time limit $T$,
let  $\epsilon_T$ be  defined as above.
Under assumptions (i)-(iv),	in the realizable setting, with probability at least $1- \delta$, \TCTbase returns in time at most $ T \cdot 2 (\frac{2}{1-\alpha})^{k+1} $ a classifier with error at most $\e_T$. 
\end{thm}
\begin{proof}
Again let $m_T$ be the number of samples sent by $\PACt$ when the time limit is $T$.
Moreover, let $\hat{i}$ be  the first round in which \TCTbase sends at least  $\frac{1}{1-\alpha} m_T$ samples,
that is, $ \frac{1}{1-\alpha} m_T  \le 2^{\hat{i}} \le \frac{2}{1-\alpha} m_T   $  .
Due to the assumptions (ii) and (iv), the time \TCTbase takes to finish round $\hat{i}$ is at most
\begin{align}
 &\sum_{i=0}^{\hat{i}}   \Big((2^i)^k \cdot f(2^i)\Big)   \le f(2^{\hat{i}}) \frac{ (2^{\hat{i}+1})^{k}}{2^k-1} \le
2 f(2^{\hat{i}}) 2^{k \hat{i}}  \label{eq:thm1} \\
 &\le 2 f \left ( \frac{2}{1-\alpha} m_T \right) \left ( \frac{2}{1-\alpha} m_T \right )^{k} \le
2 \left ( \frac{2}{1-\alpha} \right )^{k+1} T, \notag
\end{align}
the last inequality holding because of the sublinearity of $f$ and because $(m_T)^k \cdot f(m_T) \le T$, by definition of $m_T$. 


Let $S$ be the set of samples sent by \TCTbase to the Learner at round
$\hat{i}$, and let $h$ be the returned hypothesis.
 The choice of $\hat{i}$ guarantees the existence of 
a subset $U$ of $S$
containing  $(1-\alpha)|S| \ge m_T$ samples that
were drawn unbiasedly from $\mu$.
Since we are in the realizable case, 
we have that $err_U(h)=0$ and by definition of $\epsilon_T$ the 
probability of $h$ having true error at most $\epsilon_T$ 
is at least $1-\delta$. 
\end{proof}


\remove{
\begin{proof}
Again let $m_T$ be the number of samples sent by $\PACt$ when the time limit is $T$.
Moreover, let $\hat{i}$ be  the first round in which \TCTbase sends at least  $\frac{1}{1-\alpha} m_T$ samples,
that is, $ \frac{1}{1-\alpha} m_T  \le 2^{\hat{i}} \le \frac{2}{1-\alpha} m_T   $  .
Due to the assumptions (ii) and (iv), the time \TCTbase takes to finish round $\hat{i}$ is at most
\begin{align}
 &\sum_{i=0}^{\hat{i}}   \Big((2^i)^k \cdot f(2^i)\Big)   \le f(2^{\hat{i}}) \frac{ (2^{\hat{i}+1})^{k}}{2^k-1} \le
2 f(2^{\hat{i}}) 2^{k \hat{i}}  \label{eq:thm1} \\
 &\le 2 f \left ( \frac{2}{1-\alpha} m_T \right) \left ( \frac{2}{1-\alpha} m_T \right )^{k} \le
2 \left ( \frac{2}{1-\alpha} \right )^{k+1} T, \notag
\end{align}
the last inequality holding because of the sublinearity of $f$ and because $(m_T)^k \cdot f(m_T) \le T$, by definition of $m_T$. 


Let $S$ be the set of samples sent by \TCTbase to the Learner at round
$\hat{i}$, and let $h$ be the returned hypothesis.
 The choice of $\hat{i}$ guarantees the existence of 
a subset $U$ of $S$
containing  $(1-\alpha)|S| \ge m_T$ samples that
were drawn unbiasedly from $\mu$.
Since we are in the realizable case, 
we have that $err_U(h)=0$ and by definition of $\epsilon_T$ the 
probability of $h$ having true error at most $\epsilon_T$ 
is at least $1-\delta$. 
 \end{proof}
}

\remove{

\begin{thm} \label{thm:fallbackR}
		If the classification instance is realizable, then with probability at least $1- \delta$, 
the \TCTbase returns a classifier with error $O( \e)$, where $\e$ is the best possible guarantee
that is obtainable, with probability $(1-\delta)$,
when the time limit is $T$.
	\end{thm}
\begin{proof}
Let $m_T$ be the number of examples sent by $\PACt$ when the time limit is
$T$. 
There existsWe have that its true error $\epsilon$ is such that  
$$c \left ( \frac{1}{\e} (d \log \frac{1}{\e} + \log \frac{1}{\delta}) \right ) \ge m_T  $$

It follows from Lemma \ref{} that 
$\TCTbase$ sends at least $m_T/ k$ samples for some
constant $k>0$.
Since $(1-\alpha)$ of them are drawn unbiasedly from
$\mu$ we have that the true error is upper bounded by
$ \epsilon'$ where
$$c' \left ( \frac{1}{\e'} (d \log \frac{1}{\e'} + \log \frac{1}{\delta}) \right ) = \frac{(1- \alpha)  m_T}{k} , $$
for some constant $c' > c$ (CORRETO).

If we set $\epsilon'= \epsilon \frac{ c' k }{ c(1- \alpha)}$

\end{proof}

	\begin{thm} \label{thm:fallbackR}
		If the classification instance is realizable, then with probability at least $1- \delta$, after using at most $\frac{2}{1-\alpha}\, \ERM^R_\cH(\e,\delta) + 1$ samples \TCTbase returns a classifier with error $\le \e$. 
	\end{thm}
}

We shall note that when $\alpha $ is small, the time overhead  
is approximately $2^{k+2}$, which is not big  due to the assumption  that $k$ is small.
A similar result for the agnostic setting is presented  in the appendix 
 \ref{app:fallbackAA}.



\medskip
\noindent {\bf An almost exponential speedup.}
	Importantly, not only \TCTbase always takes time  similar to
that of	 \PACt as long as $\alpha$ is not big, but we also show that in some cases \TCTbase 
is almost exponentially faster.

	We consider the classic problem of learning a threshold  function on the real line $\R$ (so the classifiers are of the type $h(x) = 1$ if $x \ge v$ and $h(x) = -1$ if $x < v$, for $v \in \R$), in the realizable case. This is the canonical example where Active Learning gives an exponential improvement in sample complexity compared to standard PAC learning~\citep{dasguptaTwoFaces}.
	
\remove{	 \red{[M: Talvez remover o resto do paragrafo?]} the former uses $\log (1/\e)$ samples to obtain an $\e$-error hypothesis, whereas the latter requires $1/\e$ samples. This improvement is achieved by simply performing binary search for the right threshold $v^*$. 
}

The next theorem (proof in 
\ifSUPP	 Appendix \ref{app:expImprov} \else  
Appendix B.4\fi)	
 shows that even though the teaching algorithm \TCTbase is \underline{not tailored} to this problem, it also achieves an almost exponential speedup. Recall that $2^{O(\sqrt{\log x})}$ is asymptotically smaller than $x^c$ for any constant $c > 0$. 
	
	\begin{thm}[Improvement over \PACt] \label{thm:expImprov}
		Consider $\e,\delta \in (0,1)$. Let $T_{\PACt}$ be the smallest time limit that guarantees that $\PACt$ returns a hypothesis with error at most $\e$ with probability at least $1-\delta$ for all realizable instances of the problem of learning a threshold function on the real line, and define $T_{\TCTbase}$ analogously.

		
%
%

Then, $T_{\TCTbase} \le 2^{\,c \cdot \sqrt{\log T_{\PACt}}},$ where $c$ is a constant that depends on $k,\alpha,\delta$.  		

		
				
	\end{thm}

\remove{	
	The high-level idea of the proof of the upper bound on $T_{\TCTbase}$ is that in round $i$, with good probability the examples selected by \TCTbase reduce the $\mu$-weight of the ``uncertainty interval'' for the location of $v^*$ by a factor of $\frac{1}{\alpha 2^{i/4}}$. Chaining this over all rounds gives that in roughly $\log \frac{1}{\alpha} + \sqrt{\log \frac{1}{\e}}$ rounds the uncertainty interval has weight at most $\e$, and so the learner will output a hypothesis with error at most $\e$. This yields a total of roughly $2^{\log \frac{1}{\alpha} + \sqrt{\log \frac{1}{\e}}} = \frac{1}{\alpha} 2^{\sqrt{\log \frac{1}{\e}}}$ samples used, and running time $\approx (\frac{1}{\alpha} 2^{\sqrt{\log \frac{1}{\e}}})^k$. But since such large reduction on the uncertainty interval is not guaranteed for all rounds and scenarios, and the examples selected in one round depend on the history up to that point, we need to employ martingale concentration arguments to guarantee that enough many of these rounds are actually ``good''. 
	}


\section{Concluding Remarks}
We introduced the time-constrained learning task
 and the algorithm \TCT for tackling it.
 Our algorithm
 relies on  methodologically sound ideas, is supported by theoretical results, and, most importantly, experiments including 20 datasets, 5 different Learners,  and two other baselines suggest that it is a good choice for the time-constrained
learning task. 
Due to its simplicity and generality, \TCT could be easily implemented as a wrapper in machine learning libraries to address Time Constrained
Learning 

As a  future work it would be interesting to  investigate  ways of mixing examples selected by \TCT and active learning strategies. This
seems to be a promising direction as indicated by the experiments
presented in Section \ref{sec:active-learning}.

\bibliographystyle{unsrtnat}
\bibliography{ICML2020}



\appendix

 \section{Experimental Study: Additional Details} \label{app:additional}


\subsection{Dataset transformations}

We performed some transformations on the datasets.

\begin{itemize}
    \item Each dataset was randomly shuffled.
    \item Each dataset (with size $m$) was split into a \textit{training set} (with size $0.7 \cdot m$) and a \textit{test set} (with size $0.3 \cdot m$). The split ensures that both sets have (roughly) the same class distribution as the original set \footnote{The \textit{minist} dataset were already split, so the relatives sizes of the training set and the dataset in this case are not $0.7$ and $0.3$}.
    \item Each non-numerical feature with $n_{categories}$ possible values were converted into $n_{categories}$ binary features, with one of them 1, and all others 0.
    \item Each numerical feature  was standardized.
    
\end{itemize}

\subsection{Dataset sources}

Most of the datasets were obtained from the OpenML Repository, the UCI Machine Learning Repository and Kaggle. Most of the datasets have the ``Public Domain licence type". Below we make some additional citation requests:

\begin{itemize}
    \item Diabetes130US: Beata Strack, Jonathan P. DeShazo, Chris Gennings, Juan L. Olmo, Sebastian Ventura, Krzysztof J. Cios, and John N. Clore, “Impact of HbA1c Measurement on Hospital Readmission Rates: Analysis of 70,000 Clinical Database Patient Records,” BioMed Research International, vol. 2014, Article ID 781670, 11 pages, 2014.
    \item covtype: copyright for Jock A. Blackard and Colorado State University.
    \item vehicle\_sensIT: M. Duarte and Y. H. Hu.
    \item MiniBooNE: B. Roe et al., 'Boosted Decision Trees, an Alternative to Artificial Neural Networks' \url{https://arxiv.org/abs/physics/0408124}, Nucl. Instrum. Meth. A543, 577 (2005).
    \item cifar\_10: Alex Krizhevsky (2009) Learning Multiple Layers of Features from Tiny Images, Tech Report.
    \item GTSRB-HueHist: \url{https://www.openml.org/d/41990}
    \item aloi: \url{ http://www.csie.ntu.edu.tw/~cjlin/libsvmtools/datasets/multiclass.html}
\end{itemize}
\subsection{TCT $\times$ OSCT - additional comparisons}
\label{sec:TCT-OSCT-Additional}

\begin{figure*}[!h]
	\centering
	
	\fbox{
		\hspace{-10pt}
		\begin{minipage}{0.95\textwidth}
			\small
			\vspace{4pt}
			\hspace{5pt} {\bf Algorithm} 
			\OSCT			
			\vspace{4pt}
			
			\hspace{5pt} \textbf{Input:} Set of $m$ examples $\cX$, (guess of) the number of \learner's hypotheses $N$
			\vspace{-2pt}
			
			\begin{enumerate}
				\item Initialize weights $W^0_e = \frac{1}{2m}$ for all examples $e \in \cX$
				
				\item For each round $t = 1,2,\ldots$:
				\vspace{-3pt}
				\begin{itemize}
					\item Receive hypothesis $H_t \in \cH$ from the Learner
					\item If $H_t$ is correct in all examples, stop and return $H_t$
					
					
					\item \textbf{(Weight update)} Double the weights of all wrong examples until their weight adds up to at least 1. That is, define
					\begin{align*}
						W_e^t = \left\{\begin{array}{ll}
							2^\ell \cdot W_e^{t-1}&,~\textrm{if $e \in \wrong(H_t)$}\\
							W_e^{t-1}&,~\textrm{if $e \notin \wrong(H_t)$},
						\end{array}\right.
					\end{align*}
					where $\ell$ is the smallest non-negative integer such that $W^t(H_t):=\sum_{e \in wrong(H_t)} W^t_{e} \ge 1$
					
					\item \textbf{(Sending examples)} For every example $e$, let $D_e^t := W_e^t - W_e^{t-1}$ be the weight increase of example $e$ (note $D_e^t = 0$ if $H_t$ is not wrong on $e$) 		
					
					\vspace{1pt} Repeat 4log $N$ times: sample at most one example so that $e$ is sampled with probability $D^t_e$, and send it to \learner together with its correct label (note that $H_t$ is wrong on this example)
					
					\item If no examples were sent, return \OSCT ($\cX$, $N^2$, $\iweight$)			
				\end{itemize} 
			\end{enumerate}
			\vspace{-6pt}
		\end{minipage}
	}
	\caption{Teacher's algorithm based on an algorithm for the Online Set Covering problem.}
	\label{fig:algRealizable}
\end{figure*}


The \OSCT algorithm corresponds to the algorithm ${\cal A}_{base}$ that was discussed  and empirically evaluated in~\citep{DBLP:conf/icml/CicaleseFLM20}. ${\cal A}_{base}$ is a refinement of the algorithm proposed in \cite{conf/icml/Dasgupta0PZ19}.
 It maintains weights $W^t_e$ over the examples $e \in \cX$ for each round $t$. When a new hypothesis $h$ comes from the Learner, the Teacher verifies whether $h$ makes
no mistakes on the examples from ${\cal X}$. If so, it accepts $h$. Otherwise, it increases in exponential fashion the weights of the examples where $h$ fails until the sum of these weights becomes at least 1; then it randomly sends examples to Learner with probability proportional to the increase of the weights of the examples in this round. If no example is sent by the end of the round, the algorithm starts again with a new guess of $N$. 
Although not explicitly stated in the pseudo-code,  in our use for Time Constraint Learning, $\OSCT$ returns the last model trained within the given time limit.


We   tested some variations for the \OSCT with the aim of improving its performance.
First, we increased the initial  guess $N$ on the size of the Learner's class. By doing so we 
prevent the  Teacher sending  few examples in the first rounds.  Figure~\ref{fig:tct-vs-osct-n-meioperc-appendix} shows
 the results for  $N=2^{0.005m}$, which ensures that the Teacher sends approximately  $0.005m$ wrong examples per
 round before the estimation of $N$ is updated.
 We observe that these new results are very similar to those presented in Figure \ref{fig:tct-vs-osct}, that is,
there was no significant impact.


In another attempt,  we adopted the same approach employed by \TCT
to select the final classification  model:
among the several models built by \OSCT within the time limit, we return the one with 
the largest lower limit for the (95$\%$) accuracy's  confidence interval.
This incurs no additional cost because $\OSCT$, by design,
classifies all the examples from the training set to select the new ones that are sent to the Learner.
The results for this test are shown in Figure~\ref{fig:tct-vs-osct-savebest-appendix}, where a significant improvement of \OSCT  can be observed, in particular with regards to its stability.
Despite of this improvement,  \TCT still outperforms \OSCT for all Learners and for every (normalized) time $t \in [0,1]$.

\remove{
used the learner feedbacks as estimates of accuracy for each iteration and return the model with the best lower limit of accuracy. This test comes naturally since, like  \TCT, the \OSCT also classifies examples at each iteration and is able to obtain accuracy estimates without additional cost. The results for this test are shown in Figure~\ref{fig:tct-vs-osct-savebest-appendix}, where a significant improvement of the \OSCT method can be observed, however \TCT still outperforms \OSCT for all Learners for every (normalized) time $t \in [0,1]$.
}

\remove{
In another test aiming to improve the performance of the \OSCT we used the learner feedbacks as estimates of accuracy for each iteration and return the model with the best lower limit of accuracy. This test comes naturally since, like the \TCT, the \OSCT also classifies examples at each iteration and is able to obtain accuracy estimates without additional cost. The results for this test are shown in Figure~\ref{fig:tct-vs-osct-savebest-appendix}, where a significant improvement of the \OSCT method can be observed, however \TCT still outperforms \OSCT for all Learners for every (normalized) time $t \in [0,1]$.
}




\newcommand{\addDecTreeTCTvsOSCTnMeioPerc}{\includegraphics[width=19em]{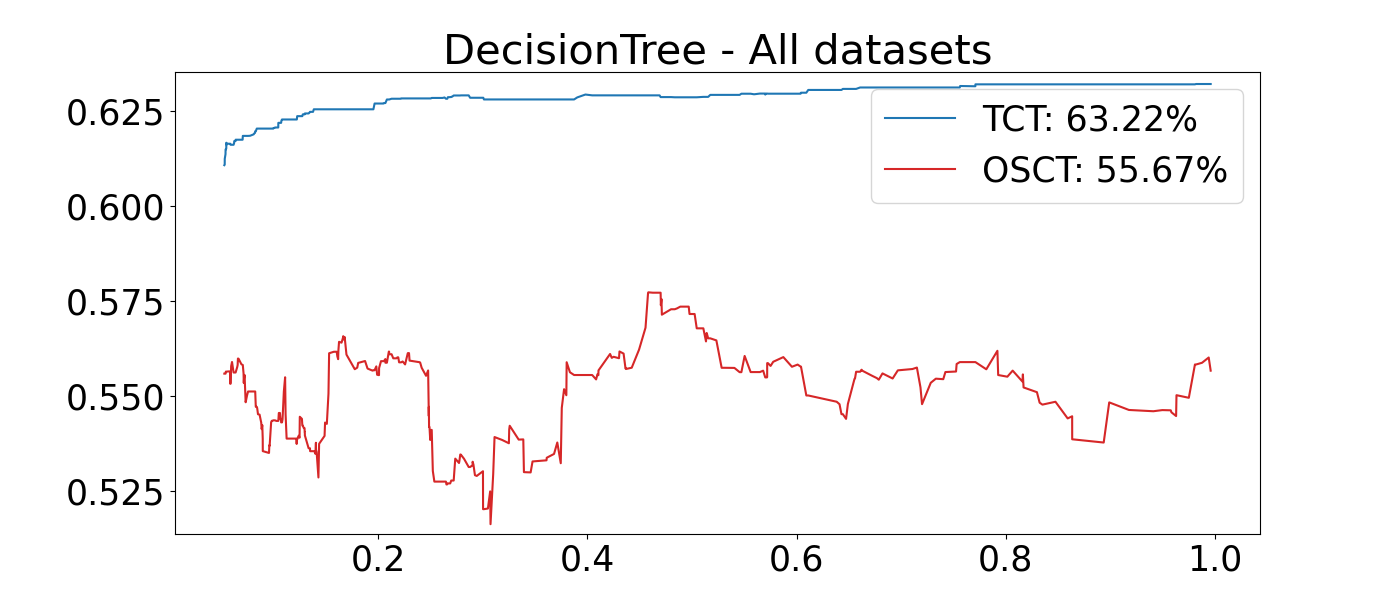}}

\newcommand{\addLGBMTCTvsOSCTnMeioPerc}{\includegraphics[width=19em]{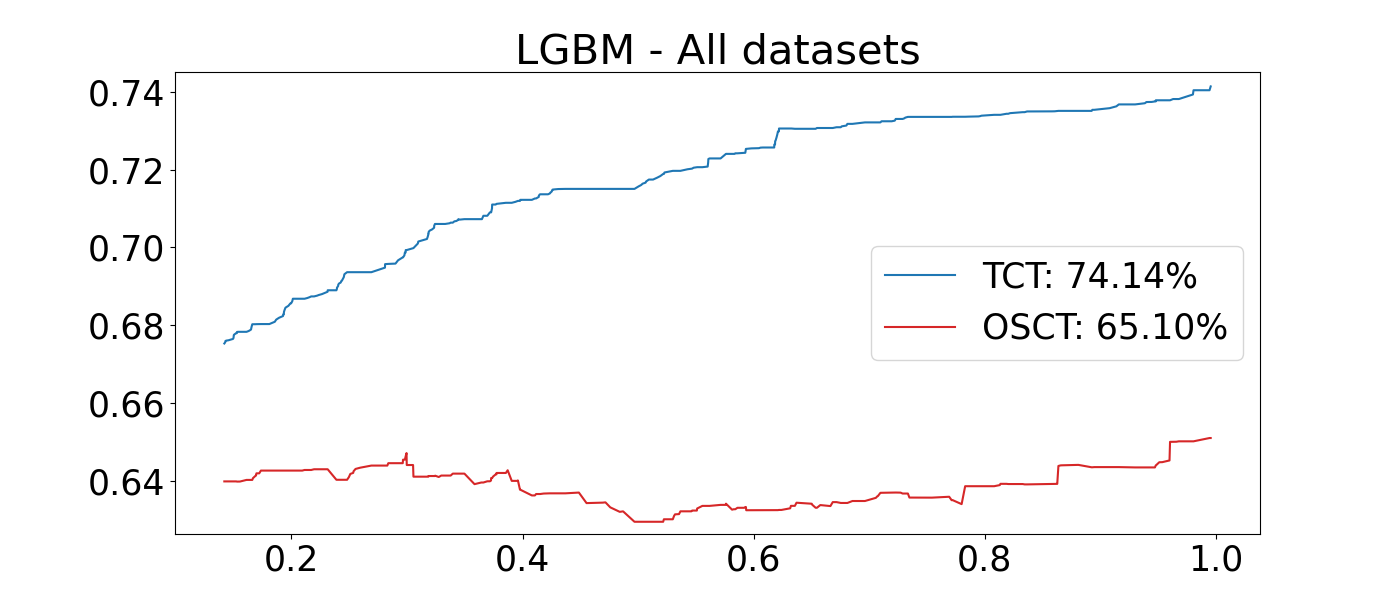}}

\newcommand{\addLogRegTCTvsOSCTnMeioPerc}{\includegraphics[width=19em]{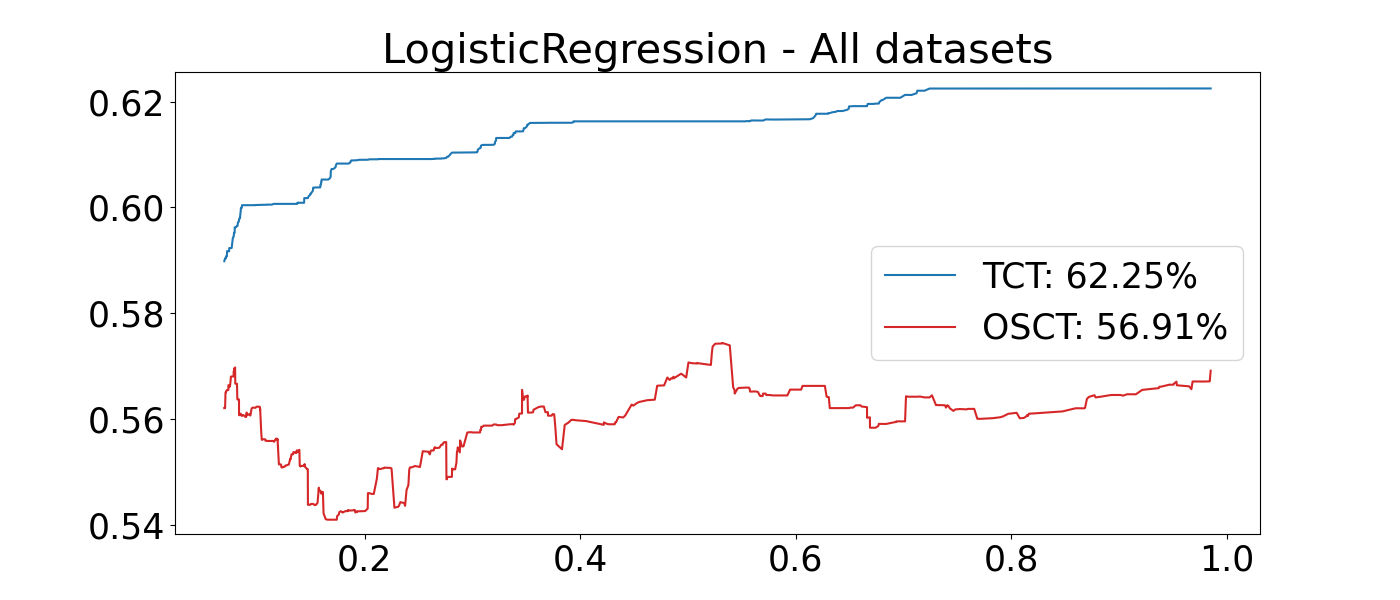}}

\newcommand{\addRandForestTCTvsOSCTnMeioPerc}{\includegraphics[width=19em]{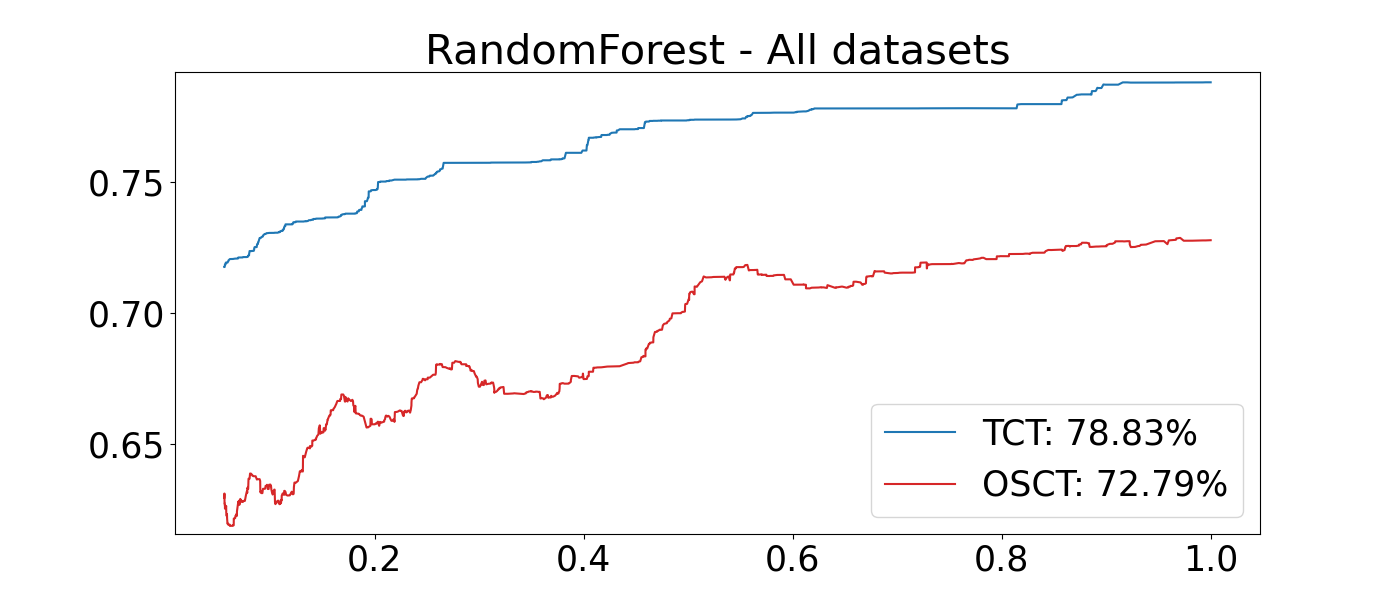}}

\newcommand{\addSVMLinTCTvsOSCTnMeioPerc}{\includegraphics[width=19em]{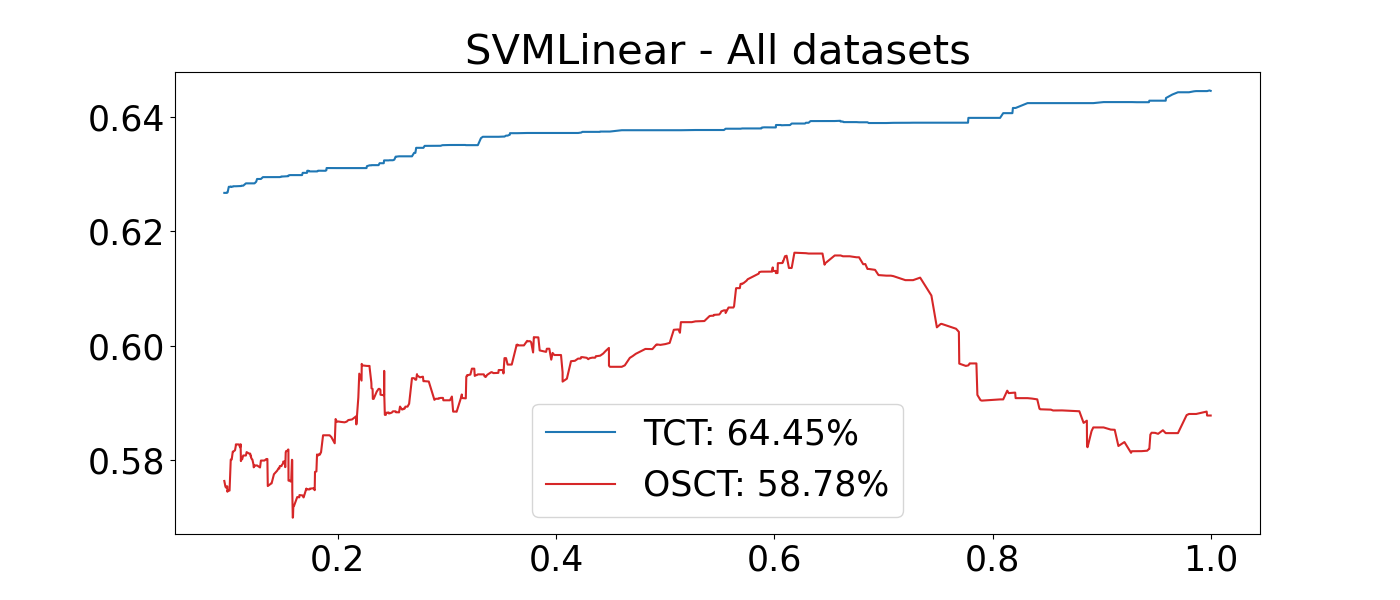}}

\begin{figure}
\begin{center}
\begin{tabular}{lcc}
& \addDecTreeTCTvsOSCTnMeioPerc  & \addLogRegTCTvsOSCTnMeioPerc    \\ 
&  \addSVMLinTCTvsOSCTnMeioPerc & \addLGBMTCTvsOSCTnMeioPerc  \\
&  \addRandForestTCTvsOSCTnMeioPerc &  \\ 
\end{tabular}
\end{center}
\caption{Average accuracies on testing set along normalized time for \TCT and \OSCT starting with $N=2^{0.005m}$. The numbers next to the labels are their average accuracies at the last normalized time limit $t=1$. }
\label{fig:tct-vs-osct-n-meioperc-appendix}
\end{figure}


\newcommand{\addDecTreeTCTvsOSCTbest}{\includegraphics[width=19em]{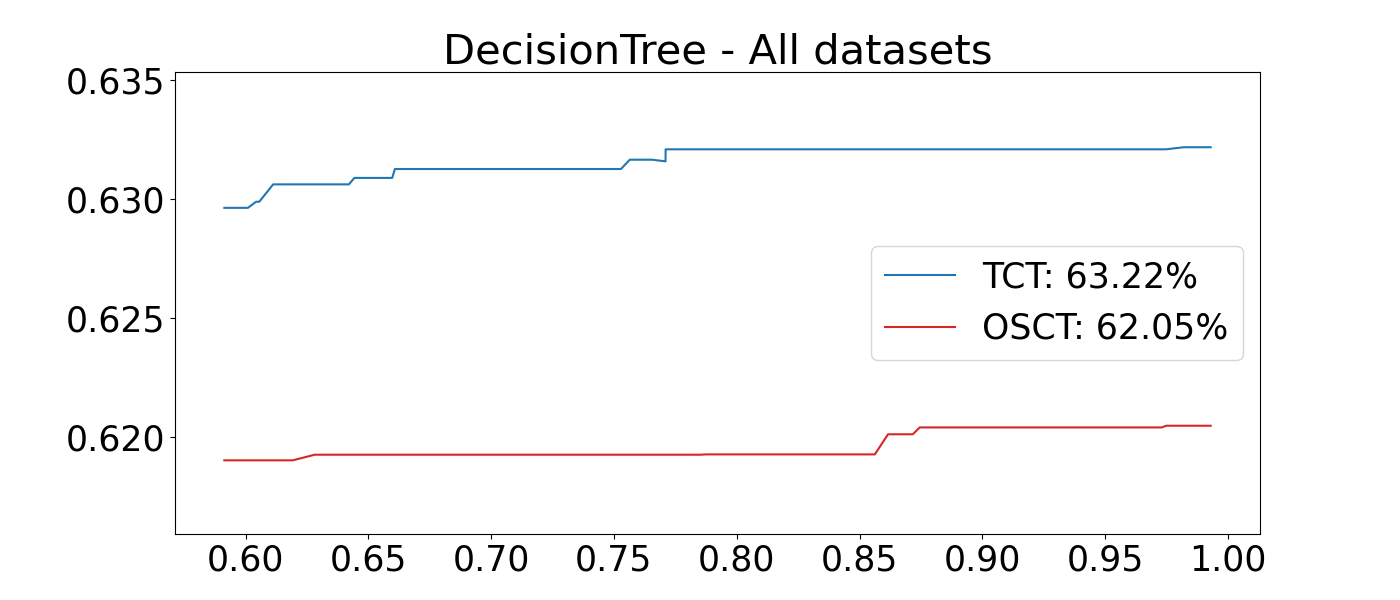}}

\newcommand{\addLGBMTCTvsOSCTbest}{\includegraphics[width=19em]{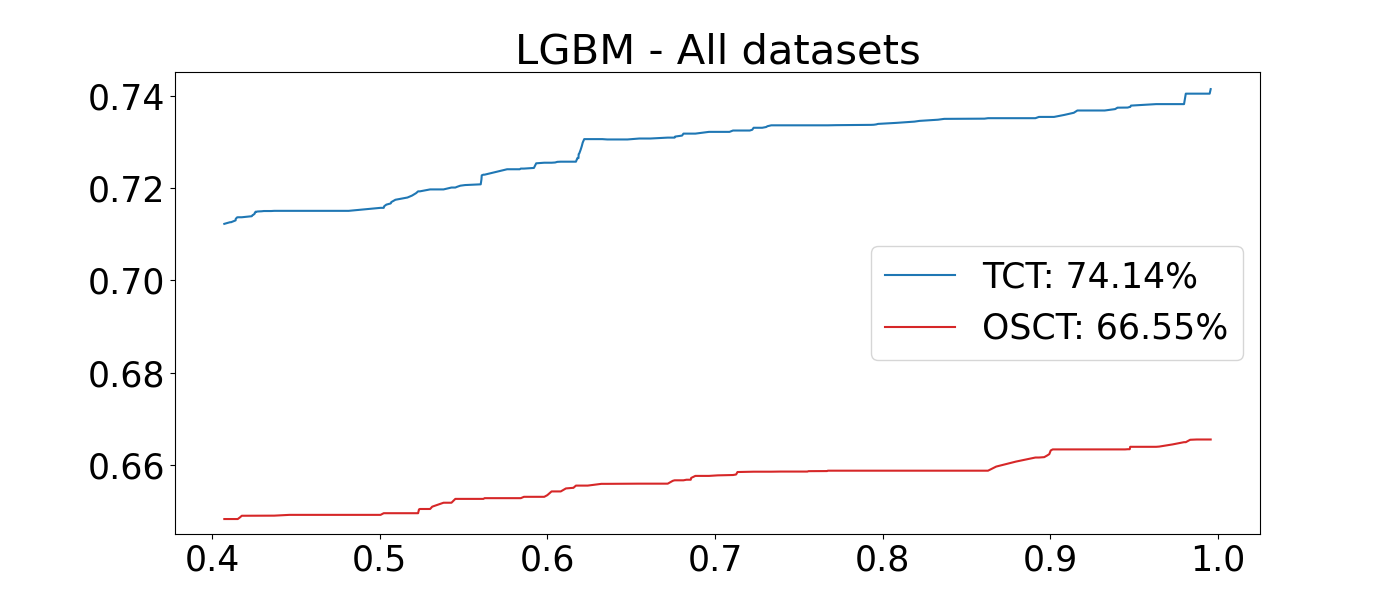}}

\newcommand{\addLogRegTCTvsOSCTbest}{\includegraphics[width=19em]{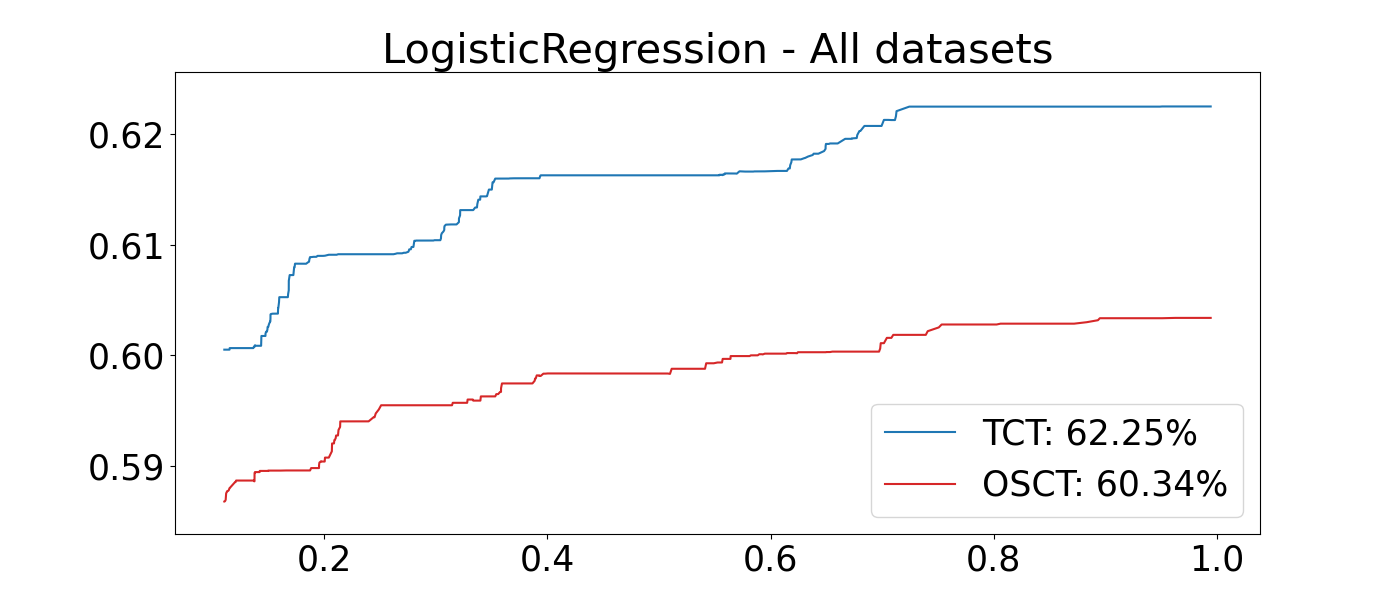}}

\newcommand{\addRandForestTCTvsOSCTbest}{\includegraphics[width=19em]{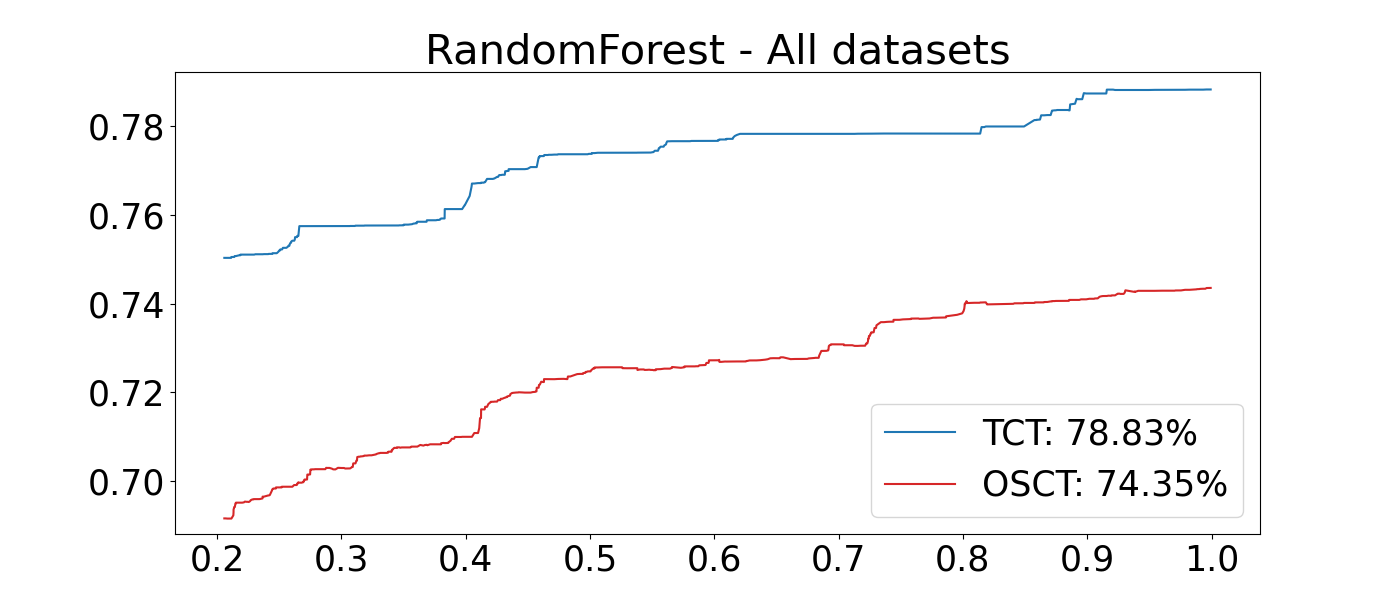}}

\newcommand{\addSVMLinTCTvsOSCTbest}{\includegraphics[width=19em]{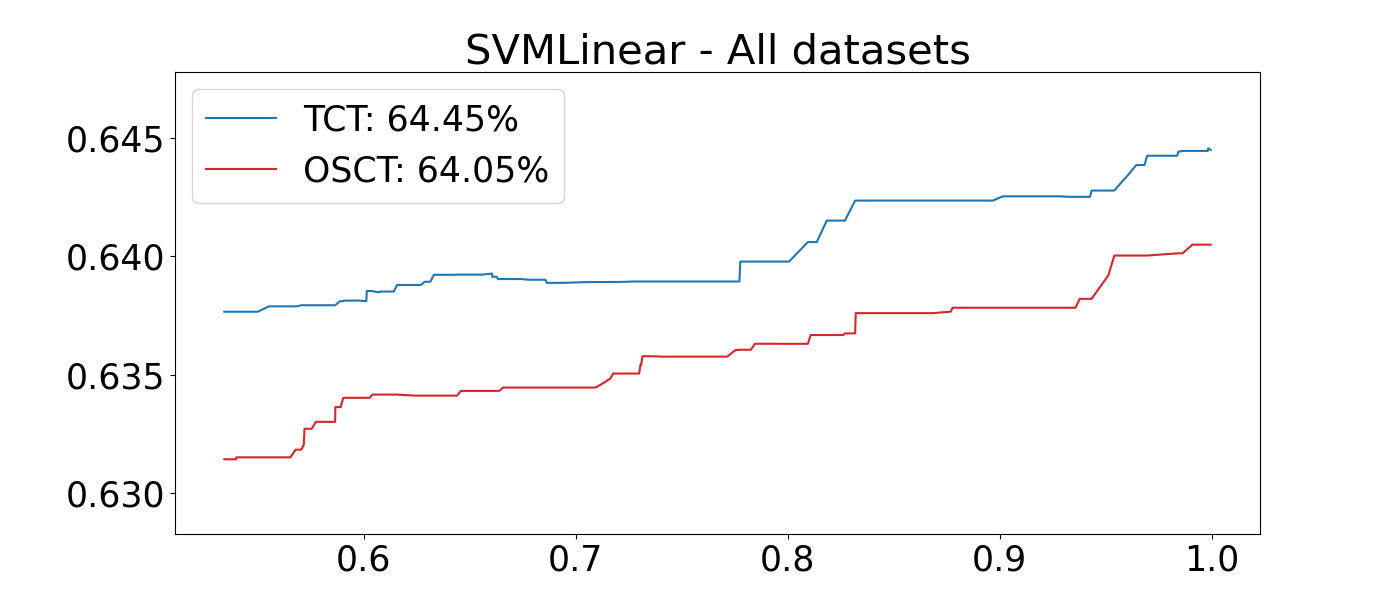}}

\begin{figure}
\begin{center}
\begin{tabular}{lcc}

& \addDecTreeTCTvsOSCTbest  & \addLogRegTCTvsOSCTbest    \\ 
&  \addSVMLinTCTvsOSCTbest & \addLGBMTCTvsOSCTbest  \\
&  \addRandForestTCTvsOSCTbest &  \\ 
\end{tabular}
\end{center}
\caption{Average accuracies on testing set along normalized time for \TCT and \OSCT which returns the model with the best accuracy estimate. The numbers next to the labels are their average accuracies at the last normalized time limit $t=1$. }
\label{fig:tct-vs-osct-savebest-appendix}
\end{figure}

\remove{


\newcommand{\addDecTreeTCTvsOSCTdouble}{\includegraphics[width=19em]{appendix/TCT_vs_dasguptaV3_n0_meioperc_doubleN/DecisionTree - All datasets.png}}

\newcommand{\addLGBMTCTvsOSCTdouble}{\includegraphics[width=19em]{appendix/TCT_vs_dasguptaV3_n0_meioperc_doubleN/LGBM - All datasets.png}}

\newcommand{\addLogRegTCTvsOSCTdouble}{\includegraphics[width=19em]{appendix/TCT_vs_dasguptaV3_n0_meioperc_doubleN/LogisticRegression - All datasets.png}}

\newcommand{\addRandForestTCTvsOSCTdouble}{\includegraphics[width=19em]{appendix/TCT_vs_dasguptaV3_n0_meioperc_doubleN/RandomForest - All datasets.png}}

\newcommand{\addSVMLinTCTvsOSCTdouble}{\includegraphics[width=19em]{appendix/TCT_vs_dasguptaV3_n0_meioperc_doubleN/SVMLinear - All datasets.png}}

\begin{figure}
\begin{center}
\begin{tabular}{lcc}
& \addDecTreeTCTvsOSCTdouble  & \addLogRegTCTvsOSCTdouble    \\ 
&  \addSVMLinTCTvsOSCTdouble & \addLGBMTCTvsOSCTdouble  \\
&  \addRandForestTCTvsOSCTdouble &  \\ 
\end{tabular}
\end{center}
\caption{Dasgupta's algorithm that doubles the teaching set in each round (v3). }
\label{fig:tct-vs-osct-double-appendix}
\end{figure}

}

\subsection{Additional Tables and Plots}


 Table~ \ref{tab:lgbm-accuracies} (resp. Tables~  \ref{tab:random-forest-accuracies}, \ref{tab:svm-accuracies}, \ref{tab:decision-tree-accuracies} and \ref{tab:logistic-regression-accuracies})  shows the average accuracy obtained by {\tt  LGBM} (resp.   {\tt  Random Forest}, {\tt  SVM}, {\tt Decision Tree} and {\tt  Logistic Regression}) on each dataset for \TCT ($\alpha=0.2$), {\tt  Double}, \OSCT and the average accuracy obtained using the entire training set (column Full). 
The results for \OSCT refer to the version of the algorithm that returns the model with the best accuracy estimate
(Figure \ref{fig:tct-vs-osct-savebest-appendix}). 
  The Time Limit column denotes the average time in seconds taken to train with the entire dataset. 
This is also the time given as a limit for the Teachers.
 It is noteworthy that some tables have more rows than others because combinations $({\cal D} , {\cal L})$ in which the time limit is less than 10 seconds are discarded.

We boldfaced the datasets for which there is a statistical difference ($95\%$ of confidence) between
the accuracies of {\tt Double} and \TCT.  More specifically,
we boldface  {\tt Double} (\TCT) for dataset ${ \cal D } $ if  the accuracy of {\tt Double} ( \TCT )  is
larger than that of \TCT ({\tt Double}) and 
$$ | acc_{Dbl}-acc_{\TCT}|-1.645 \sqrt{ \frac{acc_{Dbl}(1-acc_{Dbl})}{m_{test}} +  \frac{acc_{\TCT}(1-acc_{\TCT})}{m_{test}} } >0  ,$$
where $m_{test}$ is the size of the testing set for dataset ${\cal D}$. 
To calculate  the confidence  interval  we assume that  the examples of the testing set are drawn independently from an unknown distribution
$\mu$ (Chapter 5 of \cite{mitchell}).
We have not considered \OSCT in this comparison because it is not 
competitive with the two other Teachers. 


Figure \ref{fig:tct-vs-double-perDataset-appendix} shows how the accuracy of
both \TCT and {\tt Double} evolve over the normalized time
for some datasets.  
The images for the other datasets can be found in 
\url{https://github.com/sfilhofreitas/TimeConstrainedLearning/tree/main/experiments/results/graphics_by_dataset}.

\begin{table}[]
	\caption{LGBM accuracies in the testing sets for TCT, Double, OSCT and Full Training for each dataset.}
	\label{tab:lgbm-accuracies}
	\begin{center}
		\begin{tabular}{c|c|ccc|c}
			\textbf{Dataset} & \textbf{Time Limit} & \textbf{TCT} & \textbf{Double} & \textbf{OSCT} & \textbf{Full} \\ \hline \hline
			BNG\_letter\_5000\_1 & 198.5 & {\bf 75.9\% } & 74.8\% & 59.5\% & 76.7\% \\ \hline
            poker\_hand & 48.7 & 73.0\% & 72.8\% & 52.5\% & 83.3\% \\ \hline
            SantanderCustomerSatisfaction & 18.6 & {\bf 91.3\% }& 90.6\% & 90.0\% & 90.7\% \\ \hline
            BNG\_spectf\_test & 14.5 & 82.5\% & 82.4\% & 77.6\% & 82.9\% \\ \hline
            BNG\_wine & 19.7 & {\bf 96.1\% }& 95.8\% & 94.2\% & 96.0\% \\ \hline
            BNG\_eucalyptus & 49.8 & {\bf 74.9\%} & 73.9\% & 63.0\% & 74.3\% \\ \hline
            mnist & 134.5 &{\bf  97.4\% }& 96.0\% & 92.5\% & 97.7\% \\ \hline
            covtype & 15.5 & {\bf 86.6\%} & 85.3\% & 73.8\% & 85.3\% \\ \hline
            cifar\_10 & 793.4 & 38.6\% & 38.5\% & 37.9\% & 52.8\% \\ \hline
            volkert & 34.0 & 60.3\% & 60.1\% & 52.3\% & 69.1\% \\ \hline
            BNG\_satimage & 77.2 & {\bf 92.2\% }& 91.6\% & 87.2\% & 92.1\% \\ \hline
            Sensorless\_drive\_diagnosis & 11.7 & {\bf 99.3\%} & 98.8\% & 93.8\% & 99.9\% \\ \hline
            GTSRB-HueHist & 258.0 & 38.9\% & {\bf 39.9\%} & 33.0\% & 57.6\% \\ \hline
            BNG\_mfeat\_fourier & 214.6 & {\bf 93.7\% }& 92.9\% & 86.4\% & 93.4\% \\ \hline
            aloi & 638.1 & 11.6\% & 11.4\% & 4.6\% & 11.3\% \\ \hline 
		\end{tabular}%
	\end{center}
\end{table}

\begin{table}[]
	\caption{Random Forest accuracies in the testing sets for TCT, Double, OSCT and Full Training for each dataset.}
	\label{tab:random-forest-accuracies}
	\begin{center}
		\begin{tabular}{c|c|ccc|c}
			\textbf{Dataset} & \textbf{Time Limit} & \textbf{TCT} & \textbf{Double} & \textbf{OSCT} & \textbf{Full} \\ \hline \hline					
            Diabetes130US & 76.4 & 58.5\% & 58.4\% & 55.1\% & 59.0\% \\ \hline
            BNG\_letter\_5000\_1 & 231.9 & {\bf 71.9\% }& 70.1\% & 68.5\% & 72.9\% \\ \hline
            poker\_hand & 277.5 & {\bf 92.2\%} & 91.8\% & 92.3\% & 92.2\% \\ \hline
            SantanderCustomerSatisfaction & 485.8 & 90.0\% & 90.0\% & 88.4\% & 90.0\% \\ \hline
            BNG\_spectf\_test & 781.3 & {\bf 80.0\% } & 79.0\% & 78.0\% & 79.2\% \\ \hline
            BNG\_wine & 417.6 & 95.6\% & 95.5\% & 95.6\% & 95.6\% \\ \hline
            vehicle\_sensIT & 112.2 & 87.0\% & 86.7\% & 83.5\% & 86.9\% \\ \hline
            MiniBooNE & 59.7 & {\bf 93.9\%} & 92.5\% & 92.6\% & 93.2\% \\ \hline
            BNG\_eucalyptus & 166.2 & {\bf 69.2\%} & 67.5\% & 64.9\% & 69.1\% \\ \hline
            mnist & 32.2 &  96.3 & 95.9\% & 92.7\% & 96.3\% \\ \hline
            BNG\_spambase & 409.6 & 66.6\% & 66.7\% & 66.0\% & 66.7\% \\ \hline
            covtype & 92.7 & {\bf 93.0\%} & 88.1\% & 88.7\% & 92.4\% \\ \hline
            cifar\_10 & 123.6 & {\bf 42.9\%} & 42.1\% & 40.8\% & 45.4\% \\ \hline
            jannis & 36.2 & {\bf 69.6\%} & 68.5\% & 65.5\% & 70.0\% \\ \hline
            volkert & 18.9 & {\bf 62.4\%} & 60.6\% & 61.1\% & 64.4\% \\ \hline
            BNG\_satimage & 680.2 &{\bf  89.7\%} & 88.2\% & 88.7\% & 89.0\% \\ \hline
            Sensorless\_drive\_diagnosis & 13.7 & {\bf 99.8\%} & 99.4\% & 99.7\% & 99.8\% \\ \hline
            GTSRB-HueHist & 42.2 & {\bf 47.8\%} & 43.8\% & 44.2\% & 47.3\% \\ \hline
            BNG\_mfeat\_fourier & 1021.6 &{\bf 88.9\%} & 88.4\% & 87.8\% & 88.9\% \\ \hline
            aloi & 26.0 & 81.1\% & {\bf 84.4\%} & 32.8\% & 90.6\% \\ \hline
		\end{tabular}%
	\end{center}
\end{table}

\begin{table}[]
	\caption{SVM accuracies in the testing sets for TCT, Double, OSCT and Full Training for each dataset.}
	\label{tab:svm-accuracies}
	\begin{center}
		\begin{tabular}{c|c|ccc|c}
			\textbf{Dataset} & \textbf{Time Limit} & \textbf{TCT} & \textbf{Double} & \textbf{OSCT} & \textbf{Full} \\ \hline \hline		
			Diabetes130US & 81.6 & 57.3\% & 57.6\% & 51.5\% & 58.1\% \\ \hline
            BNG\_letter\_5000\_1 & 105.6 & {\bf 42.8\%} & 41.3\% & 42.6\% & 41.3\% \\ \hline
            poker\_hand & 17.7 & 48.1\% & {\bf 50.0\%} & 47.8\% & 50.0\% \\ \hline
            MiniBooNE & 11.6 & {\bf 90.1\%} & 89.7\% & 89.8\% & 90.1\% \\ \hline
            BNG\_eucalyptus & 80.3 & {\bf 57.0\%} & 56.4\% & 54.5\% & 56.4\% \\ \hline
            mnist & 1320.7 & {\bf 90.4\%} & 89.0\% & 89.6\% & 91.7\% \\ \hline
            BNG\_spambase & 16.5 & 66.5\% & 66.6\% & 66.1\% & 66.6\% \\ \hline
            covtype & 123.7 & {\bf 70.8\% }& 70.4\% & 70.0\% & 70.5\% \\ \hline
            volkert & 117.9 & 57.5\% & 57.3\% & 55.4\% & 57.8\% \\ \hline
            BNG\_satimage & 33.8 & {\bf 81.6\%} & 80.7\% & 81.8\% & 80.7\% \\ \hline
            Sensorless\_drive\_diagnosis & 156.7 & {\bf 78.9\%} & 73.7\% & 89.1\% & 74.3\% \\ \hline
            GTSRB-HueHist & 154.2 & 14.2\% & 14.3\% & 12.3\% & 27.2\% \\ \hline
            BNG\_mfeat\_fourier & 71.4 & 82.7\% & {\bf 83.1\%} & 82.1\% & 83.2\% \\ \hline 
		\end{tabular}%
	\end{center}
\end{table}

\begin{table}[]
	\caption{Decision Tree accuracies in the testing sets for TCT, Double, OSCT and Full Training for each dataset.}
	\label{tab:decision-tree-accuracies}
	\begin{center}
		\begin{tabular}{c|c|ccc|c}
			\textbf{Dataset} & \textbf{Time Limit} & \textbf{TCT} & \textbf{Double} & \textbf{OSCT} & \textbf{Full} \\ \hline \hline
				Diabetes130US & 12.9 & 57.1\% & 57.0\% & 54.4\% & 57.5\% \\ \hline
                SantanderCustomerSatisfaction & 17.7 & 89.0\% & {\bf 89.7}\% & 86.9\% & 89.9\% \\ \hline
                BNG\_spectf\_test & 30.0 & 77.6\% & {\bf 78.5\%} & 77.4\% & 78.5\% \\ \hline
                BNG\_eucalyptus & 12.5 & 53.8\% & {\bf 54.2}\% & 52.7\% & 54.2\% \\ \hline
                cifar\_10 & 33.6 & 24.9\% & 25.2\% & 24.8\% & 25.9\% \\ \hline
                BNG\_satimage & 25.7 & `{\bf 73.8 \%} & 73.0\% & 74.0\% & 73.0\% \\ \hline
                BNG\_mfeat\_fourier & 47.5 & {\bf 66.3\%} & 64.1\% & 64.0\% & 63.9\% \\ \hline 
                
		\end{tabular}%
	\end{center}
\end{table}

\begin{table}[]
	\caption{Logistic Regression accuracies in the testing sets for TCT, Double, OSCT and Full Training for each dataset.}
	\label{tab:logistic-regression-accuracies}
	\begin{center}
		\begin{tabular}{c|c|ccc|c}
			\textbf{Dataset} & \textbf{Time Limit} & \textbf{TCT} & \textbf{Double} & \textbf{OSCT} & \textbf{Full} \\ \hline \hline			
			Diabetes130US & 289.6 & 58.6\% & 58.7\% & 56.0\% & 59.0\% \\ \hline
            BNG\_letter\_5000\_1 & 31.5 & 45.6\% & 45.8\% & 44.1\% & 45.8\% \\ \hline
            poker\_hand & 391.5 & 48.1\% & {\bf 50.0\%} & 47.7\% & 50.0\% \\ \hline
            BNG\_eucalyptus & 35.5 & 57.8\% & 57.8\% & 55.9\% & 57.9\% \\ \hline
            mnist & 184.7 & 91.6\% & 91.3\% & 91.5\% & 92.6\% \\ \hline
            BNG\_spambase & 34.1 & 66.4\% & 66.6\% & 66.2\% & 66.6\% \\ \hline
            covtype & 84.0 & {\bf 68.2\% } & 67.6\% & 62.5\% & 68.8\% \\ \hline
            cifar\_10 & 465.0 & 37.1\% & 37.6\% & 33.7\% & 39.6\% \\ \hline
            volkert & 29.6 & 57.4\% & 57.5\% & 54.4\% & 58.6\% \\ \hline
            BNG\_satimage & 18.6 & 83.8\% & 83.7\% & 83.4\% & 83.7\% \\ \hline
            Sensorless\_drive\_diagnosis & 11.3 & {\bf 83.9\%} & 70.0\% & 81.7\% & 78.8\% \\ \hline
            GTSRB-HueHist & 134.9 & 26.3\% & 26.6\% & 24.6\% & 29.1\% \\ \hline
            BNG\_mfeat\_fourier & 46.1 & 84.4\% & {\bf 84.6\%} & 82.8\% & 84.8\% \\ \hline
		\end{tabular}%
	\end{center}
\end{table}

\newcommand{\addDecTreeMfeat}{\includegraphics[width=19em]{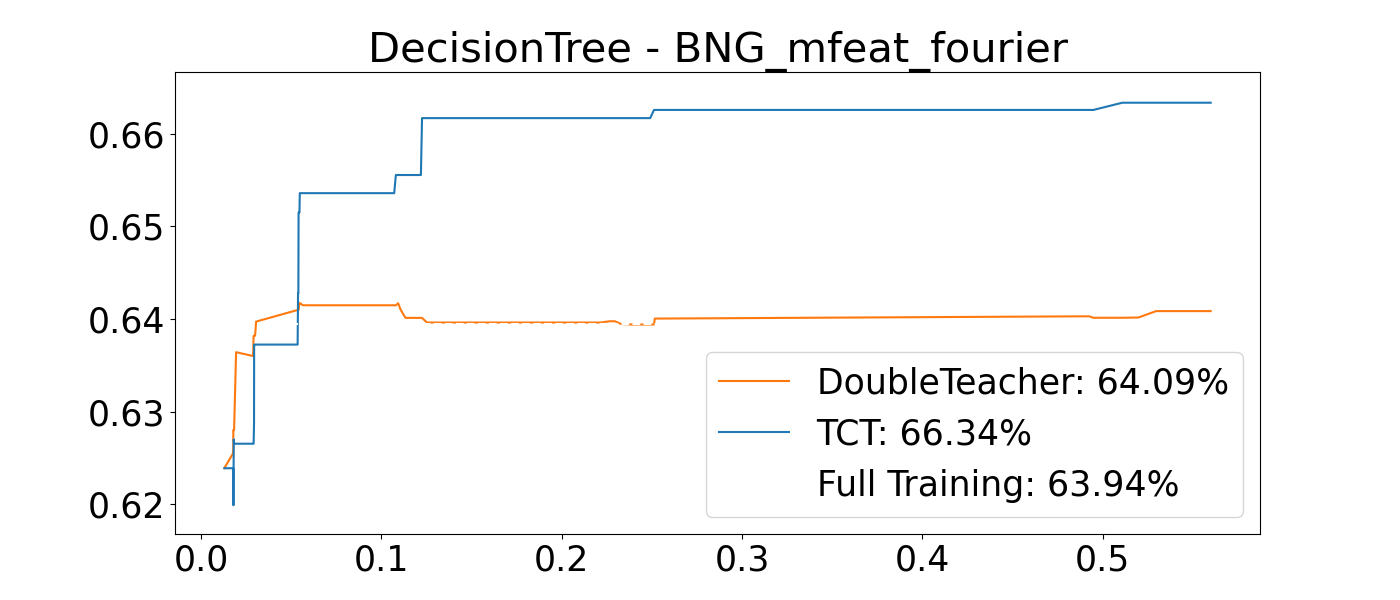}}

\newcommand{\addLGBMeucalyptus}{\includegraphics[width=19em]{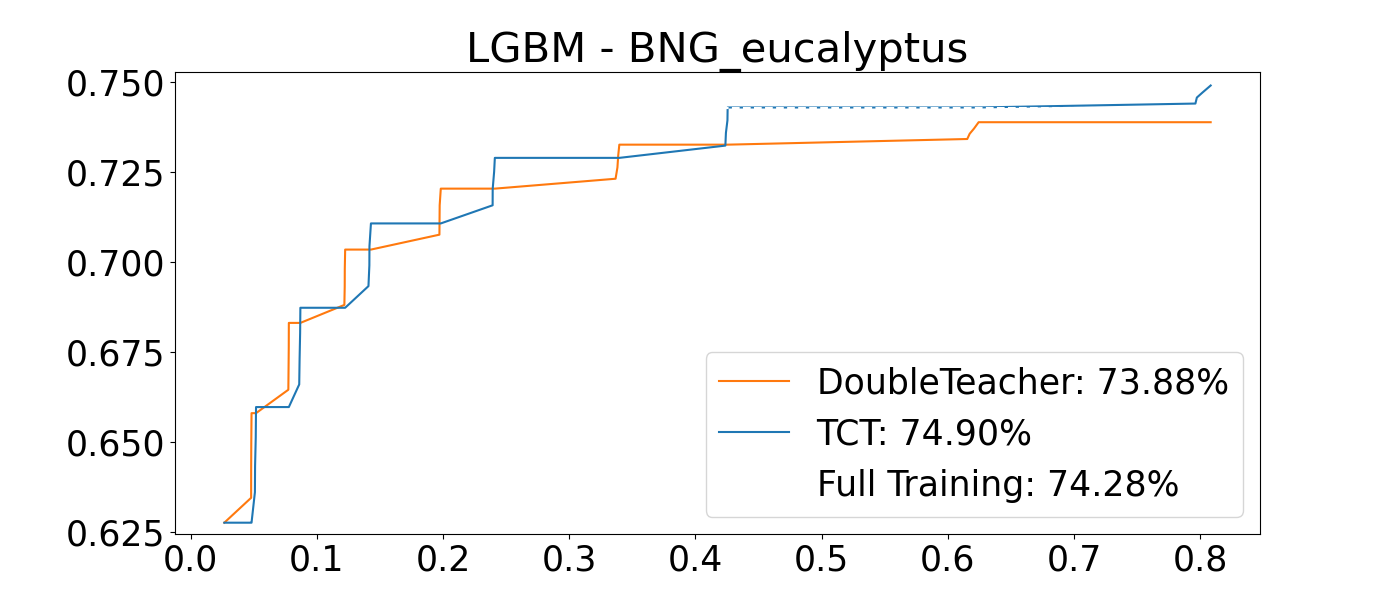}}

\newcommand{\addLogRegSensorless}{\includegraphics[width=19em]{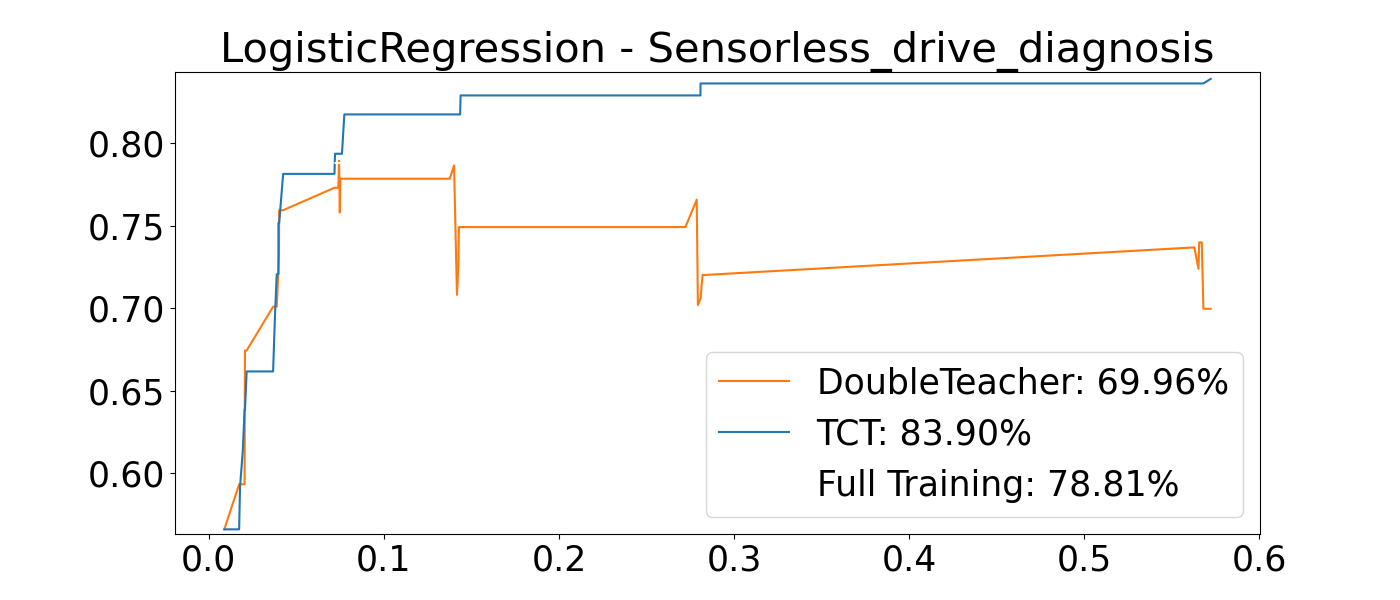}}

\newcommand{\addRandForestSatimage}{\includegraphics[width=19em]{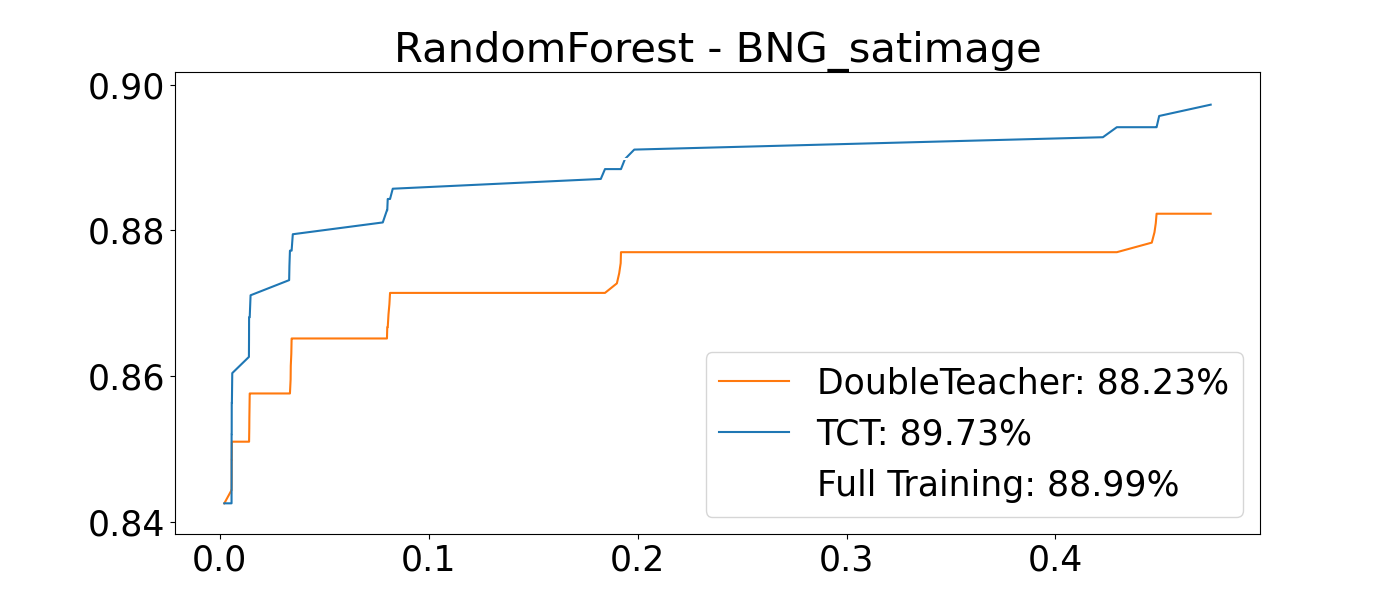}}

\newcommand{\addSVMLinMiniboone}{\includegraphics[width=19em]{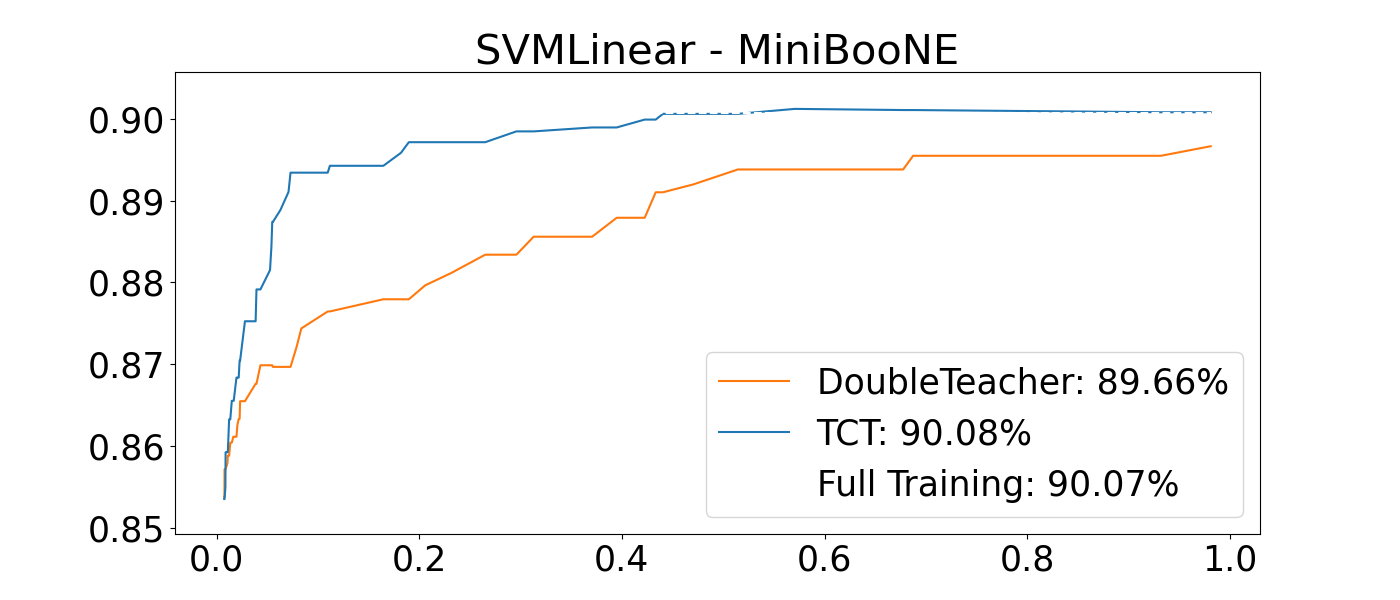}}

\begin{figure}
\begin{center}
\begin{tabular}{lcc}
& \addDecTreeMfeat  & \addLGBMeucalyptus    \\
&  \addLogRegSensorless & \addRandForestSatimage  \\
&  \addSVMLinMiniboone &  \\
\end{tabular}
\end{center}
\caption{Average accuracies on testing set along normalized time for \TCT and  {\tt Double} per dataset. The numbers next to the labels are their average accuracies at the last normalized time limit $t=1$.}
\label{fig:tct-vs-double-perDataset-appendix}
\end{figure}

\clearpage
\subsection{TCT $\times$ SGD - additional tables}

\remove{
For training the Learners via  Stochastic Gradient Descent ({\tt SGD}) we use the SGDClassifier module from the {\tt sklearn} library with its default settings. At each iteration a mini batch of random examples is used to estimate the loss function gradient and the model is updated by the {\tt partial\_fit} method.
}


Table~\ref{tab:svm-sgd-accuracies} (resp. Table~\ref{tab:logistic-regression-sgd-accuracies}) shows the average accuracy obtained by {\tt  SVM} (resp. {\tt  Logistic Regression}) on each dataset for \TCT ($\alpha=0.2$) and {\tt SGD}.

\begin{table}[!h]
	\caption{SVM accuracies in the testing sets for \TCT and {\tt SGD}.}
	\label{tab:svm-sgd-accuracies}
	\begin{center}
		\begin{tabular}{c|c|cc}
			\textbf{Dataset} & \textbf{Time Limit} & \textbf{TCT} & \textbf{SGD}  \\ \hline \hline
			Diabetes130US & 81.6 & {\bf 57.3} \% & 53.1\%  \\ \hline
			BNG\_letter\_5000\_1 & 105.6 &  {\bf 42.8}\% & 29.9\%  \\ \hline
			poker\_hand & 17.7 & 48.1\% & {\bf  49.4}\%  \\ \hline
			MiniBooNE & 11.6 &  {\bf  90.1}\% & 86.8\%  \\ \hline
			BNG\_eucalyptus & 80.3 &  {\bf  57.0}\% & 53.2\%  \\ \hline
			mnist & 1320.7 &  {\bf  90.4}\% & 85.3\%  \\ \hline
			BNG\_spambase & 16.5 & 66.5\% & 66.5\%  \\ \hline
			covtype & 123.7 & {\bf  70.8}\%  & 55.5\% \\ \hline
			volkert & 117.9 & {\bf  57.5}\% & 53.4\%  \\ \hline
			BNG\_satimage & 33.8 &  {\bf  81.6}\% & 80.3\%  \\ \hline
			Sensorless\_drive\_diagnosis & 156.7 &  {\bf  78.9\% }& 61.3\%  \\ \hline
			GTSRB-HueHist & 154.2 & 14.2\% & {\bf  16.2}\%  \\ \hline
			BNG\_mfeat\_fourier & 71.4 & {\bf  82.7}\% & 81.7\%  \\ \hline
		\end{tabular}%
	\end{center}
\end{table}

\remove{
\begin{table}[]
	\caption{SVM accuracies in the testing sets for \TCT and {\tt SGD}.}
	\label{tab:svm-sgd-accuracies}
	\begin{center}
		\begin{tabular}{c|c|cc}
			\textbf{Dataset} & \textbf{Time Limit} & \textbf{TCT} & \textbf{SGD}  \\ \hline \hline
			Diabetes130US & 81.6 & 57.3\% & 53.1\%  \\ \hline
			BNG\_letter\_5000\_1 & 105.6 &  42.8\% & 29.9\%  \\ \hline
			poker\_hand & 17.7 & 48.1\% & 49.4\%  \\ \hline
			MiniBooNE & 11.6 &  90.1\% & 86.8\%  \\ \hline
			BNG\_eucalyptus & 80.3 &  57.0\% & 53.2\%  \\ \hline
			mnist & 1320.7 &  90.4\% & 85.3\%  \\ \hline
			BNG\_spambase & 16.5 & 66.5\% & 66.5\%  \\ \hline
			covtype & 123.7 &  70.8\%  & 55.5\% \\ \hline
			volkert & 117.9 & 57.5\% & 53.4\%  \\ \hline
			BNG\_satimage & 33.8 &  81.6\% & 80.3\%  \\ \hline
			Sensorless\_drive\_diagnosis & 156.7 &  78.9\% & 61.3\%  \\ \hline
			GTSRB-HueHist & 154.2 & 14.2\% & 16.2\%  \\ \hline
			BNG\_mfeat\_fourier & 71.4 & 82.7\% & 81.7\%  \\ \hline
		\end{tabular}%
	\end{center}
\end{table}
}

\begin{table}[]
	\caption{Logistic Regression accuracies in the testing sets  for \TCT and {\tt SGD}..}
	\label{tab:logistic-regression-sgd-accuracies}
	\begin{center}
		\begin{tabular}{c|c|cc}
			\textbf{Dataset} & \textbf{Time Limit} & \textbf{TCT} & \textbf{SGD}  \\ \hline \hline			
			Diabetes130US & 289.6 &  {\bf 58.6\%} & 48.2\%  \\ \hline
			BNG\_letter\_5000\_1 & 31.5 &  {\bf 45.6\%}  & 43.3\% \\ \hline
			poker\_hand & 391.5 & 48.1\% &  {\bf 50.0\%}  \\ \hline
			BNG\_eucalyptus & 35.5 &  {\bf 57.8\%} & 55.9\%  \\ \hline
			mnist & 184.7 &  {\bf 91.6\%} & 87.1\%  \\ \hline
			BNG\_spambase & 34.1 & 66.4\% & 66.5\%  \\ \hline
			covtype & 84.0 & {\bf 68.2\%}  & 49.2\%  \\ \hline
			cifar\_10 & 465.0 & {\bf 37.1\%} & 25.6\%  \\ \hline
			volkert & 29.6 &  {\bf 57.4\%} & 54.7\% \\ \hline
			BNG\_satimage & 18.6 &  {\bf 83.8\%} & 81.3\% \\ \hline
			Sensorless\_drive\_diagnosis & 11.3 & {\bf 83.9\%} & 69.3\%  \\ \hline
			GTSRB-HueHist & 134.9 &  {\bf 26.3} \% & 19.3\%  \\ \hline
			BNG\_mfeat\_fourier & 46.1 &  {\bf 84.4}\% & 82.4\%  \\ \hline
		\end{tabular}%
	\end{center}
\end{table}

\remove{

\begin{table}[]
	\caption{Logistic Regression accuracies in the testing sets  for \TCT and {\tt SGD}..}
	\label{tab:logistic-regression-sgd-accuracies}
	\begin{center}
		\begin{tabular}{c|c|cc}
			\textbf{Dataset} & \textbf{Time Limit} & \textbf{TCT} & \textbf{SGD}  \\ \hline \hline			
			Diabetes130US & 289.6 & 58.6\% & 48.2\%  \\ \hline
			BNG\_letter\_5000\_1 & 31.5 & 45.6\% & 43.3\% \\ \hline
			poker\_hand & 391.5 & 48.1\% & 50.0\%  \\ \hline
			BNG\_eucalyptus & 35.5 & 57.8\% & 55.9\%  \\ \hline
			mnist & 184.7 & 91.6\% & 87.1\%  \\ \hline
			BNG\_spambase & 34.1 & 66.4\% & 66.5\%  \\ \hline
			covtype & 84.0 & 68.2\%  & 49.2\%  \\ \hline
			cifar\_10 & 465.0 & 37.1\% & 25.6\%  \\ \hline
			volkert & 29.6 & 57.4\% & 54.7\% \\ \hline
			BNG\_satimage & 18.6 & 83.8\% & 81.3\% \\ \hline
			Sensorless\_drive\_diagnosis & 11.3 & 83.9\% & 69.3\%  \\ \hline
			GTSRB-HueHist & 134.9 & 26.3\% & 19.3\%  \\ \hline
			BNG\_mfeat\_fourier & 46.1 & 84.4\% & 82.4\%  \\ \hline
		\end{tabular}%
	\end{center}
\end{table}

}

\clearpage

\subsection{Selecting examples via active learning - additional tables}
\label{sec:appendix-active-learning}
 Tables
\ref{tab:logistic-regression-al-accuracies},
	\ref{tab:randomforest-al-accuracies},
	\ref{tab:LGBM-al-accuracies} and
	\ref{tab:DecisionTree-al-accuracies}
show the average accuracy of \TCT and $\TCT_{AL}$
for the different Learners and datasets.

\begin{table}[!h] 
	\caption{Logistic Regression accuracies in the testing sets  for \TCT and {\tt TCT$\_{AL}$}.}
	\label{tab:logistic-regression-al-accuracies}
	\begin{center}
		\begin{tabular}{c|c|cc}
			\textbf{Dataset} & \textbf{Time Limit} & \textbf{TCT} & \textbf{TCT$\_{AL}$}  \\ \hline \hline			
			BayesianNetworkGenerator\_spambase & 45.6 & 66.4\% & 66.6\%\\  \hline
            BNG\_eucalyptus & 48.3 & 57.8\% & 57.9\%\\  \hline
            BNG\_letter\_5000\_1 & 45.1 & 45.6\% & \textbf{46.0\%}\\  \hline
            BNG\_mfeat\_fourier & 80.5 & 84.4\% & 84.5\%\\  \hline
            BNG\_satimage & 26.3 & 83.8\% & 83.8\%\\  \hline
            BNG\_spectf\_test & 10.6 & 76.9\% & \textbf{78.4\%}\\  \hline
            cifar\_10 & 938.8 & 37.1\% & 37.5\%\\  \hline
            covtype & 110.7 & 68.2\% & 67.9\%\\  \hline
            Diabetes130US & 350.7 & 58.6\% & 58.8\%\\  \hline
            GTSRB-HueHist & 237.7 & 26.3\% & 26.6\%\\  \hline
            jannis & 12.9 & 63.7\% & 64.2\%\\  \hline
            mnist & 266.1 & 91.6\% & 91.8\%\\  \hline
            poker\_hand & 526.5 & 48.1\% & \textbf{52.1\%}\\  \hline
            Sensorless\_drive\_diagnosis & 16.9 & 83.9\% & \textbf{88.1\%}\\  \hline
            volkert & 48.9 & 57.4\% & 57.4\%\\  \hline
		\end{tabular}%
	\end{center}
\end{table}

\begin{table}[] 
	\caption{Random Forest accuracies in the testing sets  for \TCT and {\tt TCT$\_{AL}$}.}
	\label{tab:randomforest-al-accuracies}
	\begin{center}
		\begin{tabular}{c|c|cc}
			\textbf{Dataset} & \textbf{Time Limit} & \textbf{TCT} & \textbf{TCT$\_{AL}$}  \\ \hline \hline			
			aloi & 25.1 & \textbf{81.1\%} & 68.7\%\\  \hline
            BayesianNetworkGenerator\_spambase & 517.3 & 66.6\% & 66.7\%\\  \hline
            BNG\_eucalyptus & 194.2 & \textbf{69.2\%} & 68.0\%\\  \hline
            BNG\_letter\_5000\_1 & 264.6 & \textbf{71.9\%} & 70.8\%\\  \hline
            BNG\_mfeat\_fourier & 889.9 & \textbf{88.8\%} & 88.5\%\\  \hline
            BNG\_satimage & 864.6 & \textbf{89.3\%} & 89.1\%\\  \hline
            BNG\_spectf\_test & 829.8 & \textbf{80.0\%} & 79.3\%\\  \hline
            BNG\_wine & 289.5 & \textbf{95.7\%} & 95.5\%\\  \hline
            cifar\_10 & 133.8 & 42.9\% & 42.1\%\\  \hline
            covtype & 63.8 & \textbf{92.9\%} & 89.5\%\\  \hline
            Diabetes130US & 69.4 & 58.5\% & 58.7\%\\  \hline
            GTSRB-HueHist & 35.0 & 40.0\% & 39.8\%\\  \hline
            jannis & 35.5 & \textbf{69.6\%} & 68.8\%\\  \hline
            MiniBooNE & 62.8 & \textbf{93.9\%} & 93.2\%\\  \hline
            mnist & 40.3 & 96.3\% & 96.4\%\\  \hline
            poker\_hand & 189.8 & 91.4\% & \textbf{92.1\%}\\  \hline
            SantanderCustomerSatisfaction & 460.7 & \textbf{90.5\%} & 90.0\%\\  \hline
            Sensorless\_drive\_diagnosis & 14.9 & 99.8\% & 99.6\%\\  \hline
            vehicle\_sensIT & 77.9 & \textbf{86.9\%} & 86.7\%\\  \hline
            volkert & 20.4 & \textbf{62.4\%} & 61.3\%\\  \hline
		\end{tabular}%
	\end{center}
\end{table}

\begin{table}[]
	\caption{LGBM accuracies in the testing sets  for \TCT and {\tt TCT$\_{AL}$}.}
	\label{tab:LGBM-al-accuracies}
	\begin{center}
		\begin{tabular}{c|c|cc}
			\textbf{Dataset} & \textbf{Time Limit} & \textbf{TCT} & \textbf{TCT$\_{AL}$}  \\ \hline \hline			
	        aloi & 982.4 & 12.6\% & \textbf{13.4\%}\\  \hline
            BNG\_eucalyptus & 57.6 & \textbf{74.9\%} & 74.2\%\\  \hline
            BNG\_letter\_5000\_1 & 162.4 & \textbf{74.1\%} & 73.5\%\\  \hline
            BNG\_mfeat\_fourier & 279.5 & \textbf{93.7\%} & 93.1\%\\  \hline
            BNG\_satimage & 83.7 & \textbf{92.2\%} & 91.5\%\\  \hline
            BNG\_spectf\_test & 18.5 & 82.6\% & 82.6\%\\  \hline
            BNG\_wine & 20.7 & 95.9\% & 95.9\%\\  \hline
            cifar\_10 & 977.3 & 42.6\% & 42.7\%\\  \hline
            covtype & 18.1 & \textbf{86.6\%} & 85.8\%\\  \hline
            GTSRB-HueHist & 270.7 & 34.0\% & \textbf{36.8\%}\\  \hline
            mnist & 109.0 & 96.4\% & \textbf{96.8\%}\\  \hline
            poker\_hand & 54.7 & 73.0\% & \textbf{77.9\%}\\  \hline
            SantanderCustomerSatisfaction & 23.0 & \textbf{91.3\%} & 90.9\%\\  \hline
            Sensorless\_drive\_diagnosis & 14.1 & \textbf{99.3\%} & 99.0\%\\  \hline
            volkert & 36.1 & 59.7\% & 59.1\%\\  \hline
		\end{tabular}%
	\end{center}
\end{table}

\begin{table}[] 
	\caption{Decision Tree accuracies in the testing sets  for \TCT and {\tt TCT$\_{AL}$}.}
	\label{tab:DecisionTree-al-accuracies}
	\begin{center}
		\begin{tabular}{c|c|cc}
			\textbf{Dataset} & \textbf{Time Limit} & \textbf{TCT} & \textbf{TCT$\_{AL}$}  \\ \hline \hline			
	        BNG\_eucalyptus & 13.3 & 53.1\% & \textbf{54.5\%}\\  \hline
            BNG\_mfeat\_fourier & 52.1 & 62.4\% & \textbf{64.0\%}\\  \hline
            BNG\_satimage & 22.1 & \textbf{73.8\%} & 73.0\%\\  \hline
            BNG\_spectf\_test & 27.1 & 77.6\% & \textbf{78.5\%}\\  \hline
            cifar\_10 & 34.0 & 24.9\% & 25.4\%\\  \hline
            Diabetes130US & 10.8 & 57.1\% & 57.1\%\\  \hline
            poker\_hand & 10.6 & 51.3\% & \textbf{53.5\%}\\  \hline
            SantanderCustomerSatisfaction & 18.4 & 89.0\% & \textbf{89.8\%}\\  \hline
		\end{tabular}%
	\end{center}
\end{table}

\remove{ 
\begin{table}[] 
	\caption{Logistic Regression accuracies in the testing sets  for \TCT and {\tt TCT$_{AL}$}.}
	\label{tab:logistic-regression-al-accuracies}
	\begin{center}
		\begin{tabular}{c|c|cc}
			\textbf{Dataset} & \textbf{Time Limit} & \textbf{TCT} & \textbf{TCT$_{AL}$}  \\ \hline \hline			
			BayesianNetworkGenerator\_spambase  &  45.6  &  66.4\%  &  66.5\% \\ \hline
            BNG\_eucalyptus  &  48.3  &  57.8\%  &  58.2\% \\ \hline
            BNG\_letter\_5000\_1  &  45.1  &  45.6\%  &  46.0\% \\ \hline
            BNG\_mfeat\_fourier  &  80.5  &  84.4\%  &  84.6\% \\ \hline
            BNG\_satimage  &  26.3  &  83.8\%  &  83.8\% \\ \hline
            BNG\_spectf\_test  &  10.6  &  76.9\%  &  78.4\% \\ \hline
            cifar\_10  &  938.8  &  37.1\%  &  37.8\% \\ \hline
            covtype  &  110.7  &  68.2\%  &  68.4\% \\ \hline
            Diabetes130US  &  350.7  &  58.6\%  &  58.9\% \\ \hline
            GTSRB-HueHist  &  237.7  &  26.3\%  &  26.2\% \\ \hline
            jannis  &  12.9  &  63.7\%  &  65.0\% \\ \hline
            mnist  &  266.1  &  91.6\%  &  92.1\% \\ \hline
            poker\_hand  &  526.5  &  48.1\%  &  51.4\% \\ \hline
            Sensorless\_drive\_diagnosis  &  16.9  &  83.9\%  &  91.3\% \\ \hline
            volkert  &  48.9  &  57.4\%  &  58.0\% \\ \hline
		\end{tabular}%
	\end{center}
\end{table}

\begin{table}[] 
	\caption{Random Forest accuracies in the testing sets  for \TCT and {\tt TCT$_{AL}$}.}
	\label{tab:randomforest-al-accuracies}
	\begin{center}
		\begin{tabular}{c|c|cc}
			\textbf{Dataset} & \textbf{Time Limit} & \textbf{TCT} & \textbf{TCT$_{AL}$}  \\ \hline \hline			
			aloi  &  25.1  &  81.1\%  &  72.0\% \\ \hline
            BayesianNetworkGenerator\_spambase  &  517.3  &  66.6\%  &  66.6\% \\ \hline
            BNG\_eucalyptus  &  194.2  &  69.2\%  &  69.8\% \\ \hline
            BNG\_letter\_5000\_1  &  264.6  &  71.9\%  &  73.5\% \\ \hline
            BNG\_mfeat\_fourier  &  889.9  &  88.8\%  &  89.5\% \\ \hline
            BNG\_satimage  &  864.6  &  89.7\%  &  89.0\% \\ \hline
            BNG\_spectf\_test  &  829.8  &  79.7\%  &  79.6\% \\ \hline
            BNG\_wine  &  289.5  &  95.7\%  &  95.5\% \\ \hline
            cifar\_10  &  133.8  &  42.9\%  &  43.4\% \\ \hline
            covtype  &  63.8  &  92.9\%  &  91.6\% \\ \hline
            Diabetes130US  &  69.4  &  58.5\%  &  58.9\% \\ \hline
            GTSRB-HueHist  &  35.0  &  40.0\%  &  42.9\% \\ \hline
            jannis  &  35.5  &  69.6\%  &  70.0\% \\ \hline
            MiniBooNE  &  62.8  &  93.9\%  &  93.1\% \\ \hline
            mnist  &  40.3  &  96.3\%  &  96.3\% \\ \hline
            poker\_hand  &  189.8  &  91.4\%  &  71.4\% \\ \hline
            SantanderCustomerSatisfaction  &  460.7  &  90.6\%  &  90.0\% \\ \hline
            Sensorless\_drive\_diagnosis  &  14.9  &  99.8\%  &  99.7\% \\ \hline
            vehicle\_sensIT  &  77.9  &  86.9\%  &  86.8\% \\ \hline
            volkert  &  20.4  &  62.4\%  &  63.6\% \\ \hline
		\end{tabular}%
	\end{center}
\end{table}

\begin{table}[]
	\caption{LGBM accuracies in the testing sets  for \TCT and {\tt TCT$_{AL}$}.}
	\label{tab:LGBM-al-accuracies}
	\begin{center}
		\begin{tabular}{c|c|cc}
			\textbf{Dataset} & \textbf{Time Limit} & \textbf{TCT} & \textbf{TCT$_{AL}$}  \\ \hline \hline			
		aloi  &  982.4  &  12.6\%  &  12.6\% \\ \hline
        BNG\_eucalyptus  &  57.6  &  74.9\%  &  75.3\% \\ \hline
        BNG\_letter\_5000\_1  &  162.4  &  74.1\%  &  75.4\% \\ \hline
        BNG\_mfeat\_fourier  &  279.5  &  93.7\%  &  93.2\% \\ \hline
        BNG\_satimage  &  83.7  &  92.2\%  &  91.8\% \\ \hline
        BNG\_spectf\_test  &  18.5  &  82.6\%  &  82.4\% \\ \hline
        BNG\_wine  &  20.7  &  95.9\%  &  95.8\% \\ \hline
        cifar\_10  &  977.3  &  42.6\%  &  40.6\% \\ \hline
        covtype  &  18.1  &  85.7\%  &  86.5\% \\ \hline
        GTSRB-HueHist  &  270.7  &  34.0\%  &  33.8\% \\ \hline
        mnist  &  109.0  &  96.4\%  &  96.1\% \\ \hline
        poker\_hand  &  54.7  &  73.0\%  &  77.4\% \\ \hline
        SantanderCustomerSatisfaction  &  23.0  &  91.1\%  &  90.6\% \\ \hline
        Sensorless\_drive\_diagnosis  &  14.1  &  99.3\%  &  98.9\% \\ \hline
        volkert  &  36.1  &  57.8\%  &  58.4\% \\ \hline
		\end{tabular}%
	\end{center}
\end{table}

\begin{table}[] 
	\caption{Decision Tree accuracies in the testing sets  for \TCT and {\tt TCT$_{AL}$}.}
	\label{tab:DecisionTree-al-accuracies}
	\begin{center}
		\begin{tabular}{c|c|cc}
			\textbf{Dataset} & \textbf{Time Limit} & \textbf{TCT} & \textbf{TCT$_{AL}$}  \\ \hline \hline			
	        BNG\_eucalyptus  &  13.3  &  53.8\%  &  53.6\% \\ \hline
            BNG\_mfeat\_fourier  &  52.1  &  66.3\%  &  62.7\% \\ \hline
            BNG\_satimage  &  22.1  &  73.8\%  &  75.2\% \\ \hline
            BNG\_spectf\_test  &  27.1  &  77.6\%  &  78.2\% \\ \hline
            cifar\_10  &  34.0  &  24.9\%  &  25.2\% \\ \hline
            Diabetes130US  &  10.8  &  57.1\%  &  57.4\% \\ \hline
            poker\_hand  &  10.6  &  51.3\%  &  53.2\% \\ \hline
            SantanderCustomerSatisfaction  &  18.4  &  89.0\%  &  89.8\% \\ \hline
		\end{tabular}%
	\end{center}
\end{table}
}

\remove{ 
\newcommand{\addDecTreeTCTvsAL}{\includegraphics[width=19em]{appendix/TCT_vs_AL/DecisionTree-Alldatasets.png}}

\newcommand{\addLGBMTCTvsAL}{\includegraphics[width=19em]{appendix/TCT_vs_AL/LGBM-All datasets.png}}

\newcommand{\addLogRegTCTvsAL}{\includegraphics[width=19em]{appendix/TCT_vs_AL/LogisticRegression-Alldatasets.png}}

\newcommand{\addRandForestTCTvsAL}{\includegraphics[width=19em]{appendix/TCT_vs_AL/RandomForest-Alldatasets.png}}

\begin{figure}
\begin{center}
\begin{tabular}{lcc}
& \addDecTreeTCTvsAL  & \addLogRegTCTvsAL    \\ 
&  \addRandForestTCTvsAL & \addLGBMTCTvsAL  \\

\end{tabular}
\end{center}
\caption{Average accuracies on testing set along normalized time for TCT and TCT$_{AL}$. The numbers next to the labels are their average accuracies at the last normalized time limit $t=1$.}
\label{fig:tct-vs-al-appendix}
\end{figure}

\begin{table}[]
	\caption{Logistic Regression accuracies in the testing sets  for \TCT and {\tt TCT$_{AL}$}.}
	\label{tab:logistic-regression-al-accuracies}
	\begin{center}
		\begin{tabular}{c|c|cc}
			\textbf{Dataset} & \textbf{Time Limit} & \textbf{TCT} & \textbf{TCT$_{AL}$}  \\ \hline \hline			
			BayesianNetworkGenerator\_spambase  &  45.6  &  66.4\%  &  66.5\% \\ \hline
            BNG\_eucalyptus  &  48.3  &  57.8\%  &  58.2\% \\ \hline
            BNG\_letter\_5000\_1  &  45.1  &  45.6\%  &  46.0\% \\ \hline
            BNG\_mfeat\_fourier  &  80.5  &  84.4\%  &  84.7\% \\ \hline
            BNG\_satimage  &  26.3  &  83.8\%  &  83.8\% \\ \hline
            BNG\_spectf\_test  &  10.6  &  76.9\%  &  78.4\% \\ \hline
            cifar\_10  &  938.8  &  37.1\%  &  37.8\% \\ \hline
            covtype  &  110.7  &  68.2\%  &  68.4\% \\ \hline
            Diabetes130US  &  350.7  &  58.6\%  &  58.9\% \\ \hline
            GTSRB-HueHist  &  237.7  &  26.3\%  &  26.2\% \\ \hline
            jannis  &  12.9  &  63.7\%  &  65.0\% \\ \hline
            mnist  &  266.1  &  91.6\%  &  92.1\% \\ \hline
            poker\_hand  &  526.5  &  48.1\%  &  51.3\% \\ \hline
            Sensorless\_drive\_diagnosis  &  16.9  &  83.9\%  &  82.9\% \\ \hline
            volkert  &  48.9  &  57.4\%  &  58.0\% \\ \hline
		\end{tabular}%
	\end{center}
\end{table}

\begin{table}[]
	\caption{Random Forest accuracies in the testing sets  for \TCT and {\tt TCT$_{AL}$}.}
	\label{tab:logistic-regression-al-accuracies}
	\begin{center}
		\begin{tabular}{c|c|cc}
			\textbf{Dataset} & \textbf{Time Limit} & \textbf{TCT} & \textbf{TCT$_{AL}$}  \\ \hline \hline			
			aloi  &  25.1  &  81.1\%  &  72.0\% \\ \hline
            BayesianNetworkGenerator\_spambase  &  517.3  &  66.6\%  &  66.6\% \\ \hline
            BNG\_eucalyptus  &  194.2  &  69.2\%  &  69.8\% \\ \hline
            BNG\_letter\_5000\_1  &  264.6  &  71.9\%  &  73.5\% \\ \hline
            BNG\_mfeat\_fourier  &  889.9  &  88.8\%  &  89.5\% \\ \hline
            BNG\_satimage  &  864.6  &  89.7\%  &  89.0\% \\ \hline
            BNG\_spectf\_test  &  829.8  &  79.7\%  &  79.6\% \\ \hline
            BNG\_wine  &  289.5  &  95.7\%  &  95.5\% \\ \hline
            cifar\_10  &  133.8  &  42.9\%  &  43.4\% \\ \hline
            covtype  &  63.8  &  92.9\%  &  91.6\% \\ \hline
            Diabetes130US  &  69.4  &  58.5\%  &  58.9\% \\ \hline
            GTSRB-HueHist  &  35.0  &  40.0\%  &  42.9\% \\ \hline
            jannis  &  35.5  &  69.6\%  &  70.0\% \\ \hline
            MiniBooNE  &  62.8  &  93.9\%  &  93.1\% \\ \hline
            mnist  &  40.3  &  96.3\%  &  96.3\% \\ \hline
            poker\_hand  &  189.8  &  91.4\%  &  70.7\% \\ \hline
            SantanderCustomerSatisfaction  &  460.7  &  90.6\%  &  90.0\% \\ \hline
            Sensorless\_drive\_diagnosis  &  14.9  &  99.8\%  &  99.7\% \\ \hline
            vehicle\_sensIT  &  77.9  &  86.9\%  &  86.8\% \\ \hline
            volkert  &  20.4  &  62.4\%  &  63.6\% \\ \hline
		\end{tabular}%
	\end{center}
\end{table}

\begin{table}[]
	\caption{LGBM accuracies in the testing sets  for \TCT and {\tt TCT$_{AL}$}.}
	\label{tab:logistic-regression-al-accuracies}
	\begin{center}
		\begin{tabular}{c|c|cc}
			\textbf{Dataset} & \textbf{Time Limit} & \textbf{TCT} & \textbf{TCT$_{AL}$}  \\ \hline \hline			
		aloi  &  982.4  &  12.6\%  &  12.7\% \\ \hline
        BNG\_eucalyptus  &  57.6  &  74.9\%  &  75.3\% \\ \hline
        BNG\_letter\_5000\_1  &  162.4  &  74.1\%  &  75.4\% \\ \hline
        BNG\_mfeat\_fourier  &  279.5  &  93.7\%  &  93.2\% \\ \hline
        BNG\_satimage  &  83.7  &  92.2\%  &  91.8\% \\ \hline
        BNG\_spectf\_test  &  18.5  &  82.6\%  &  82.4\% \\ \hline
        BNG\_wine  &  20.7  &  95.9\%  &  95.8\% \\ \hline
        cifar\_10  &  977.3  &  42.6\%  &  40.6\% \\ \hline
        covtype  &  18.1  &  85.7\%  &  86.5\% \\ \hline
        GTSRB-HueHist  &  270.7  &  34.0\%  &  33.8\% \\ \hline
        mnist  &  109.0  &  96.4\%  &  96.1\% \\ \hline
        poker\_hand  &  54.7  &  73.0\%  &  75.3\% \\ \hline
        SantanderCustomerSatisfaction  &  23.0  &  91.1\%  &  90.6\% \\ \hline
        Sensorless\_drive\_diagnosis  &  14.1  &  99.3\%  &  98.9\% \\ \hline
        volkert  &  36.1  &  57.8\%  &  58.4\% \\ \hline
		\end{tabular}%
	\end{center}
\end{table}

\begin{table}[]
	\caption{Decision Tree accuracies in the testing sets  for \TCT and {\tt TCT$_{AL}$}.}
	\label{tab:logistic-regression-al-accuracies}
	\begin{center}
		\begin{tabular}{c|c|cc}
			\textbf{Dataset} & \textbf{Time Limit} & \textbf{TCT} & \textbf{TCT$_{AL}$}  \\ \hline \hline			
	        BNG\_eucalyptus  &  13.3  &  53.8\%  &  53.6\% \\ \hline
            BNG\_mfeat\_fourier  &  52.1  &  66.3\%  &  62.7\% \\ \hline
            BNG\_satimage  &  22.1  &  73.8\%  &  75.2\% \\ \hline
            BNG\_spectf\_test  &  27.1  &  77.6\%  &  78.2\% \\ \hline
            cifar\_10  &  34.0  &  24.9\%  &  25.2\% \\ \hline
            Diabetes130US  &  10.8  &  57.1\%  &  57.4\% \\ \hline
            poker\_hand  &  10.6  &  51.3\%  &  53.2\% \\ \hline
            SantanderCustomerSatisfaction  &  18.4  &  89.0\%  &  89.8\% \\ \hline
		\end{tabular}%
	\end{center}
\end{table}
}

\clearpage

 	\vspace{0.5cm}
\section{Proofs of Section \ref{sec:ProvableGuarantees}} 

	\subsection{Sending too many ``wrong examples'' is bad} \label{app:badExample}

	We construct a simple instance where the algorithm \TCTbase set with the ``wrong samples'' percentage $\alpha$ too high has very poor accuracy. 
	
	More precisely, this non-realizable instance has points $\mathcal{X} = \{1,2,\ldots,6\}$ and a hypothesis class with 4 classifiers $\mathcal{H} = \{h_1,h_2,h_3,\bar{h}\}$ that classify points as $+1$ as follows (the remaining points are classified as $-1$):
	\begin{align*}
		h_1:&~~~~ \{1,2,3,4\}\\
		h_2:&~~~~ \{1,2,5,6\}\\
		h_3:&~~~~ \{3,4,5,6\}\\
		\bar{h}:&~~~~ \{1,2,3,4,5,6\}.
	\end{align*}
	The correct classification $h^*$ classifies as $+1$ the odd points $\{1,3,5\}$. Finally, the distribution $\mu$ puts $\frac{2}{9}$ probability on each odd number and probability $\frac{1}{9}$ on each even number of $\mathcal{X}$. 

	The best classifier in $\mathcal{H}$ is $\bar{h}$, which has error $\err(\bar{h}) = \frac{1}{3}$, while all other classifiers have error $\err(h_i) = \frac{4}{9}$. 
	
	We conducted experiments with \TCTbase where it sends $\alpha=90\%$ ``wrong'' samples in each round. We ran the algorithm for 20 rounds, so at the last round it has a total of $2^{21}-1 \approx 2,000,000$ samples (there can be/are multiple copies the same sample $(x,h^*(x))$). Over 100 attempts, this algorithm only found the best classifier $\bar{h}$ (in any of its rounds) $8\%$ of the time. In contrast, \PACt with $1,000$ samples found the best classifier $100\%$ of the time. 
	
	Note that at the last round \TCTbase has $(1-\alpha)\cdot (2^{21} - 1) \approx 200,000$ unbiased samples, which is orders of magnitude larger than those used by \PACt, but still the ``wrong samples'' caused it to have a very poor performance. 
	
	Also notice that when $\alpha$ is large as in this example, the bound from Theorem \ref{thm:fallbackA} below is vacuous, due to the last error term.

	\subsection{The Agnostic Case}
	\label{app:fallbackAA}

Consider the agnostic case. Let $\e^A_T = \e^A_T(\cH,\mu,\delta)$ be such that \PACt with time limit $T$ returns a hypothesis with true error at most that of the best classifier in $\cH$ plus $\e^A_T$ with probability at least $1- \delta$ regardless of the ERM learner; more precisely, let $S$ be a set of $m_T$ random samples from $\mu$ and let $\e_T^A$ be the smallest value such that 
	\begin{align*}
		\Pr\Big(\,\sup_{h \in \cH} |\err(h) - \err_S(h)| > \frac{\e_T^A}{2}\,\Big) < \delta.   
	\end{align*}
 We then have the following guarantee.

	\begin{thm}\label{thm:fallbackA}
	Given $\delta \in (0,1)$ and time limit $T$, let $\e_T^A$ be the $(1-\delta)$-probability additional error of \PACt as defined above.

	Under assumptions (i)-(iv), in the agnostic setting, with probability at least $1- \delta$, \TCTbase returns in time at most $ T  \cdot 2(\frac{2}{1-\alpha})^{k+1}$ a classifier $h$ with additional error at most $\e^A_T + \frac{\alpha}{1-\alpha}$, namely $$\err(h) \le \min_{h' \in \cH} \err(h') ~+~\e^A_T + \frac{\alpha}{1-\alpha}.$$  
	\end{thm}
	
	\begin{proof}	
	Again let $m_T$ be the number of samples sent by $\PACt$ when the time limit is $T$.
Moreover, let $\hat{i}$ be the first round in which \TCTbase sends at least  $\frac{1}{1-\alpha} m_T$ samples,
that is, $ \frac{1}{1-\alpha} m_T  \le 2^{\hat{i}} \le \frac{2}{1-\alpha} m_T$. Again due to the assumptions (ii) and (iv), inequality \eqref{eq:thm1} shows that \TCTbase finishes round $\hat{i}$ by time $2 ( \frac{2}{1-\alpha} )^{k+1} T$.
		
	Let $S$ be the set of samples sent by \TCTbase to the Learner by the end of round $\hat{i}$, and let $\hat{h}$ be the returned hypothesis. Also let $U$ be the subset of the samples $S$ that were sampled unbiasedly from $\mu$, and $W = S \setminus U$ the remaining ones. 

Since $U$ and $W$ make up a $(1-\alpha)$- and $\alpha$-fraction of $S$ respectively, we have 
		\begin{align}
		\err_{S}(\cdot) = (1-\alpha) \, \err_{U}(\cdot) + \alpha \, \err_{W}(\cdot). \label{eq:sampleError}
		\end{align}
In addition, by definition of $\hat{i}$ we have $|U| \ge (1-\alpha) |S| \ge m_T$, and so using the definition of $\e^A_T$ we have that with probability at least $1- \delta$,  for all $h \in \cH$ $$|\err(h) - \err_U(h)| \le \frac{\e^A_T}{2};$$ in particular, in light of \eqref{eq:sampleError}, for all $h \in \cH$
		\begin{align*}
		\err_{S}(h) \le (1-\alpha)\, \bigg[\err(h) + \frac{\e^A_T}{2}\bigg] + \alpha~~~~\textrm{ and }~~~~~ \err_{S}(h) \ge (1-\alpha)\, \bigg[\err(h) - \frac{\e^A_T}{2}\bigg].
		\end{align*}	
	Under this event, the classifier $\hat{h}$ returned by the ERM learner satisfies the following bound against every $h \in \cH$:
		\begin{align*}
			(1-\alpha)\, \err(\hat{h}) \le 		\err_{S}(\hat{h}) + (1-\alpha) \frac{\e^A_T}{2} \le \err_{S}(h) + (1-\alpha) \frac{\e^A_T}{2} \le (1-\alpha) \err(h) + (1-\alpha) \e^A_T + \alpha,
		\end{align*}
		and so taking an infimum over $h \in \cH$ we get $$err(\hat{h}) \le \min_{h' \in \cH} \err(h') + \e^A_T + \frac{\alpha}{1-\alpha}.$$ This concludes the proof. 	
	\end{proof}
	

	
	\subsection{Proof of Theorem \ref{thm:expImprov}} \label{app:expImprov}
	
	To prove this result we will need to use martingales. Recall that a sequence of random variables $X_1,\ldots,X_n$ is a \emph{martingale difference sequence} if $\E[X_i \mid X_1,\ldots,X_{i-1}] = 0$ for all $i$. We need the classic Freedman's Inequality for martingales. 
	
	\begin{thm}[Theorem 1.6 of \cite{freedman}]
		Consider a martingale difference sequence $X_1,\ldots,X_n$ such that $X_i \le 1$, and its predictable quadratic variation $V := \sum_i \E[X_i^2 \mid X_1, \ldots,X_{i-1}]$. Then for any $\lambda \ge 0$ and $v > 0$
		\begin{align*}
			\Pr\bigg(\sum_{i \le n} X_i \ge \lambda ~\textrm{ and }~ V \le v\bigg) \le \bigg(\frac{v}{\lambda + v} \bigg)^{\lambda + v} \,e^{\lambda}.
		\end{align*} 
	\end{thm}

	\begin{proof}[Proof of Theorem \ref{thm:expImprov}]
It suffices to prove that $$\textstyle T_{\PACt} \ge \Big(\Omega(\frac{1}{\e} \log \frac{1}{\delta})\Big)^k$$ and $$\textstyle T_{\TCTbase} \le \Big(2^{O(\sqrt{\log 1/\e})} \cdot \log \frac{1}{\delta} \cdot \big(\frac{1}{\alpha}\big)^{O(1)}\Big)^{k+1}.$$

		For the lower bound on $T_{\PACt}$, standard sample complexity lower bound for statistical learning threshold functions (for example Theorem 5.3 of \cite{anthonyBartlett}) says that there is a realizable instance where $\Omega(\frac{1}{\e} \log \frac{1}{\delta})$ samples are required by the \learner to obtain error at most $\e$ with probability at least $1-\delta$. Thus, given assumption (ii), the time limit needs to be at least 
		\begin{align*}
		 \textstyle	T_{\PACt} \ge \Big( \Omega(\frac{1}{\e} \log \frac{1}{\delta}) \Big)^k \cdot f\Big( \Omega(\frac{1}{\e} \log \frac{1}{\delta}) \Big) \ge \Big( \Omega(\frac{1}{\e} \log \frac{1}{\delta}) \Big)^k
		\end{align*} 
		to allow the \learner to train with these many samples. 
		
		We now upper bound $T_{\TCTbase}$ by first seeing how many examples and rounds are required to attain error $\e$ with probability at least $1-\delta$. Recall that $S_i$ is the set of examples that \TCTbase sends to the learner in round $i$, and $h_i$ is the hypothesis obtained in return. We again use the notation that $U_i \subseteq S_i$ is the set of samples up to the beginning of round $i$ that were drawn unbiasedly from $\mu$, and $W_i \subseteq S_i$ is the set ``wrong examples'' drawn thus far. We use $\mu_x$ to denote the marginal of $\mu$ on $\cX = \R$. 

		Let $v^* \in \R$ be the correct threshold for the given instance. Let $I_i \subseteq \R$ be the maximal interval containing $v^*$ that contains none of the examples $S_i$ (the ``uncertainty region'' at this time). Define $E_i \subseteq \R$ as the points where $h_i$'s classification is incorrect. Notice that this region is an interval that is either ``to the left'' or ``to the right'' of $v^*$ (depending whether the threshold used in $h_i$ is to the left or to the right of $v^*$). Also notice that $E_i$ is contained in $I_i$: since we are in a realizable instance, the rightmost sample in $S_i$ to the left of $v^*$ (which is the starting point of the open interval $I_i$) forces the ERM \textsc{Learner} to classify all points before it correctly as $-1$, and similarly to the points after the end of the open interval $I_i$.

	We define the ``left weight'' $\Lw(I)$ of the interval $I$ to be the amount of $\mu_x$-mass in the part of the interval that is to the left of $v^*$, i.e., $\Lw(I) = \mu_x(I \cap (-\infty,v^*])$. Similarly, define the ``right weight'' $\Rw(I) = \mu_x(I \cap [v^*,\infty))$. The following claim is the basis of the ``automatic binary search'' idea and says that with good probability in each round we significantly reduce the left/right-weight of the uncertainty region. 
	
	\begin{claim} \label{claim:weightLoss}
		For each realization of the samples drawn before round $i$ where the error region $E_i$ is to the left of $v^*$, with probability at least $1 - e^{-2^{3i/4}}$ (with respect to the samples drawn at round $i$) we have $$\Lw(I_{i+1}) \le \frac{\Lw(I_i)}{\alpha 2^{i/4}}.$$ Similarly, if the error region $E_i$ is to the right of $v^*$, with probability at least $1 - e^{-2^{3i/4}}$ we have $$\Rw(I_{i+1}) \le \frac{\Rw(I_i)}{\alpha 2^{i/4}}.$$
	\end{claim}
	
	\begin{proof}
	  We only prove the first statement, the second being analogous. Fix a realization of the samples up to the beginning of round $i$ where $E_i$ is to the left of $v^*$. By construction, at round $i$ the algorithm takes $\alpha 2^i$ samples from the distribution $\mu_x$ conditioned to being in the set $E_i$, call them $X_1,\ldots,X_{\alpha 2^i}$. Let $\bar{v}$ be the point such that the interval $[\bar{v},v^*]$ has $\mu_x$-measure $\frac{1}{\alpha 2^{i/4}}\cdot \mu_x(E_i)$. The probability that none of the samples $X_i$ lands in this interval is $\big(1-\frac{1}{\alpha 2^{i/4}}\big)^{\alpha 2^i} \le e^{-2^{3i/4}}$, so with probability at least $1 - e^{-2^{3i/4}}$ one of these samples lands in $[\bar{v},v^*]$. When this happens, the next uncertainty $I_{i+1}$ set starts at/after the point $\bar{v}$, and so its left weight satisfies $$\Lw(I_{i+1}) \le \mu_x([\bar{v}, v^*]) = \frac{1}{\alpha 2^{i/4}}\, \mu_x(E_i) \le \frac{1}{\alpha 2^{i/4}}\, \Lw(I_i),$$ where the last inequality is because $E_i \subseteq I_i$ and thus (since $E_i$ is to the left of $v^*$) $\mu_x(E_i) \le \mu_x(I_i \cap (-\infty,v^*]) = \Lw(I_i)$. This gives the desired result. 
	 \end{proof}
	
	Let $R := cst \cdot (\log \frac{1}{\alpha} + \sqrt{\log \frac{1}{\e}} + 1) + \log \log \frac{1}{\delta}$, for a sufficiently large constant $cst$. Using the above claim, we show that with probability at least $1-\delta$ the error set $E_{R+1}$ at round $R+1$ has $\mu_x$-measure at most $\e$, i.e. $h_{R+1}$ has error at most $\e$. Let $B_i$ be the indicator of the bad event that the weight reduction prescribed by the previous claim did not happened at round $i$. Then $\E[B_i \mid B_1, \ldots, B_{i-1}] \le e^{- 2^{3i/4}}$. Moreover, using Freedman's Inequality we have the following tail bound.
	
	\begin{claim} \label{claim:concentration}
	\begin{align*}
	\Pr\bigg(\sum_{i = 2R/3}^{R} B_i \ge \frac{R}{6}~\bigg) \le e^{-2^{R/2}} \le \delta,
	\end{align*}
	\end{claim}
	
	\begin{proof}
		Define $\widetilde{B}_i := B_i - \E[B_i \mid B_1,\ldots,B_{i-1}]$, so that the sequence $\widetilde{B}_{2R/3},\ldots,\widetilde{B}_R$ is a martingale difference sequence. Moreover, since $B_i$ only takes value 0 or 1, we have $|\widetilde{B}_i| \le 1$, and so $$(\widetilde{B}_i)^2 \,\le\, |\widetilde{B}_i| \,\le\, B_i + \E[B_i \mid B_1,\ldots,B_{i-1}].$$ From Claim \ref{claim:weightLoss}, for $i \ge \frac{2R}{3}$ we have 
	\begin{align*}
		&\E[B_i \mid B_1, \ldots, B_{i-1}] \le e^{-2^{3i/4}} \le e^{-2^{R/2}}\\
		\textrm{and so }~ &\E\big[\widetilde{B}_i^2 \mid \widetilde{B}_{2R/3},\ldots,\widetilde{B}_{i-1}] \le 2e^{-2^{R/2}}.
	\end{align*}  
	Then $v := 2 \cdot \frac{2R}{3} \cdot e^{-2^{R/2}}$ is an upper bound for both the shifts introduced in $\widetilde{B}_i$ and the predictable quadratic variation with probability 1:
	\begin{align*}
		&\textstyle \sum_{i = 2R/3}^R \E[B_i \mid B_1, \ldots, B_{i-1}] \le v\\
		\textrm{and }~&\textstyle \sum_{i = 2R/3}^R ~ \E\big[\widetilde{B}_i^2 \mid \widetilde{B}_{2R/3},\ldots,\widetilde{B}_{i-1}] \le v.
	\end{align*}  
	Then Freedman's Inequality gives
	\begin{align*}
		\Pr\bigg(\sum_{i = 2R/3}^R B_i \ge \frac{R}{6} \bigg) \le \Pr\bigg(\sum_{i = 2R/3}^R \widetilde{B}_i \ge \frac{R}{6} - v \bigg) \le \bigg(\frac{v}{R/6}\bigg)^{R/6} e^{R/6} &\le (8e \cdot e^{-2^{R/2}})^{R/6} \\
		&\le e^{-2^{R/2}},
	\end{align*} 
	where the last inequality uses the fact that $R \ge 10$ (by setting $cst$ large enough). Since $R \ge \log \log \frac{1}{\delta} + \sqrt{\log \frac{1}{\e}}$ and by assumption $\e \le \delta$, 
	 we have $e^{-2^{R/2}} \le e^{-2^{\log \log 1/\delta}} \le \delta$. This proves the claim.	
	\end{proof}
		
	It then suffice to show that whenever $\sum_{i = 2R/3}^{R} B_i < \frac{R}{6}$, we have $\mu_x(E_R) \le \e$. So fix a scenario satisfying the former and assume by contradiction that $\mu_x(E_R) > \e$. Define the subsets $G,L,R$ of the rounds $\{1,\ldots,R\}$ as follows: $G$ is the set of ``good'' indices $i$ such that $B_i = 0$, $L$ the set of indices where $E_i$ is to the left of $v^*$, and $R$ the set of indices where $E_i$ is to the right of $v^*$. 
	
	Then $L \cap G$ are the rounds where the reduction prescribed by Claim \ref{claim:weightLoss} happened on the left weight of $I_i$ and $R \cap G$ where it happened on its right weight. Moreover, the sets $I_1, I_2,\ldots$ are monotonically decreasing, so even for the rounds outside of the good set $G$ the left and right weights of $I_i$ only decrease over time. Then
	\begin{align*}
	 	\Lw(I_{R+1}) ~\le~  \Lw(\R) \cdot \prod_{i \in L \cap G} \frac{1}{\alpha 2^{i/4}} ~\le~ \bigg(\frac{1}{\alpha}\bigg)^R \cdot \frac{1}{2^{\sum_{i \in L \cap G} i/4}}.
	\end{align*}
	Combining with the fact $\Lw(I_{R+1}) \ge \Lw(E_{R+1}) \ge \mu_x(E_{R+1}) > \e$ and laking logs, this implies 
	\begin{align*}
		\sum_{i \in L \cap G} i \,\le\, 4 R \log \frac{1}{\alpha} + 4 \log \frac{1}{\e}.
	\end{align*}
	The same inequality holds with $L$ replaced by $R$, and adding these inequalities gives
	\begin{align}
		\sum_{i \in (L \cup R) \cap G} i \,\le\, 8 R \log \frac{1}{\alpha} + 8\log \frac{1}{\e}. \label{eq:contra}
	\end{align}	
	
	By the assumption of the scenario, we know that at least $\frac{R}{3} - \frac{R}{6} = \frac{R}{6}$ of the rounds in the interval $\{\frac{2R}{3},\ldots,R\}$ are in the good set $G$, that is, $|G| \ge \frac{R}{6}$. Then 
	%
	\begin{align*}
		\sum_{i \in (L \cup R) \cap G} i \,=\, \sum_{i \in G} i \,\ge\, \sum_{i = 1}^{|G|} i \ge \sum_{i = 1}^{R/6} i\,\ge\, \frac{R^2}{72}.
	\end{align*}	
	But since $R \ge cst \cdot (\log \frac{1}{\alpha} + \sqrt{\log \frac{1}{\e}})$ for a sufficiently large constant $cst$, this contradicts inequality \eqref{eq:contra} ($cst = 8 \cdot 72$ suffices, but the constants throughout have not been optimized). 
	
	So with probability at least $1-\delta$, by round $R$ the algorithm obtains a classifies with error at most $\e$. Due to the assumptions (ii) and (iv), the same development as in inequality \eqref{eq:thm1} shows that this round finishes by time 
	\begin{align*}
		\sum_{i=1}^R \Big((2^i)^k \cdot f(2^i)\Big) \le 2 f(2^R) 2^{Rk} \le 2^{R(k+1) + 1}.
	\end{align*}
	Using the value of $R$, this time is at most $$\textstyle \Big(2^{O(\sqrt{\log 1/\e})} \cdot \log \frac{1}{\delta} \cdot \big(\frac{1}{\alpha}\big)^{O(1)}\Big)^{k+1},$$ which then concludes the proof. 
	\end{proof}

\end{document}